\documentclass[10pt,journal,compsoc]{IEEEtran}
\ifCLASSOPTIONcompsoc
  \usepackage[nocompress]{cite}
\else
  \usepackage{cite}
\fi
\ifCLASSINFOpdf
\else
\fi

\usepackage{amsmath}
\usepackage{bm}
\usepackage{amssymb}
\usepackage{bbold}
\usepackage{algorithm}
\usepackage{algorithmic}
\usepackage{graphicx}
\usepackage{multirow}
\usepackage{arydshln}
\usepackage{pifont}
\usepackage{amsthm}

\usepackage{color}

\usepackage{subcaption}
\usepackage{mwe}

\usepackage[inline]{enumitem}

\newtheorem{lemma}{Lemma}
\newtheorem{theorem}{Theorem}
\newtheorem{definition}{Definition}
\newtheorem{remark}{Remark}

\newcommand\ceil[1]{\lceil#1\rceil}

\newcommand{\cmark}{\ding{51}}%
\newcommand{\xmark}{\ding{55}}%

\begin{document}
%
\title{Online Product Quantization}
%
%
%
%

\author{Donna~Xu,
        Ivor~W.~Tsang,
        and~Ying~Zhang
\IEEEcompsocitemizethanks{\IEEEcompsocthanksitem D. Xu, I. W. Tsang and Y. Zhang are with the Centre for Artificial Intelligence, FEIT, University of Technology Sydney, NSW, Australia.\protect\\
E-mail: donna.xu@student.uts.edu.au, ivor.tsang@uts.edu.au, ying.zhang@uts.edu.au
}
}

\IEEEtitleabstractindextext{%
\begin{abstract}
Approximate nearest neighbor (ANN) search has achieved great success in many tasks.
However, existing popular methods for ANN search, such as hashing and quantization methods, are designed for static databases only.
They cannot handle well the database with data distribution evolving dynamically, due to the high computational effort for retraining the model based on the new database.
In this paper, we address the problem by developing an online product quantization (online PQ) model and incrementally updating the quantization codebook that accommodates to the incoming streaming data.
Moreover, to further alleviate the issue of large scale computation for the online PQ update, we design two budget constraints for the model to update partial PQ codebook instead of all.
We derive a loss bound which guarantees the performance of our online PQ model.
Furthermore, we develop an online PQ model over a sliding window with both data insertion and deletion supported, to reflect the real-time behaviour of the data.
The experiments demonstrate that our online PQ model is both time-efficient and effective for ANN search in dynamic large scale databases compared with baseline methods and the idea of partial PQ codebook update further reduces the update cost.
\end{abstract}

}

\maketitle

\IEEEdisplaynontitleabstractindextext

%
\IEEEpeerreviewmaketitle

\IEEEraisesectionheading{\section{Introduction}\label{sec:introduction}}

%
%
%
%
\IEEEPARstart{A}{pproximate} nearest neighbor (ANN) search in a static database has achieved great success in supporting many tasks, such as information retrieval, classification and object detection.
However, due to the massive amount of data generation at an unprecedented rate daily in the era of big data, databases are dynamically growing with data distribution evolving over time, and existing ANN search methods would achieve unsatisfactory performance without new data incorporated in their models.
In addition, it is impractical for these methods to retrain the model from scratch for the continuously changing database due to the large scale computational time and memory. 
Therefore, it is increasingly important to handle ANN search in a dynamic database environment.

ANN search in a dynamic database has a widespread applications in the real world.
For example, a large number of news articles are generated and updated on hourly/daily basis, so a news searching system \cite{DBLP:journals/tkde/MoffatZS97} requires to support news topic tracking and retrieval in a frequently changing news database.
For object detection in video surveillance \cite{DBLP:conf/www/PopoviciWG14}, video data is continuously recorded, so that the distances between/among similar or dissimilar objects are continuously changing. 
For image retrieval in dynamic databases \cite{DBLP:conf/icmcs/DongB03}, relevant images are retrieved from a constantly changing image collection, and the retrieved images could therefore be different over time given the same image query.
In such an environment, real-time query needs to be answered based on all the data collected to the database so far.

In recent years, there has been an increasing concern over the computational cost and memory requirement dealing with continuously growing large scale databases, and therefore there are many online learning algorithm works \cite{DBLP:journals/jmlr/CrammerDKSS06,DBLP:conf/icml/ZhangYJXZ16} proposed to update the model each time streaming data coming in.
Therefore, we consider the following problem. 
Given a dynamic database environment, develop an online learning model accommodating the new streaming data with low computational cost for ANN search.

Recently, several studies on online hashing \cite{DBLP:conf/ijcai/HuangYZ13,huang2017online,DBLP:journals/corr/GhashamiA15,DBLP:conf/cvpr/LengWC0L15,DBLP:conf/iccv/CakirS15,DBLP:conf/ijcai/YangHZL13,cakir2016online} show that hashing based ANN approaches can be adapted to the dynamic database environment by updating hash functions accommodating new streaming data and then updating the hash codes of the exiting stored data via the new hash functions.
Searching is performed in the Hamming space which is efficient and has low computational cost.
However, an important problem that these works have not addressed is the computation of the hash code maintenance.
To handle the streaming fashion of the data, the hash functions are required to be frequently updated, which will result in constant hash code recomputation of all the existing data in the reference database.
This will inevitably incur an increasing amount of update time as the data volume increases.
In addition, these online hashing approaches require the system to keep the old data so that the new hash code of the old data can be updated each time, leading to inefficiency in memory and computational load.
Therefore, computational complexity and storage cost are still our major concerns in developing an online indexing model.
\begin{figure*}
	\centering
	\includegraphics[scale=0.43]{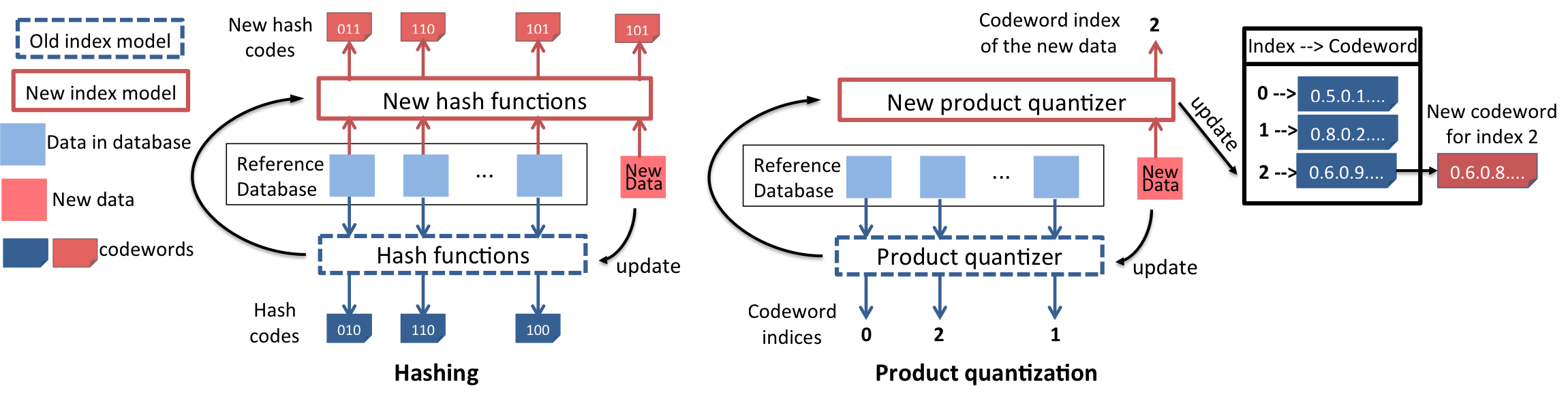}
	\caption{\label{hashing_vs_pq} Hashing vs PQ in online update. The hash codes of the data points in the reference database will get updated if the hash functions get updated by the new data. \emph{The index of the codewords in the PQ codebook, on the other hand, will remain the same even though the codebook gets updated by the new data.} Thus online PQ is able to save severely much time by avoiding codewords maintenance of the reference database. (Best viewed in colors)}
\end{figure*}

Product quantization (PQ) \cite{DBLP:journals/pami/JegouDS11} is an effective and successful alternative solution for ANN search.
PQ partitions the original space into a Cartesian product of low dimensional subspaces and quantizes each subspace into a number of sub-codewords.
In this way, PQ is able to produce a large number of codewords with low storage cost and perform ANN search with inexpensive computation.
Moreover, it preserves the quantization error and can achieve satisfactory recall performance. 
Most importantly, unlike hashing-based methods representing each data instance by a hash code, which depends on a set of hash functions, quantization-based methods represent each data instance by an index, which associates with a codeword that is in the same vector space with the data instance.
However, PQ is a batch mode method which is not designed for the problem of accommodating streaming data in the model.
Therefore, to address the problem of handling streaming data for ANN search and tackle the challenge of hash code recomputation, we develop an online PQ approach, which updates the codewords by streaming data without the need to update the indices of the existing data in the reference database, to further alleviate the issue of large scale update computational cost.

Figure \ref{hashing_vs_pq} compares hashing method and PQ in the code representation and maintenance, which illustrates the advantage of PQ over hashing in computational cost and memory efficiency.
Once the index models get updated by the streaming data, the updated hash functions in hashing methods will produce new hash codes for each data point in the reference database, which will incur expensive cost for large scale databases.
The updated product quantizer in PQ, on the other hand, updates the codewords in the codebook, but it does not change the index of the updated codewords of each data point in the reference database.
To further reduce the update computational cost, we illustrate the idea of partial codebook update \cite{DBLP:journals/tip/MaTPL17} and present two budget constraints for the model to update the codebook partially instead of all.
Furthermore, we derive a loss bound which guarantees the performance of online PQ.
Unlike traditional analysis, our model is a non-convex problem with matrices as the variables, so its theoretical analysis is not trivial to be handled.
To emphasize the real-time data for querying, we also propose an online PQ model over a sliding window, which support both data insertion and deletion.

\section{Related Works}
\begin{table*}[]
\centering
\caption{Comparison between the existing methods and ours}
\label{attribute-table}
\scalebox{0.8}{
	\begin{tabular}{|c|c|c|c|c|c|}
		\hline
		Method & \begin{tabular}[c]{@{}c@{}}Applicable to\\ streaming data\end{tabular} & \begin{tabular}[c]{@{}c@{}}Preserves the\\ quantization error\end{tabular} & 
		\begin{tabular}[c]{@{}c@{}}Does not require \\ codewords maintenance\end{tabular} &
		\begin{tabular}[c]{@{}c@{}}Does not \\ require labels\end{tabular} &
		\begin{tabular}[c]{@{}c@{}}Does not require \\ to keep old data\end{tabular}  \\ \hline
		\begin{tabular}[c]{@{}c@{}}Supervised data-dependent hashing\\ (\cite{DBLP:conf/icml/NorouziF11,DBLP:conf/cvpr/LiuWJJC12})\end{tabular} & \xmark & \xmark & \xmark & \xmark & \xmark \\ \hline
		\begin{tabular}[c]{@{}c@{}}Unsupervised data-dependent hashing\\ (\cite{DBLP:conf/cvpr/GongL11,DBLP:conf/nips/KongL12,DBLP:conf/nips/WeissTF08})\end{tabular} & \xmark & \xmark & \xmark & \cmark & \xmark \\ \hline
		\begin{tabular}[c]{@{}c@{}} Data-independent hashing\\ (\cite{DBLP:conf/vldb/GionisIM99})\end{tabular} & \cmark & \xmark & \xmark & \cmark & \cmark \\ \hline
		\begin{tabular}[c]{@{}c@{}} Online supervised hashing\\ (\cite{DBLP:conf/ijcai/HuangYZ13,huang2017online,DBLP:conf/iccv/CakirS15,DBLP:conf/ijcai/YangHZL13,cakir2016online})\end{tabular} & \cmark & \xmark & \xmark & \xmark & \xmark \\ \hline
		\begin{tabular}[c]{@{}c@{}} Online unsupervised hashing\\ (\cite{DBLP:journals/corr/GhashamiA15,DBLP:conf/cvpr/LengWC0L15})\end{tabular} & \cmark & \xmark & \xmark & \cmark & \xmark \\ \hline
		\begin{tabular}[c]{@{}c@{}} Quantization that requires codewords maintenance: \\AQ, CQ, SQ, TQ, KMH, ABQ, SSH (\cite{DBLP:conf/cvpr/BabenkoL14,DBLP:conf/icml/ZhangDW14, DBLP:conf/cvpr/ZhangQTW15,DBLP:conf/cvpr/BabenkoL15,DBLP:conf/cvpr/HeWS13,DBLP:conf/ecai/LiLWS16,DBLP:journals/tip/LiuLDT17,DBLP:journals/tcyb/LiuDDLL16})\end{tabular} & \xmark & \cmark & \xmark & \cmark & \xmark \\ \hline
		\begin{tabular}[c]{@{}c@{}} Quantization that does not require codewords maintenance: \\PQ, OPQ (\cite{DBLP:journals/pami/JegouDS11,DBLP:journals/pami/GeHK014})\end{tabular} & \xmark & \cmark & \cmark & \cmark & \xmark \\ \hline
		Proposed model & \cmark & \cmark & \cmark & \cmark & \cmark \\ \hline
	\end{tabular} }
\end{table*}
Hashing methods generate a set of hash functions to map a data instance to a hash code in order to facilitate fast nearest neighbor search.
Existing hashing methods are grouped in data-independent hashing and data-dependent hashing.
One of the most representative work for data-independent hashing is Locality Sensitive Hashing (LSH) \cite{DBLP:conf/vldb/GionisIM99}, where its hashing functions are randomly generated.
LSH has the theoretical performance guarantee that similar data instances will be mapped to similar hash codes with a certain probability.
Since data-independent hashing methods are independent from the input data, they can be easily adopted in an online fashion.
Data-dependent hashing, on the other hand, learns the hash functions from the given data, which can achieve better performance than data-independent hashing methods.
Its representative works are Spectral Hashing (SH) \cite{DBLP:conf/nips/WeissTF08,DBLP:journals/tcyb/ChenXTL14}, which uses spectral method to encode similarity graph of the input into hash functions, IsoH \cite{DBLP:conf/nips/KongL12} which finds a rotation matrix for equal variance in the projected dimensions and ITQ \cite{DBLP:conf/cvpr/GongL11} which learns an orthogonal rotation matrix for minimizing the quantization error of data items to their hash codes.

To handle nearest neighbor search in a dynamic database, online hashing methods \cite{DBLP:conf/ijcai/HuangYZ13,huang2017online,DBLP:journals/corr/GhashamiA15,DBLP:conf/cvpr/LengWC0L15,DBLP:conf/iccv/CakirS15,DBLP:conf/ijcai/YangHZL13,cakir2016online} have attracted a great attention in recent years. They allow their models to accommodate to the new data coming sequentially, without retraining all stored data points.
Specifically, Online Hashing \cite{DBLP:conf/ijcai/HuangYZ13,huang2017online}, AdaptHash \cite{DBLP:conf/iccv/CakirS15} and Online Supervised Hashing \cite{cakir2016online} are online supervised hashing methods, requiring label information, which might not be commonly available in many real-world applications.
Stream Spectral Binary Coding (SSBC) \cite{DBLP:journals/corr/GhashamiA15} and Online Sketching Hashing (OSH) \cite{DBLP:conf/cvpr/LengWC0L15} are the only two  existing online unsupervised hashing methods which do not require labels, where both of them are matrix sketch-based methods to learn to represent the data seen so far by a small sketch.
However, all the online hashing methods suffer from the existing data storage and the high computational cost of hash code maintenance on the existing data.
Each time new data comes, they update their hash functions accommodating to the new data and then update the hash codes of all stored data according to the new hash functions, which could be very time-consuming for a large scale database.

Multi-codebook quantization (MCQ) methods \cite{DBLP:journals/pami/JegouDS11,DBLP:journals/pami/GeHK014,DBLP:conf/cvpr/BabenkoL14,DBLP:conf/icml/ZhangDW14, DBLP:conf/cvpr/ZhangQTW15,DBLP:conf/cvpr/BabenkoL15} are derived by minimizing the quantization error between the original input data and their corresponding codewords.
Each codeword is represented by a set of sub-codewords selected from multiple codebooks.
Please refer to \cite{matsui2018survey} for a comprehensive literature survey.
To the best of our knowledge, no MCQ methods have been explored to an online fasion.
Some of the popular works such as Composite Quantization (CQ) \cite{DBLP:conf/icml/ZhangDW14}, Sparse Composite Quantization (SQ) \cite{DBLP:conf/cvpr/ZhangQTW15}, Additive Quantization (AQ) \cite{DBLP:conf/cvpr/BabenkoL14} and Tree Quantization (TQ) \cite{DBLP:conf/cvpr/BabenkoL15} require the codeword maintenance of the old data to update the codewords due to the constraints or structure of their models.
Product Quantization (PQ) \cite{DBLP:journals/pami/JegouDS11}, as one of the most classical MCQ method for fast nearest neighbor search, decomposes the input space into a Cartesian product of subspaces.
The codeword of a data instance is represented by the concatenation of the sub-codeword of the data in all subspaces.
More specifically, each sub-codeword is represented by an index.
Thus the codeword of a data instance is the concatenation of the sub-codeword indices of all subspaces.
Assuming that the codebook will not change much with an update by a streaming data, an online PQ model can be proposed.
Under this assumption, there is no need of codewords maintenance of the existing data, as the indices of the sub-codewords of the data remain the same even though the actual sub-codewords are updated.
It is different from online hashing methods, which requires the hash code update for the existing data each time the hash functions get updated.
Extension works of PQ such as Optimized Product Quantization (OPQ) \cite{DBLP:journals/pami/GeHK014}, K-means Hashing (KMH) \cite{DBLP:conf/cvpr/HeWS13}, Adaptive Binary Quantization (ABQ) \cite{DBLP:conf/ecai/LiLWS16,DBLP:journals/tip/LiuLDT17} and Structure Sensitive Hashing (SSH) \cite{DBLP:journals/tcyb/LiuDDLL16} can also be developed to an online fashion.
Specifically, KMH, ABQ and SSH approximate the distance between codewords by their Hamming distance, so they will require codewords maintenance in the online setting.
In this paper, we focus on PQ method, so a novel online paradigm for PQ is proposed.
The difference between existing methods and ours are summarized in Table \ref{attribute-table}.

\section{Online Product Quantization with Budget Constraints}
\begin{figure}
	\centering
	\includegraphics[scale=0.5]{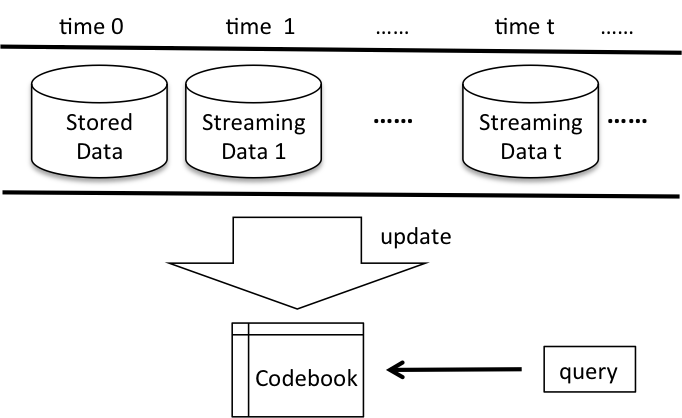}
	\caption{\label{online_update} A general procedure for Online Product Quantization update. At each iteration, the new codebook gets updated by the streaming data. Searching can be performed using the latest codebook.}
\end{figure}

The traditional PQ method assumes that the database is static and hence it is not suitable for non-stationary data setting especially when the data is time-varying.
Therefore, it is crucial to develop an online version of PQ to deal with dynamic database.
A general procedure of our online product quantization framework is illustrated in Figure \ref{online_update}.
The codebook at each iteration gets updated by the streaming data without retraining all the collected data.
ANN search can be conducted against the latest codebook in terms of user queries.
Unlike online hashing methods which update hashing functions and hash codes of the existing data, online PQ updates codebooks only and the codeword index of the existing data remains the same.

\subsection{Preliminaries}
We first define the vector quantization approach \cite{gray1998quantization} and the concept of quantization error, and then introduce Product quantization.
Table \ref{notations} summarizes the notations frequently used in this paper.
\begin{table}[]
	\centering
	\caption{Summary of Notations}
	\label{notations}
	\scalebox{0.87}{
	\begin{tabular}{|l|l|}
		\hline
		\textbf{Notation} & \textbf{Definition} \\ \hline
		$\mathcal{X}$ & a set of data points \\ \hline
		$\mathcal{Z}$ & codebook \\ \hline
		$||\cdot||$ & $l_2$-norm \\ \hline
		$||\cdot||_{F}$ & Frobenius norm \\ \hline
		$t$ & iteration number. $t = 1, 2, ..., T$ \\ \hline
		$M$ & the number of subspaces (subquantizers) \\ \hline
		$K$ & the number of sub-codewords in each subspace \\ \hline
		$x^t \in \mathbb{R}^{D}$ & the streaming data in iteration $t$ \\ \hline
		$x_m^t \in \mathbb{R}^{\frac{D}{M}}$ & the subvector of the data in the $m$th subspace in iteration $t$ \\ \hline
		$z_{m,k}^t \in \mathbb{R}^{\frac{D}{M}}$ & the $k$th sub-codeword of the $m$th subspace in iteration $t$ \\ \hline
		$Z_{m}^{t} \in \mathbb{R}^{\frac{D}{M} \times K}$ & sub-codewords concatenation $[z_{m,1}^{t}, ..., z_{m,K}^{t}]$ \\ \hline
		$\mathcal{C}_{m,k}^{t}$ &  the set of vector index assigned to the sub-codeword $z_{m,k}^t$ \\ \hline
		$n_{m,k}$ &  the number of vectors assigned to the sub-codeword $z_{m,k}^t$ \\ \hline
		$\Delta z_{m,k}^t \in \mathbb{R}^{\frac{D}{M}}$ & \begin{tabular}[c]{@{}l@{}}the difference of the $z_{m,k}^t$ in two consecutive iterations\\ i.e., $z_{m,k}^{t+1} = z_{m,k}^{t} + \Delta z_{m,k}^t$\end{tabular} \\ \hline
		$\Delta Z_m^t \in \mathbb{R}^{\frac{D}{M} \times K}$ & differences concatenation $[\Delta z_{m,1}^t,...,\Delta z_{m,K}^t]$ \\ \hline
		$B$ & size of the mini-batch \\ \hline
		$\alpha$ & the number of subspaces to be updated \\ \hline
		$\lambda$ & the percentage of sub-codewords to be updated \\ \hline
		$L$ & size of the sliding window \\ \hline
	\end{tabular}}
\end{table}
\begin{definition}[Vector quantization \cite{gray1998quantization}] \label{def1}
\emph{Vector quantization} approach quantizes a vector $x \in {\rm I\!R}^D$ to its codeword $z_k$ in a codebook $\mathcal{Z} = \{z_k\}$ where $k \in \{1,...,K\}$. 
\end{definition}
\begin{definition}[Quantization error] \label{def2}
Given a finite set of data points $\mathcal{X}$, vector quantization aims to minimize the \emph{quantization error} which is defined in the following:
\begin{align*}
\min_{\substack{\mathcal{C}_{1},...,\mathcal{C}_{K}\\ z_{1},...,z_{K}}} \sum_{i=1}^{|\mathcal{X}|}\|x^i - z_{k} \|^2 
\end{align*}
where $\mathcal{C}_{k}$ is the set of data indices assigned to the codeword $z_k$. \\
\textbf{\emph{Remark.}} In batch methods, ${C_k}$ can be automatically created after ${z_k}$ are decided. In online methods, however, ${C_k}$ is updated over time and it is not recomputed with respect to ${z_k}$ at each iteration.

\end{definition}
According to the first Lloyd's condition, $x^i$ should be mapped to its nearest codeword $z_k$ in the codebook.
A codeword $z_k$ can be computed as the centroid of the vectors with index in $\mathcal{C}_k$.
All of the codewords form the codebook $\mathcal{Z}$ with size $K$.

Unlike vector quantization which uses one quantizer to map a vector, product quantization (PQ) uses $M$ subquantizers.
It represents any $x \in {\rm I\!R}^D$ as a concatenation of $M$ sub-vectors $[x_1,...,x_m,...,x_M]$ where $x_m \in {\rm I\!R}^{D/M}$, assuming that $D$ is a multiple of $M$ for simplicity.
The PQ codebook is then composed of $M$ sub-codebooks and each of the sub-codebook contains $K$ sub-codewords quantized from a distinct subquantizer.
Any codeword belongs to the Cartesian product of the sub-codewords in each sub-codebook.
The codeword of x is constructed by the concatenation of M sub-codewords $z = [z_{1,k_1},...,z_{m,k_m},...,z_{M,k_M}]$, where $z_{m,k_m}$ is the sub-codeword of $x_m$.

\subsection{Online Product Quantization}
Inspired by product quantization and online learning, the objective function of the online product quantization at each iteration $t$ is shown in the following:
\begin{equation}\label{3}
\begin{split}
\min_{\substack{\mathcal{C}_{1,1}^t,...,\mathcal{C}_{m,k}^t,...,\mathcal{C}_{M,K}^t\\ z_{1,1}^t,...,z_{m,k}^t,...,z_{M,K}^t}} \sum_{m=1}^M\|x_m^{t} - z_{m,k}^t \|^2 \\
\end{split}
\end{equation}
where  $x_m^t$ is the streaming data in the $m$th subspace in the $t$th iteration and its nearest sub-codeword is $z_{m,k}^t$.
We expect to minimize the quantization error of the data at the current iteration $t$.
Inspired by sequential vector quantization algorithm \cite{Alpaydin:2010:IML:1734076} to update the codebook, the solution of online PQ is shown in Algorithm \ref{alg:1}.
\begin{algorithm}
	\caption{\label{alg:1} Online PQ}
	\begin{algorithmic}[1]
		\STATE initialize PQ with the $M*K$ sub-codewords $z_{1,1}^0,...,z_{m,k}^0,...,z_{M,K}^0$ using a initial set of data  \\
		\STATE initialize $C_{1,1}^0,...,C_{m,k}^0,...,C_{M,K}^0$ to be the cluster sets that contain the index of the initial data that belong to the cluster
		\STATE create counters $n_{1,1},...,n_{m,k},...,n_{M,K}$ for each cluster and initialize each $n_{m,k}$ to be the number of initial data points assigned to the corresponding $C_{m,k}^0$ \\
		\FOR{$t=1, 2, 3, ...$}
		\STATE get a new data $x^t$\\
		\STATE partition $x^t$ into $M$ subspaces $[x_1^t,...,x_M^t]$\\
		\STATE in each subspace $m \in \{1,...,M\}$, determine and assign the nearest sub-codeword $z_{m,k}^t$ for each sub-vector $x_m^t$\\ \label{alg:codeword}
		\STATE update the cluster set $C_{m,k}^t \leftarrow C_{m,k}^{t-1} \cup \{ind\}$ $\forall m \in \{1,...,M\}$ where $ind$ is the index number of $x^{t}$ \\
		\STATE update the number of points for each sub-codeword: $n_{m,k} \leftarrow n_{m,k} + 1$ $\forall m \in \{1,...,M\}$ \\ \label{alg:update1}
		\STATE update the sub-codeword: $z_{m,k}^{t+1} \leftarrow z_{m,k}^t + \frac{1}{n_{m,k}}(x_m^t - z_{m,k}^t)$ $\forall m \in \{1,...,M\}$ \\ \label{alg:update2}
		\ENDFOR
	\end{algorithmic}
\end{algorithm}

\subsection{Mini-batch Extension}
In addition to processing one streaming data at a time, our framework can also handle a mini-batch of data at a time.
In the case of processing mini-batch of data, we assume that each time we get a new batch of data points $X^t \in {\rm I\!R}^{B \times D}$ where $B$ is the size of the mini-batch. Its objective function is stated as the following:
\begin{equation}\label{3-1}
\begin{split}
\min_{\substack{\mathcal{C}_{1,1}^t,...,\mathcal{C}_{m,k}^t,...,\mathcal{C}_{M,K}^t\\ z_{1,1}^t,...,z_{m,k}^t,...,z_{M,K}^t}} \sum_{m=1}^M\sum_{i=1}^{B}\|x_m^{t,i} - z_{m,k}^t \|^2 \\
\end{split}
\end{equation}
where  $x_m^{t,i}$ is the $i$th streaming data of the current mini-batch in the $m$th subspace at the $t$th iteration and its nearest sub-codeword is $z_{m,k}^t$.

Follow the aforementioned algorithm, we determine the sub-codeword for each sub-vector in each subspace for all data in the mini-batch, and update the counters accordingly for the determined sub-codewords.
Finally, each determined sub-codewords can be updated as $z_{m,k}^{t+1} \leftarrow z_{m,k}^t + \frac{1}{n_{m,k}}\sum_{\{i \in \mathcal{C}_{m,k}^t\}}(x_m^{t,i} - z_{m,k}^t)$ $\forall m \in \{1,...,M\}$ where  $x_m^{t,i}$ is a streaming data point in the $m$th subspace in the $t$th iteration with $z_{m,k}^t$ as its nearest sub-codeword.

\subsection{Partial Codebook Update}
\begin{figure}
	\centering
	\includegraphics[scale=0.31]{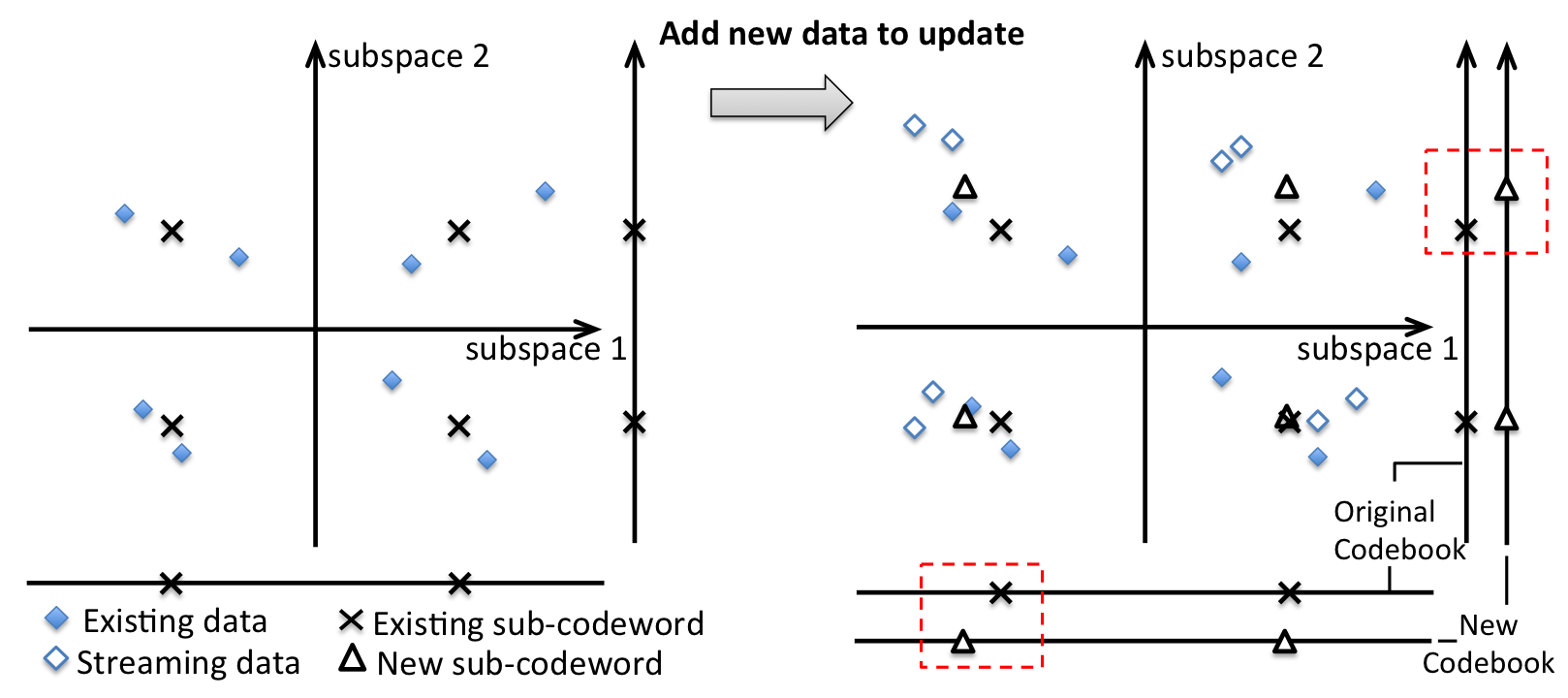}
	\caption{\label{budget_constraint} A schematic figure of online product quantization with budget constraints. There are two subspaces where each subspace has two sub-codewords. After the codebook adapting to the new data, two of the four sub-codewords get hugely changed (highlighted in a red dashed rectangle) and the rest two sub-codewords barely changed.}
\end{figure}
As we mentioned in the introduction, one of the issues is that online indexing model might incur high computational cost in update.
Each new incoming data point might contribute in different significance of changes in different subspaces of nearest sub-codeword update.
An obvious example of this is that, given a new streaming data, one of its sub-vector is far from its nearest sub-codeword and another of its sub-vector is close to its nearest sub-codeword, then the first one contributes more in PQ index update than the second one.
Moreover, mini-batch streaming data sometimes would result in a number of sub-codewords to be updated across different subspaces, and different sub-codewords (within or outside the same subspace) would have different significance of changes in update.
It is worthless to update the sub-codeword when the update change is minimal.
To better illustrate the idea, we show the update of a mini-batch of streaming data in Figure \ref{budget_constraint}.
After assigning the nearest codeword to the new data, it shows that there is one sub-codeword in each subspace that is hugely different from its previous sub-codeword.
The other two sub-codewords barely changed.
Therefore, the update cost can be further reduced by ignoring the update of these sub-codewords that have less significant changes.
Thus we can tackle the issue of possible high computational cost of update as we mentioned in the introduction by employing partial codebook update strategy, which can be achieved by adding one of the two budget constraints: the number of subspaces and the number of sub-codewords to be updated.
\subsubsection{Constraint On Subspace Update}
Each streaming data point is assigned to a sub-codeword in each subspace, so at least one sub-codeword in each subspace needs to be updated at each iteration.
It is possible that the features in some subspaces of the new data have a vital contribution in their corresponding sub-codeword update and the features in some other subspaces have trivial contribution.
Therefore, we target at updating the sub-codewords in the subspaces with significant update changes only.
The subspace update constraint we add to our framework is used to update a subset of subspaces $\phi$.
\begin{equation*}
\phi \subseteq \{1,...,M\}, |\phi| \leq \alpha
\end{equation*}
where $\alpha$ represents the number of subspaces to be updated and $1 \leq \alpha \leq M$.
We apply heuristics to get the optimal solution by selecting the top $\alpha$ subspaces with the most significant update.
The significance of the subspace update can be computed by the sum of the quantization errors of streaming data.
Thus Steps~\ref{alg:update1} and ~\ref{alg:update2} in Algorithm \ref{alg:1} are applied for determined sub-codewords in the selected top $\alpha$ subspaces.

\subsubsection{Constraint On Sub-codeword Update}
Specifically in the case of mini-batch streaming data update, it is likely that each mini-batch consists of different classes of data, which results in a number of sub-codewords in each subspace to be updated.
Similar to subspace update constraint, we propose a sub-codeword update constraint to select a subset of sub-codewords $\psi$ to update.

\begin{equation*}
\psi \subseteq \{(1,1), ...,(m,k),...,(M,K)\}, |\psi| \leq \lambda MK
\end{equation*}

\noindent
where $\lambda$ represents the percentage of sub-codewords to be updated and $0 \leq \lambda \leq 1$.
Each $(m,k)$ tuple in $\psi$ represents the index of the sub-codeword $z_{m,k}$.
Similar to the solution for handling the subspace update constraint, we select the top $\lambda MK$ sub-codewords with the highest quantization error and apply the update steps ~\ref{alg:update1} and ~\ref{alg:update2} to the selected sub-codwords. 

\subsection{Complexity Analysis}
For our online PQ model with $M$ subspaces and $K$ sub-codewords in each subspace, the codebook update complexity for each iteration by a mini-batch of streaming data with size $B$ ($B \ge 1$) and $D$ dimensions is $\mathcal{O}(BKD + BM + BD)$, where these three elements represent the complexity of streaming data encoding, codewords counter update and codewords update respectively.
This complexity can be reduced in two ways.
First, the complexity of streaming data encoding, requiring the distance computation between data and each of the codewords, can be reduced by applying random projection methods such as LSH.
In this case, its complexity can be reduced from $\mathcal{O}(BKD)$ to $\mathcal{O}(BKb)$, where $b$ is the number of bits learned by the random projection method.
Thus both $\mathcal{O}(BKb)$ and $\mathcal{O}(BD)$ are dominant in the online model, especially for high-dimensional dataset (large $D$) with short hash codes (small $b$).
Second, by applying the budget constraints in our model, the complexity of the codewords update $\mathcal{O}(BD)$ will get proportionally decreased to the constraint parameters.
Note the overall update complexity does not depend on the volume of the database at the current iteration.

\section{Loss Bound}
In this section we study the relative loss bounds for our online product quantization, assuming our framework processes streaming data one at a time. 
Traditional analysis for online models are convex problems with vectors as variables.
Our model, on the other hand, is non-convex and has matrices as variables, which makes the analysis non-trivial to be handled.
Moreover, each of the continuously learned codewords may not be consistently matching with each codeword in the best fixed batch model.
For example, a new incoming data may be assigned to the codeword with index 1 in our model but to the codeword with index 3 in the best batch model.
This will make the loss bound even more difficult to study.
Without using the properties of the convex function, we derive the loss bound of our model.

Here we define ``loss" and ``codeword" analogous to ``prediction loss" and ``prediction", and follow the analysis in \cite{DBLP:journals/jmlr/CrammerDKSS06}.
Since all $M$ subquantizers are independent to each other, we focus on the loss bound in a subquantizer $m$.
The instantaneous quantization error function for our algorithm using the $m$th subquantizer during iteration $t$ is defined as:
\begin{equation}\label{loss1}
\begin{split}
&  \ell^{t}(Z_{m}^{t}) = ||x_{m}^{t} - z_{m,k}^t||^2
\end{split}
\end{equation}
where $Z_{m}^{t} \in \mathbb{R}^{\frac{D}{M} \times K}$ is the concatenation of all sub-codewords in the $m$th subspace $[z_{m,1}^{t}, ..., z_{m,K}^{t}]$ and $z_{m,k}^t$ is the closest sub-codeword to $x_{m}^{t}$.
For convenience, we use $\ell_{Z_{m}}^{t}$ to denote $\ell^{t}(Z_{m}^{t})$ given $Z_{m}^t$ as the sub-codewords.
To study the relative loss bound of our online model, by following \cite{DBLP:journals/jmlr/CrammerDKSS06}, we introduce an arbitrary matrix $U$, and we will prove that if there exists a matrix $U$ as the concatenated sub-codewords learned from the best batch model in hindsight (with the minimum quantization error for any $t$), our online model can converge to this batch model. 
Here we assume that $U \in {\rm I\!R}^{\frac{D}{M} \times K}$ concatenates $K$ sub-codewords $[u_{1}, ..., u_{K}]$.
We arrange the order of these sub-codewords in the way that the column vectors of $Z_{m}^t$ and $U$ are paired by minimum distance, \emph{i.e.} $z_{m,k}^t$ matches $u_{k}$ by sub-codeword index to achieve a minimum $\sum_{k=1}^{K} {||z_{m,k}^t - u_{k}||^2}$, so that we will have streaming data assigned to the same index of the sub-codewords in $Z_{m}^t$ and $U$ as likely as possible.
We use $\ell_{U}^{t}$ to denote the quantization error given $U$ as the sub-codewords:
\begin{equation}\label{loss2}
\begin{split}
&  \ell_{U}^{t} = \ell^t(U) = ||x_{m}^{t} - u_{k*}||^2
\end{split}
\end{equation}
where $u_{k*}$ is the closest sub-codeword to $x_{m}^{t}$.

Following Lemma 1 in \cite{DBLP:journals/jmlr/CrammerDKSS06}, we derive the following lemma:
\begin{lemma}\label{lemma1}
	Let $x_{m}^{1}, ..., x_{m}^{T}$ be a sequence of examples for the $m$th subspace where $x_{m}^{t} \in {\rm I\!R}^{\frac{D}{M}}$.
	Assume $||x_{m}^{t} - u_{k}|| - ||x_{m}^{t} - u_{k*}|| \leq \beta$ where $\beta$ is a constant. Let $\mathcal{T}_{n}$ be the set of iteration numbers where $z_{m,k}^{t}$ does not match $u_{k*}$, \emph{i.e.} $||u_{k} - u_{k*}|| \ne 0$. Using the notation provided in Eq.\ref{loss1} and Eq.\ref{loss2}, then
	\begin{equation*}
	\begin{split}
	& \sum_{t=1}^{T}\frac{1}{n_{m,k}^t} ((1 - \frac{1}{n_{m,k}^t})\ell_{Z_{m}}^{t} - \ell_{U}^{t} - ||x_{m}^{t}||^2 - ||u_{k}||^2) \\
	& \indent \indent \indent \indent \indent \indent - \beta \sum_{t \in \mathcal{T}_{n}}\frac{1}{n_{m,k}^t} \leq ||Z_{m}^1 - U||^2
	\end{split}
	\end{equation*}
\end{lemma}
where $Z_{m}^1$ is initialized to be nonzero matrix.

\begin{proof}
Following the proof of Lemma 1 in \cite{DBLP:journals/jmlr/CrammerDKSS06}, define $\Delta_{t}$ to be $||Z_m^t - U||_F^2 - ||Z_{m}^{t+1} - U||_F^2$. We obtain that, 
\begin{equation*}
\sum_{t=1}^{T} \Delta_t \leq ||Z_m^1 - U||_F^2
\end{equation*}

We now try to bound $\Delta_t$. If a sub-codeword of a streaming data point is not selected to be updated due to the budget constraint during iteration $t$, \emph{i.e.} $Z_m^{t+1}$ is the same as $Z_m^t$, then this $\Delta_t = 0$. Therefore, we focus on the iterations for which $||\Delta {Z_m^t}||_F \textgreater 0$. As $z_{m,k}^{t+1} = z_{m,k}^t + \Delta z_{m,k}^t$ where $\Delta z_{m,k}^t = \frac{x_m^t - z_{m,k}^t}{n_{m,k}^t}$ when $z_{m,k}^t$ is the codeword for $x_m^t$, we define $\Phi_{k} (\Delta z_{m,k}^t)$ as the difference between $Z_m^{t+1}$ and $Z_m^{t}$, where $\Delta z_{m,k}^t$ is the difference vector between the $k$th column vector of $Z_m^{t+1}$ and $Z_m^{t}$. We can therefore write $\Delta_t$ as,
\begin{equation*}
\begin{split}
\Delta_t & = ||Z_m^t - U||_F^2 - ||Z_{m}^{t+1} - U||_F^2 \\
& = ||Z_m^t - U||_F^2 - ||Z_{m}^{t} - U + \Phi_{k} (\Delta z_{m,k}^t)||_F^2 \\
& = -2trace((Z_{m}^{t} - U)^{T}\Phi_{k} (\Delta z_{m,k}^t)) \\
& \indent \indent \indent \indent \indent \indent - trace(\Phi_{k} (\Delta z_{m,k}^t)^T \Phi_{k} (\Delta z_{m,k}^t)) \\
& = -2(z_{m,k}^{t} - u_k) \cdot \Delta z_{m,k}^t - ||\Delta z_{m,k}^t||^2 \\
& = \frac{1}{n_{m,k}^t}(-2 z_{m,k}^t \cdot x_m^t + 2 x_m^t \cdot u_k + 2||z_{m,k}^t||^2 \\
& \indent \indent \indent \indent \indent \indent \indent - 2 z_{m,k}^t \cdot u_k - \frac{||x_m^t - z_{m,k}^t||^2}{{n_{m.k}^t}})
\end{split}
\end{equation*}

From Eq.\ref{loss1} and Eq.\ref{loss2}, we obtain that $\ell_{Z_{m}}^{t} - ||x_m^t||^2 = ||z_{m,k}^t||^2 - 2 x_m^t \cdot z_{m,k}^t$ and $-\ell_{U}^{t} \leq 2x_m^t\cdot u_{k*}$.
In addition, $-||u_k||^2 \leq ||z_{m,k}^t||^2 - 2z_{m,k}^t \cdot u_k$,
If $z_{m,k}^t$ matches $u_{k*}$, then $u_{k*}$ is the same as $u_{k}$, then
\begin{equation*}
\Delta_t \geq \frac{1}{n_{m,k}^t}(\ell_{Z_{m}}^{t} - \frac{\ell_{Z_{m}}^{t}}{n_{m,k}^t} - \ell_{U}^{t} - ||x_m^t||^2 - ||u_k||^2)
\end{equation*}

If $z_{m,k}^t$ does not match $u_{k*}$, then based on the assumption $||x_{m}^{t} - u_{k}|| - ||x_{m}^{t} - u_{k*}|| \leq \beta$,
\begin{equation*}
\Delta_t \geq \frac{1}{n_{m,k}^t}(\ell_{Z_{m}}^{t} - \frac{\ell_{Z_{m}}^{t}}{n_{m,k}^t} - \ell_{U}^{t} - \beta - ||x_m^t||^2 - ||u_k||^2)
\end{equation*}

Overall, we obtain our conclusion.
\end{proof}

Following the Theorem 2 in \cite{DBLP:journals/jmlr/CrammerDKSS06} and our Lemma 1, we derive our theorem.
\begin{theorem} \label{the2}
	Let $x_{m}^{1}, ..., x_{m}^{T}$ be a sequence of examples for the $m$th subspace where $x_{m}^{t} \in {\rm I\!R}^{\frac{D}{M}}$ and $||x_{m}^{t}||^2 \leq R^2$.
	Assume that there exists a matrix $U$ such that $\ell_{U}^{t}$ is minimized for all $t$, and $\max_{1\le k \le K} ||u_k||^2 \le F^2$.
	Then, the cumulative quantization error of our algorithm is bounded by
	\begin{equation}\label{bound}
	\begin{split}
	& \sum_{t=1}^{T} \ell_{Z_{m}}^{t} \textless 4(||Z_m^1 - U||_F^2 + \beta \sum_{t \in \mathcal{T}_{n}}\frac{1}{n_{m,k}^t}) \\
	& \indent \indent \indent + 4T(R^2 + F^2) + 8 + 4\sum_{t=1}^{T} \ell_{U}^{t}
	\end{split}	
	\end{equation}
\end{theorem}

\begin{proof}
Since $1 \leq n_t^{m,k} \leq t$, then $\frac{1}{n^t_{m,k}} \leq 1$. Using the facts that $||x_{m}^{t}||^2 \leq R^2$ and $\max_{1\le k \le K} ||u_k||^2 \le F^2$, Lemma 1 implies that, \begin{equation*}
\begin{split}
\sum_{t=1}^{T}\frac{1}{n_{m,k}^t} (1 - \frac{1}{n_{m,k}^t})\ell_{Z_{m}}^{t} - \sum_{t=1}^{T}\ell_{U}^{t} \leq ||Z_{m}^1 - U||_F^2 \\
+ \beta \sum_{t \in \mathcal{T}_{n}}\frac{1}{n_{m,k}^t} + T(R^2 + F^2) + 2
\end{split}
\end{equation*}

Since $\frac{1}{n_{m,k}^t} (1 - \frac{1}{n_{m,k}^t}) \leq \frac{1}{4}$, we get our relative loss bound.
\end{proof}

\begin{remark}
If there exists a concatenated sub-codewords $U$ that is produced by the best fixed batch algorithm in hindsight for any $t$, then $\ell_{U}^{t}$, representing the quantization error of the batch method, can be minimal.
The first term in the inequality (\ref{bound}) is the difference between the initialized codewords  $Z_m^1$ of our online model and the best batch algorithm solution $U$.
The second term represents the summation of the reciprocal of the counter $n_{m,k}$ that belongs to the updated cluster $C_{m,k}^t$ in the iterations when $z_{m,k}^t$ does not match $u_{k*}$, \emph{i.e.}, the streaming data belongs to two different indices of the clusters from our online model and the best batch model.
A tighter bound can be achieved if the initialized sub-codewords is close to the optimal sub-codewords and $z_{m,k}^t$ matches $u_k$ for each iteration.
Since all the terms except for the third one in the inequality (\ref{bound}) are constants, and $(R^2 + F^2)$ is also constant, the cumulative loss bound scales linearly with the number of iterations $T$.
Thus the performance of online PQ model is guaranteed for unseen data.
\end{remark}
\begin{remark}
The proved online loss bound above can be used to obtain the generalization error bound \cite{DBLP:conf/nips/DekelS05}, and the generalization error bound will show that the codewords of our online method are similar to the ones of the batch method.
Therefore, the quantization error of our online model can converge to the error of the batch algorithm. Its corresponding experimental result is shown in Section~\ref{convergence-section}.
\end{remark}

\section{Online Product Quantization over a Sliding Window}
From the cumulative loss bound proved in Theorem~\ref{the2}, we can see that it is important to reduce the number of iterations when $z_{m,k}^t$ does match $u_k$ in order to achieve a tighter bound.
One reason for the mismatched case is the tradeoff between update efficiency and quantization error.
The old data and the new data that belong to the same codeword in the online model may be generated from different data distributions.
Thus it gives us the insight to consider the sliding window approach of handling the streaming data, where the updated codewords will be emphasized on the real-time data and the expired data will be removed.
Using this approach allows us to reduce the effect of the codewords mismatching iterations.

In the context of data streams, data distribution evolves over time.
Some applications may aim at capturing the real-time behaviour of the data streams and thus emphasize the most recent data points in the stream.
As the data stream evolves, some data will be expired based on their arrival times.
Leveraging a sliding window in the model allows us to reflect the real-time and evolving behaviour of the data streams in order to facilitate the most recent data maintenance.
Data insertion is performed to the online indexing model using the recent data in the sliding window, and data deletion removes the contributions of the codebook made by the expired data.
In this section, we present the online PQ model over a time-based sliding window with both data insertion and deletion supported.

\begin{figure}
	\centering
	\includegraphics[scale=0.45]{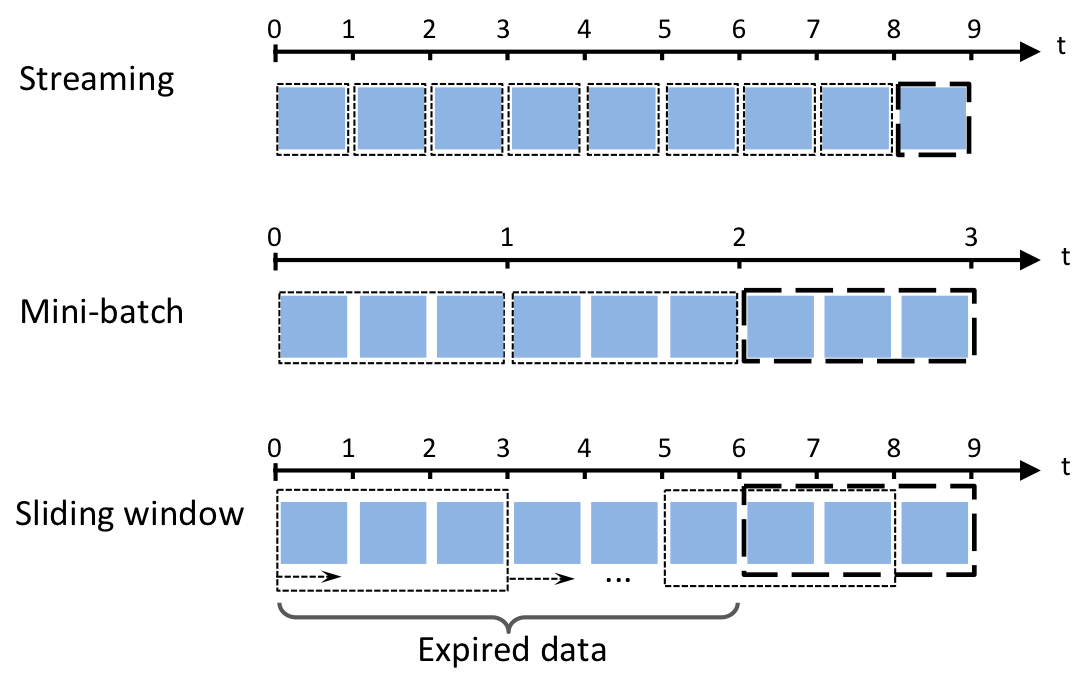}
	\caption{\label{streaming_comparison} Approaches of Handling Data Streams. Streaming: data streams one at a time. Mini-batch: a mini-batch of data with size 3 is processed by the model at each iteration t. Sliding window: a moving window with size 3 applied on continuously changing data.}
\end{figure}

\subsection{Online Product Quantization with Data Insertion and Deletion}
Assume we are given a sliding window of size $L$.
For window at the $t$th iteration, it consists of a stream of data $x^{t,1}, x^{t,2}, ...,x^{t,L}$, where $x^{t,L}$ is the newly inserted data point to the window at the current iteration, and $x^{t-1,1}$ is just expired and removed from the window.
Therefore, the objective function of the Online PQ over a sliding window at the $t$th iteration is stated as the following:
\begin{equation}\label{5}
\begin{split}
\min_{\substack{\mathcal{C}_{1,1}^t,...,\mathcal{C}_{m,k}^t,...,\mathcal{C}_{M,K}^t\\ z_{1,1}^t,...,z_{m,k}^t,...,z_{M,K}^t}} \sum_{m=1}^M\sum_{i=1}^{L}\|x_m^{t,i} - z_{m,k}^t \|^2 \\
\end{split}
\end{equation}
where  $x_m^{t,i}$ is the $i$th streaming data of the window in the $m$th subspace at the $t$th iteration and its nearest sub-codeword is $z_{m,k}^t$.
To emphasize the real-time data in the stream, we want our model only affected by the data in the sliding window at the current iteration.
Therefore, each time a new data streaming into the system, it moves to the sliding window.
We first update the codebook by adding the contributions made by the new data.
Correspondingly, the oldest data in the sliding window will be removed.
We tackle the issue of data expiry by deleting the contribution to the codebook made by the data point that is just removed from the window.
The solution of online PQ over a sliding window to handle insertion and deletion to the codebook is shown in Algorithm \ref{alg:2}:
\begin{algorithm}
	\caption{\label{alg:2} Online PQ over a Sliding Window}
	\begin{algorithmic}[1]
		\STATE initialize PQ with the $M*K$ sub-codewords $z_{1,1}^0,...,z_{m,k}^0,...,z_{M,K}^0$ \\
		\STATE initialize $C_{1,1}^0,...,C_{m,k}^0,...,C_{M,K}^0$ to be the cluster sets that contain the index of the initial data that belong to the cluster
		\STATE create counters $n_{1,1},...,n_{m,k},...,n_{M,K}$ for cluster and initialize each $n_{m,k}$ to be the number of initial data points assigned to the corresponding $C_{m,k}^0$ \\
		\STATE initialize the size of the sliding window $L$
		\FOR{$t=1, 2, 3, ...$}
		\STATE remove $x^{t-1,1}$ from the sliding window
		\STATE $x^{t-1,i} (2 \le i \le L)$ from the $(t-1)$th sliding window is now represented as $x^{t,i-1}$ at the current iteration
		\STATE insert the new data $x^{t,L}$ to the sliding window at position $L$\\
		\STATE partition the newly inserted data $x^{t,L}$ and the deleted one $x^{t-1,1}$ into $M$ subspaces $[x_1^{t,L},...,x_M^{t,L}]$ and $[x_1^{t-1,1},...,x_M^{t-1,1}]$, respectively\\
		\textbf{Insertion:}
		\FOR{$m=1,...,M$}
		\STATE determine and assign the nearest sub-codeword $z_{m,k}^t$ for each sub-vector $x_m^{t,L}$\\
		\STATE update the cluster set $C_{m,k}^t \leftarrow C_{m,k}^{t-1} \cup \{ind\}$ where $ind$ is the the index number of $x^{t,L}$
		\STATE update the counter for $C_{m,k}^t$: $n_{m,k} \leftarrow n_{m,k} + 1$ \\
		\STATE update the sub-codeword: $z_{m,k}^{t+1} \leftarrow z_{m,k}^t + \frac{1}{n_{m,k}}(x_m^{t,L} - z_{m,k}^t)$ \\
		\ENDFOR \\
		\textbf{Deletion:}
		\FOR{$m=1,...,M$}
		\STATE determine and assign the nearest sub-codeword $z_{m,k}^t$ for each sub-vector $x_m^{t-1,1}$\\		
		\STATE update the cluster set $C_{m,k}^t \leftarrow C_{m,k}^{t-1} \setminus \{ind\}$ where $ind$ is the the index number of $x^{t-1,1}$
		\STATE update the counter for $C_{m,k}^t$: $n_{m,k} \leftarrow n_{m,k} - 1$ \\
		\STATE update the sub-codeword: $z_{m,k}^{t+1} \leftarrow z_{m,k}^t - \frac{1}{n_{m,k}}(x_m^{t-1,1} - z_{m,k}^t)$ \\
		\ENDFOR
		\ENDFOR
	\end{algorithmic}
\end{algorithm}

\subsection{Connections among Online PQ Algorithms}
Figure \ref{streaming_comparison} compares three different approaches of handling data streams for our online PQ model.
Streaming Online PQ processes new data one at a time. Mini-batch Online PQ processes a mini-batch of data at a time. If the size of the mini-batch is set to 1, Streaming Online PQ is the same as Mini-batch Online PQ. Online PQ over a Sliding Window involves the deletion of the expired data that is removed from the moving window at the current iteration. If the size of the sliding window is set to be infinite, no data deletion will be performed and it is the same as Streaming or Mini-batch Online PQ depending on the size of the new data to be updated at each iteration.

\section{Experiments}
We conduct a series of experiments on several real-world datasets to evaluate the efficiency and effectiveness of our model.
In this section, we first introduce the datasets used in the experiments.
We then show the convergence of our online PQ model to the batch PQ method in terms of the quantization error, and then compare the online version and the mini-batch version of our online PQ model.
After that, we analyze the impact of the parameters $\alpha$ and $\lambda$ in update constraints.
Finally, we compare our proposed model with existing related hashing methods for different applications.

\subsection{Datasets and evaluation criterion}
There are one text dataset, four image datasets and two video datasets employed to evaluate the proposed method.
\textbf{20 Newsgroups Data} (News20) \cite{DBLP:conf/icml/Lang95} consists of chronologically ordered 18,845 newsgroup messages.
\textbf{Caltech-101} \cite{DBLP:conf/cvpr/LiFP04} consists of 9144 images and each image belongs to one of the 101 categories.
\textbf{Half dome} \cite{DBLP:conf/cvpr/WinderB07} includes 107,732 image patches obtained from Photo Tourism reconstructions from Half Dome (Yosemite).
\textbf{Sun397} \cite{DBLP:conf/cvpr/XiaoHEOT10} contains around 108K images in 397 scenes.
\textbf{ImageNet} \cite{DBLP:conf/cvpr/DengDSLL009} has over 1.2 million images with a total of 1000 classes.
\textbf{YoutubeFaces}\footnote{https://www.cs.tau.ac.il/~wolf/ytfaces/} contains 3,425 videos of 1,595 different people, with a total of 621,126 frames.
\textbf{UQ\_VIDEO}\footnote{http://staff.itee.uq.edu.au/shenht/UQ\_VIDEO/} consists of 169,952 videos with 3,305,525 frames in total.
We use 300-D doc2vec features to represent each news article in News20 and 512-D GIST features to represent each image in the four image datasets. We use two different features, 480-D Center-Symmetric LBP (CSLBP) and 560-D Four-Patch LBP (FPLBP) to represent each frame in YoutubeFaces.
162-D HSV feature is used in UQ\_VIDEO dataset.
Table \ref{tbl:dataset} shows detailed statistical information about datasets used in evaluation.

We measure the performance of our proposed model by the model update time and the search quality measurement recall@R adopted in \cite{DBLP:journals/pami/JegouDS11}.
We use recall@20 which indicates that fraction of the query for which the nearest neighbor is in the top 20 retrieved images by the model.

\begin{table}[]
	\centering
	\caption{Detailed datasets information}
	\label{tbl:dataset}
	\begin{tabular}{|c|c|c|l|c|}
		\hline
		\textbf{Dataset} & \textbf{Class no.} & \multicolumn{2}{c|}{\textbf{Size}} & \textbf{Feature} \\ \hline
		News20 & 20 & \multicolumn{2}{c|}{18,845} & Doc2vec (300) \\ \hline
		Caltech-101 & 101 & \multicolumn{2}{c|}{9,144} & GIST (512) \\ \hline
		Half dome & 28,086 & \multicolumn{2}{c|}{107,732} & GIST (512) \\ \hline
		Sun397 & 397 & \multicolumn{2}{c|}{108,753} & GIST (512) \\ \hline	
		ImageNet & 1000 & \multicolumn{2}{c|}{1,281,167} & GIST (512) \\ \hline
		\textbf{Video Dataset} & \textbf{Video} & \multicolumn{2}{c|}{\textbf{Frame}} & \textbf{Feature} \\ \hline
		YoutubeFaces (CSLBP) & 3,425 & \multicolumn{2}{c|}{621,126} & CSLBP (480) \\ \hline
		YoutubeFaces (FPLBP) & 3,425 & \multicolumn{2}{c|}{621,126} & FPLBP (560) \\ \hline
		UQ\_VIDEO & 169,952 & \multicolumn{2}{c|}{3,304,554} & HSV (162) \\ \hline
	\end{tabular}
\end{table}

\subsection{Convergence}\label{convergence-section}

\begin{figure}[!htbp]
	\centering
	\includegraphics[scale=0.5]{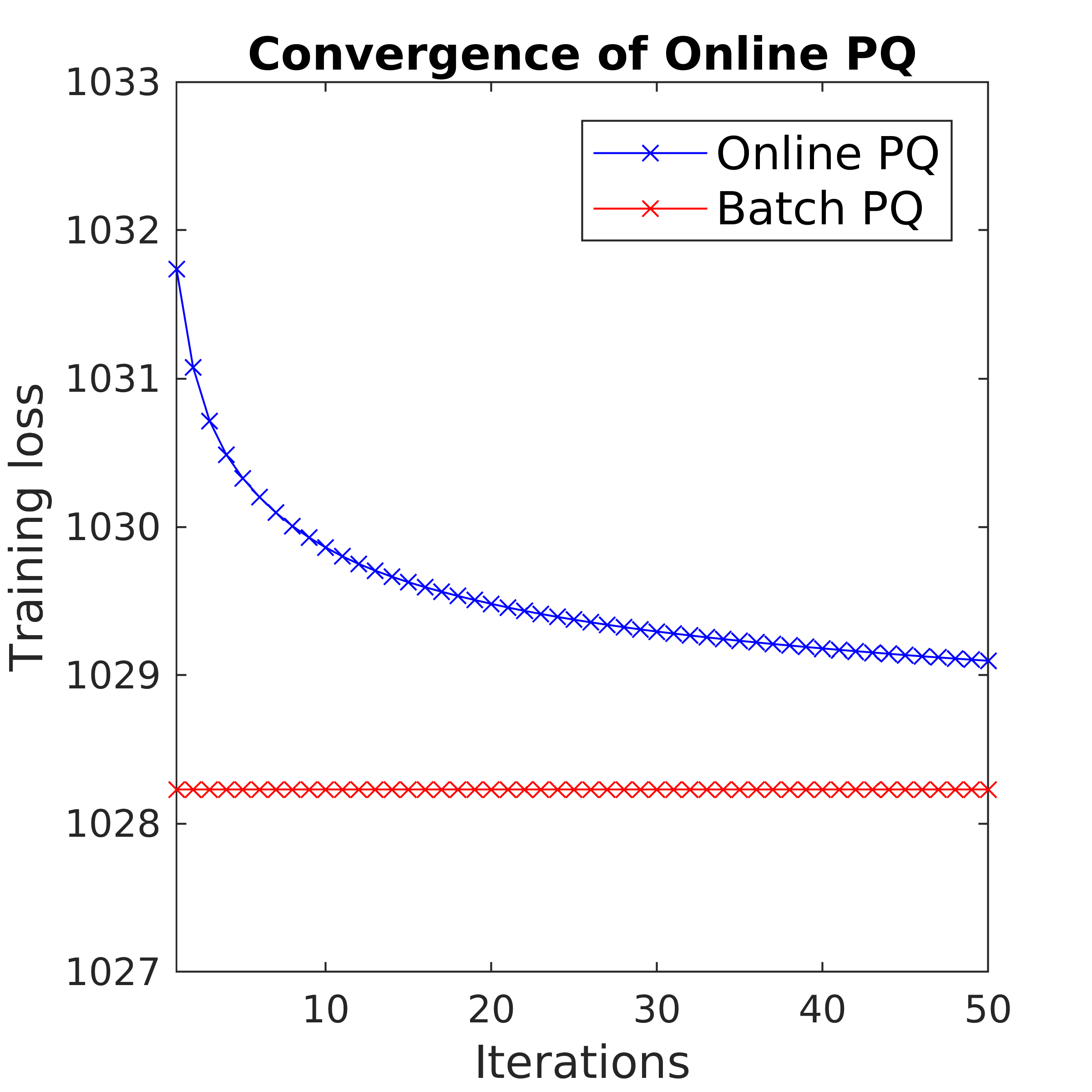}
	\caption{\label{convergence} Convergence of online PQ using ImageNet dataset. Effective iterations are shown on the x-axis.}
\end{figure}

The data instances in the entire dataset are input sequentially to our online PQ model. We run our algorithm for 50 effective iterations\footnote{Running one pass using the entire dataset is an effective iteration}. To show the convergence of our online model, we compare its training loss at each iteration with the one of the batch PQ method. The training loss is computed as the averaged quantization error for all data points in one pass. Figure \ref{convergence} shows that the training loss of our online model converges to the one of the batch model, implying that codewords learned from the online PQ model are similar to the ones learned from the batch PQ approach. Therefore, the performance of the online PQ model converges to the batch PQ performance.

\begin{figure}[!htbp]
	\centering
	\includegraphics[scale=0.41]{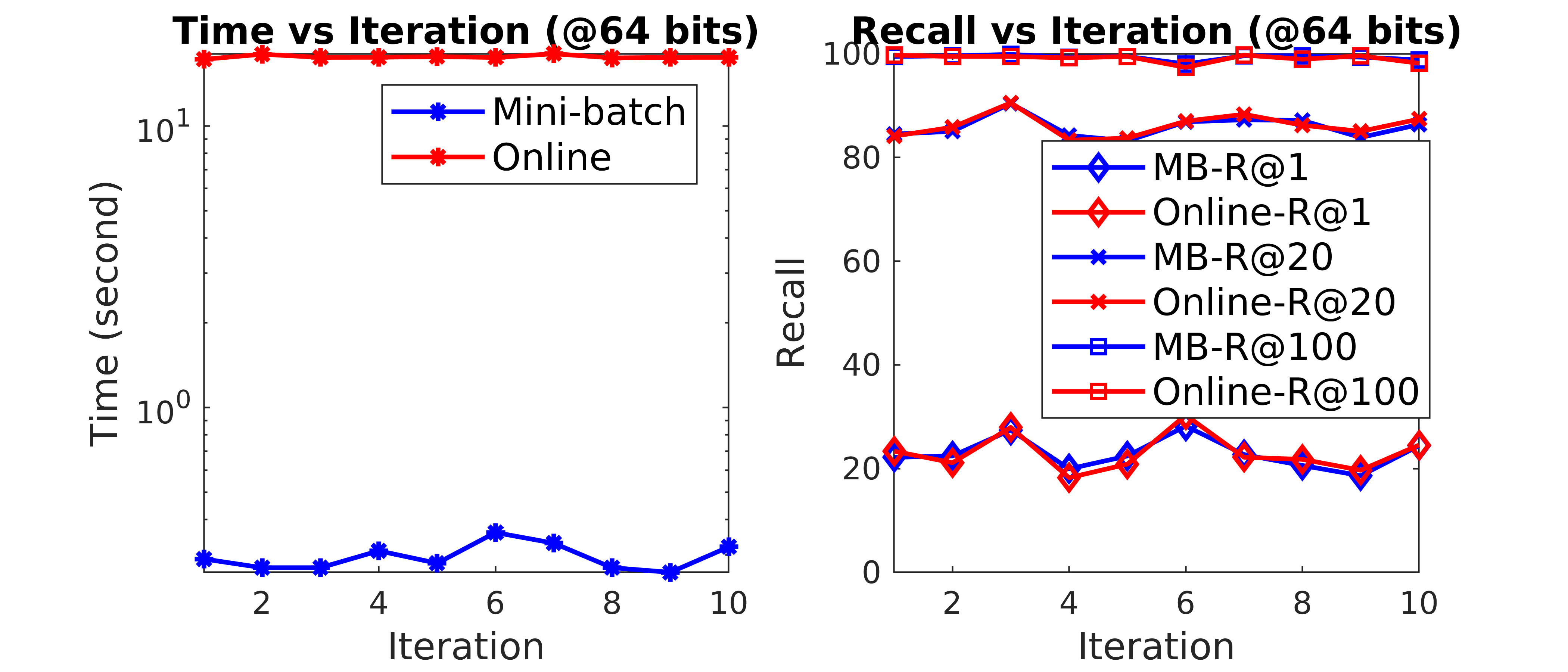}
	\caption{\label{online_vs_minibatch} The left figure shows the update time for each iteration of update. The time of the online version for each iteration sums up the update time of the streaming data corresponding to the ones in the mini-batch. The right figure shows the recall@1, 20 and 100 for each iteration. }
\end{figure}

\subsection{Online vs mini-batch}
In real-world applications, streaming data might be processed one at a time or in small batches.
For example, real-time topic detection in streaming media can be applied on texts, images or videos. If the streaming data is Twitter post, it might be processed one at a time. If the media is video, then the streaming data can be processed in mini-batches of video frames.
Our model can process streaming data either one at a time or in mini-batches.
We compare these two versions of our model on Caltech-101 dataset.
For our model, we use $M = 8$ and $K = 256$.
Since the memory cost of storing each codeword is $M\ceil{\log_{2} K}$ bits \cite{DBLP:journals/pami/JegouDS11}, where $\ceil{.}$ is the ceiling function that maps the value to its nearest integer up, then the number of bits used in our model is 64.
We split the data into twelve groups, where one of the groups is used for learning the codebook, one of the groups is used as the query and each one of the rest of the ten groups is used to update the original codebook, so that we have the performance for ten iterations.
Figure \ref{online_vs_minibatch} shows the comparison of the online version and the mini-batch (MB) version in update time and recall@R measurements.
It indicates that the mini-batch version takes much less update time than the online version but they have similar search quality.
Therefore, we adopt mini-batch version of our model in the rest of the experiments.

\begin{figure}[!htbp]
	\centering
	\includegraphics[scale=0.565]{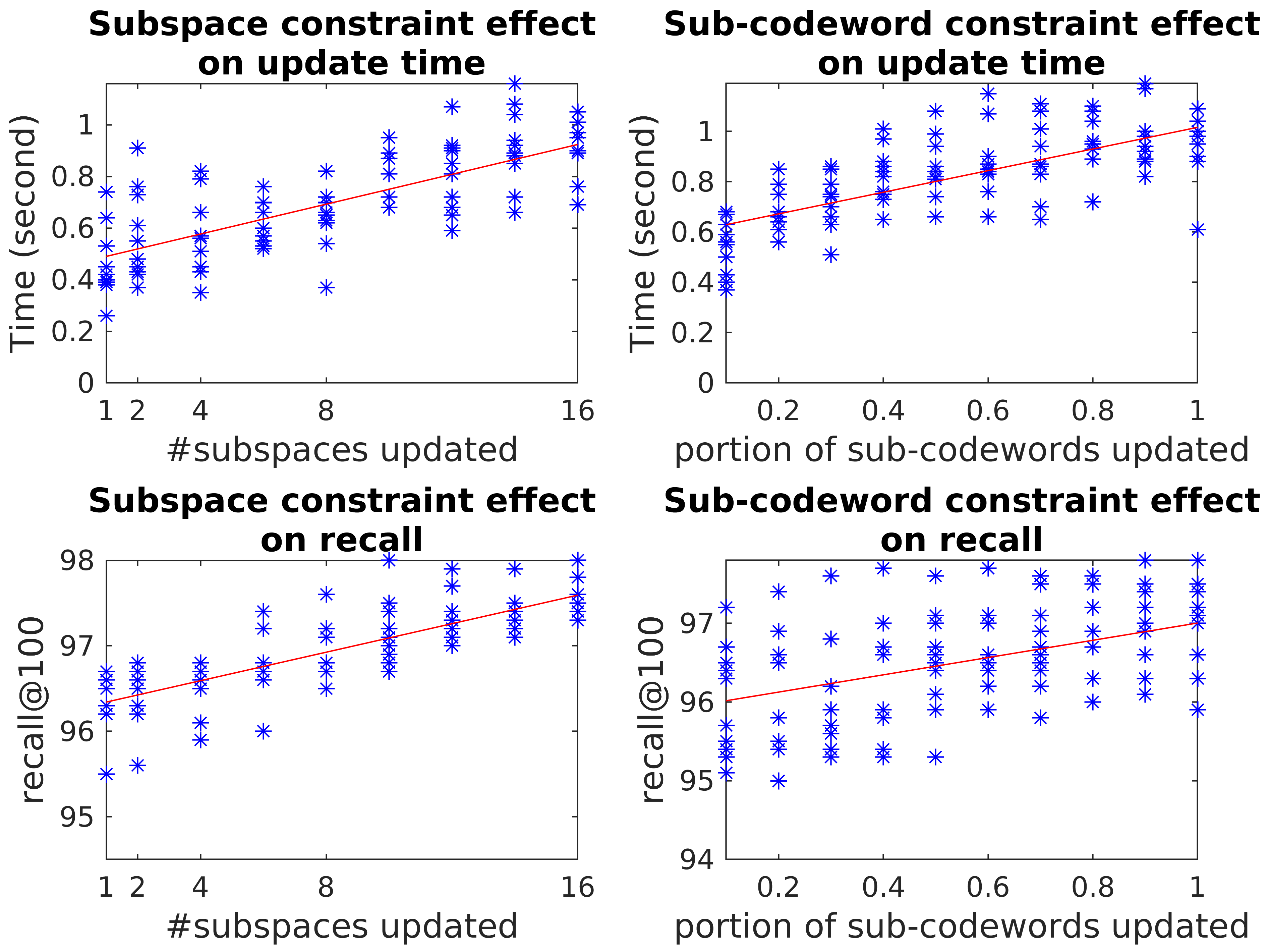}
	\caption{\label{tradeoff} Trade-offs between update time cost and the search accuracy. The first column shows the impact of the subspace update constraint. The second column shows the impact of the sub-codeword update constraint. The red line is the reference line for the scatter plot.}
\end{figure}

\subsection{Update time complexity vs search accuracy}

There are two budget constraints we proposed for the codebook update to further reduce the time cost: number of subspaces and number of sub-codewords to be updated.
In this experiment, we evaluate the impact of these two constraints and the trade-offs between update time cost and the search quality using a synthetics dataset. 
We randomly sampled 12000 data points with 128-D features from multivariate normal distribution.
We use recall@50 as the performance measurement which indicates that fraction of the query for which the nearest neighbor is in the top 50 retrieved data points by the model.
We set $M = 16$ and $K = 256$, and vary the number of updated subspaces from 1 to 16 and the portion of updated sub-codewords from 0.1 to 1 respectively.
We split the dataset evenly into twelve groups and set one of the groups as the learning set to learn the codebook and another one as the query set, and use each of the rest of ten groups to update the learned codebook and record the update time cost and the search accuracy for 10 times while varying the update constraints.
From Figure \ref{tradeoff}, we observe that the search quality and update time cost strongly depend on these two update constraints. As we increase the number of subspaces or the portion of sub-codewords to be updated, the update time cost is increasing, along with the search accuracy.
Therefore, higher update time cost is required for better search accuracy.

\subsection{Baseline methods}
To verify that our online PQ model is time-efficient in update and effective in nearest neighbor search, we make comparison with several related online indexing and batch learning methods.
We evaluate the performance in both search accuracy and update time cost.
We select SSBC \cite{DBLP:journals/corr/GhashamiA15} and OSH \cite{DBLP:conf/cvpr/LengWC0L15} as two of the baseline methods, as they are the only unsupervised online indexing methods to the best of our knowledge.
Specifically, SSBC is only applied on two of our smallest datasets, News20 and Caltech-101, as it takes a significant amount of time to train.
In addition, two supervised online indexing methods OH \cite{DBLP:conf/ijcai/HuangYZ13,huang2017online} and AdaptHash \cite{DBLP:conf/iccv/CakirS15} are selected, with the top 5 percentile nearest neighbors in Euclidean space as the ground truth following the setting as in \cite{DBLP:conf/cvpr/LengWC0L15}.
Further, five batch learning indexing methods are selected.
They are all unsupervised data-dependent methods: PQ \cite{DBLP:journals/pami/JegouDS11}, spectral hashing (SH) \cite{DBLP:conf/nips/WeissTF08}, IsoH \cite{DBLP:conf/nips/KongL12}, ITQ \cite{DBLP:conf/cvpr/GongL11} and KMH \cite{DBLP:conf/cvpr/HeWS13}.
Each of these methods is compared with online PQ in two ways.
The first way uses all the data points seen so far to retrain the model (batch) at each iteration. The second way uses the model trained from the initial iteration to assign the codeword to the streaming data in all the rest of the iterations (no update).
We do not apply the batch learning methods on our large-scale datasets, UQ\_VIDEO and ImageNet, as it takes too much retraining time at each iteration.

\subsection{Object tracking and retrieval in a dynamic database}

In many real-world applications, data is continuously generated everyday and the database needs to get updated dynamically by the newly available data. 
For example, news articles can be posted any time and it is important to enhance user experience in news topic tracking and related news retrieval.
New images with new animal species may be inserted to the large scale image database.
Index update needs to be supported to allow users to retrieve images with expected animal in a dynamically changing database.
A live video or a surveillance video may generate several frame per second, which makes the real-time object tracking or face recognition a crucial task to solve.
In this experiment, we evaluate our model on how it handles dynamic updates in both time efficiency and search accuracy in three different types of data: text, image and video.

\subsubsection{Setting}\label{setting}
\begin{table*}[!htbp]
	\centering
	\caption{\small Number of iterations and average mini-batch size for each dataset}
	\label{tbl:streaming-stats}
	\begin{tabular}{|c|c|c|c|c|c|c|c|}
		\hline
		\textbf{Dataset} & \textbf{News20} & \textbf{Caltech-101} & \textbf{Half dome} & \textbf{Sun397} & \textbf{ImageNet} & \textbf{YoutubeFaces} & \textbf{UQ\_VIDEO} \\ \hline
		\textbf{Iteration No.} & 20 & 12 & 12 & 21 & 101 & 67 & 25 \\ \hline
		\textbf{Avg mini-batch size} & 942.25 & 762 & 8,977.7 & 5,178.7  & 12,684.8 & 8,628.4 & 132,182 \\ \hline
	\end{tabular}
\end{table*}

\begin{figure}
	\centering
	\begin{subfigure}[b]{0.5\textwidth}
		\centering
		\includegraphics[width=1\textwidth]{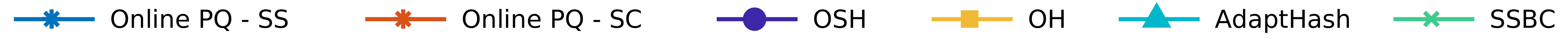}    
	\end{subfigure}
	\begin{subfigure}[b]{0.48\textwidth}
		\centering
		\centerline{\includegraphics[width=9cm]{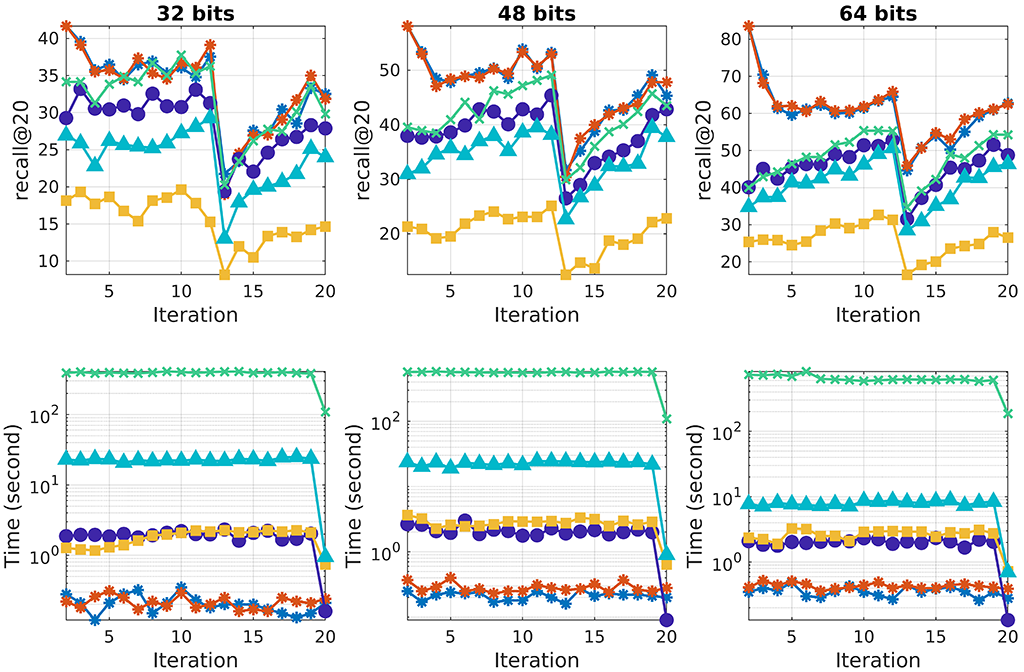}}    
	\end{subfigure}
	
	\caption[]
	{\small Recall@20 performance (1st row) and Update time cost (2nd row) comparison against online hashing methods at each iteration for News20 using different number of bits. 1st column: 32 bits. 2nd column: 48 bits. 3rd column: 64 bits. Time cost is in log scale.} 
	\label{news20_diff_bits}
\end{figure}

For each dataset, we split data into several mini-batches, and stream each mini-batch into the models in turn.
Since the news articles in News20 dataset are ordered by date, we stream the data to the models in its chronological order.
Image datasets consist of different classes.
To simulate data evolution over time, we stream images by classes and each pair of two consecutive mini-batches have half of the images from the same class.
In the video datasets, videos are ordered by their labels, such as the videos belonging to the same person are grouped together and then sets of videos stream to the models in turn.
For text and image datasets, we have dynamic query set in each iteration.
After the initial mini-batch of data used to initialize the model, each time a mini-batch of streaming data comes, we use each of them as the query to search for the nearest neighbors from the existing database, and then update the current model by accommodating these streaming data.
For YoutubeFace video dataset, we have a randomly sampled fixed set of queries consisting 226 videos with 43,020 frames in total.
UQ\_VIDEO dataset provides a fixed set of 24 videos with 902 frames in total.
Table \ref{tbl:streaming-stats} shows detailed information about the data streams.

In this experiment, we compare our method with subspace update constraint (Online PQ - SS) and sub-codeword update constraint (Online PQ - SC) to several batch mode methods and online methods for the task of continuous update in a dynamic setting.
In our model, we set $M = 8$ and $K = 256$. then the number of bits used for vector encoding is 64. 
We constraint the number of the updated subspaces $\alpha$ to be 4 and the portion of the updated sub-codewords $\lambda$ to be 0.5 respectively.
The first batch is used for codebook initialization and the rest of the batches are used to update the existing codebook one at a time.
All the key parameters in the baseline methods are set to the ones recommended in the corresponding papers.
All the  methods compared are implemented in Matlab provided by the authors and all experiments are conducted on a workstation with a 3.10GHZ Intel CPU and 120GB main memory running on a Linux platform.
We show that how our method performs using different number of bits (32, 48 and 64 bits) compared with online baseline methods for News20 dataset. For all other experiments, we use 64 bits for vector encoding in all of the comparison models for fair comparisons.

\begin{figure*}
	\centering
	\begin{subfigure}[b]{1\textwidth}
		\centering
		\includegraphics[width=0.8\textwidth]{online_legend.png}    
	\end{subfigure}
	\begin{subfigure}[b]{0.23\textwidth}
		\centering
		\includegraphics[width=4.2cm,height=3.6cm]{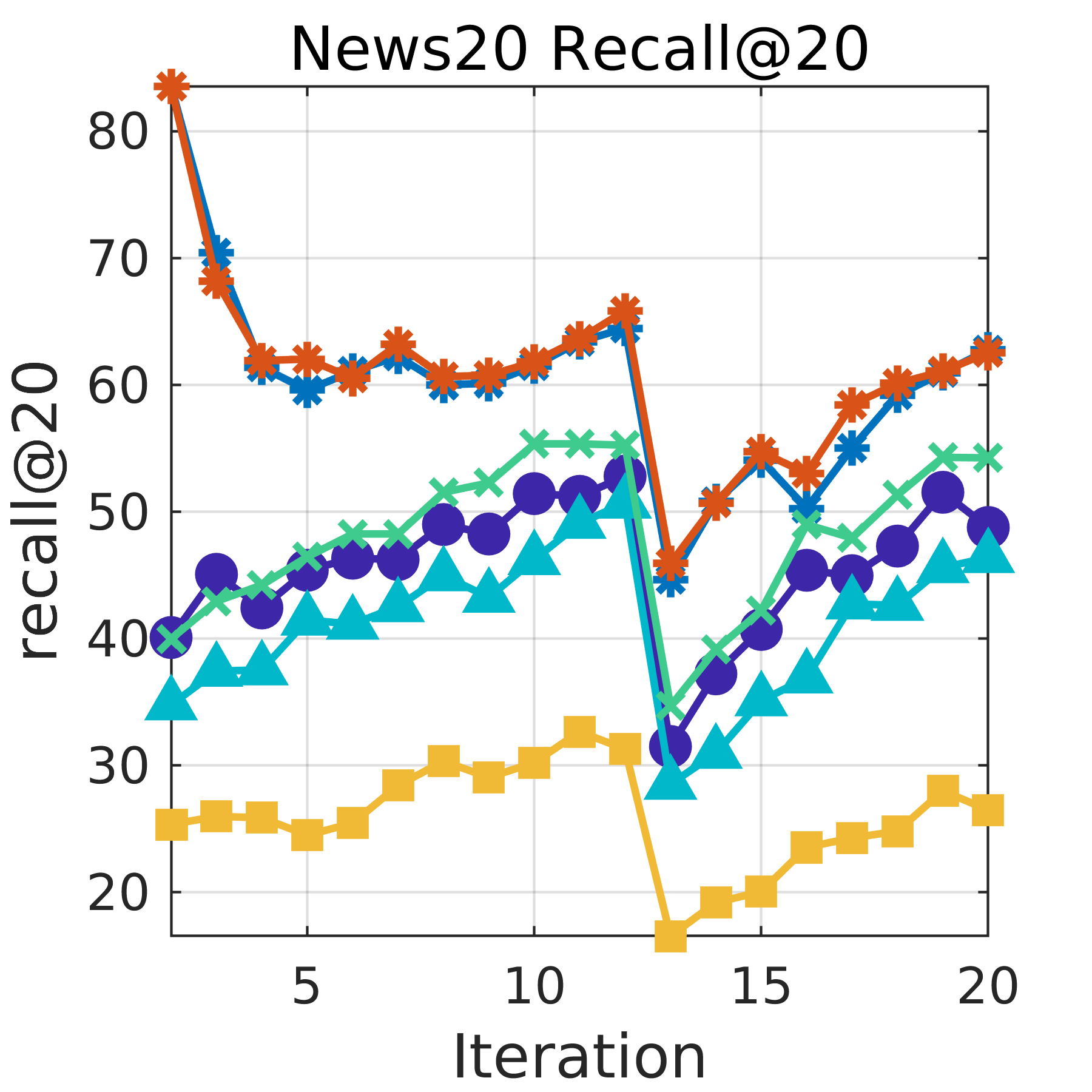}    
	\end{subfigure}
	\begin{subfigure}[b]{0.23\textwidth}  
		\centering 
		\includegraphics[width=4.2cm,height=3.6cm]{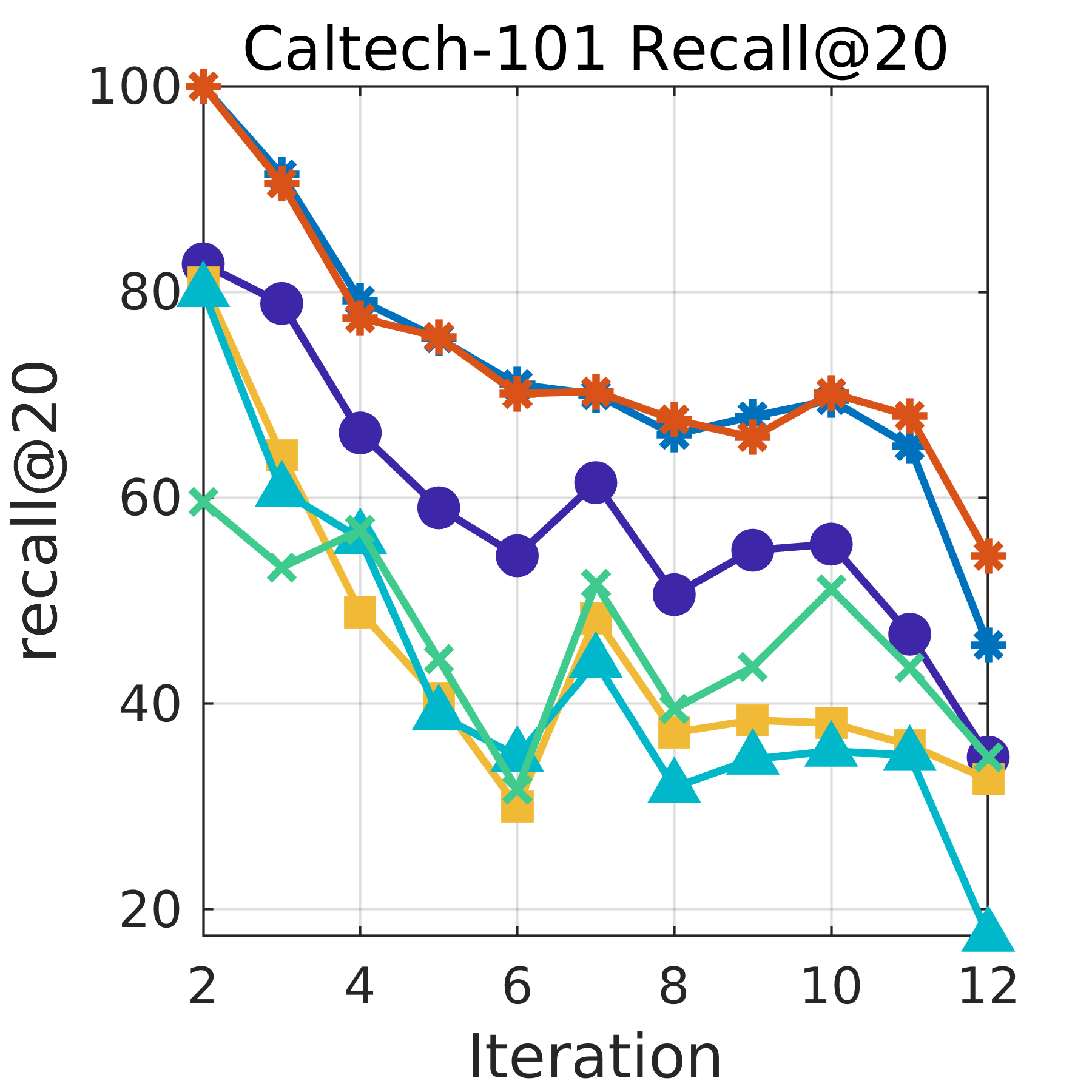}
	\end{subfigure}
	\begin{subfigure}[b]{0.23\textwidth}   
		\centering 
		\includegraphics[width=4.2cm,height=3.6cm]{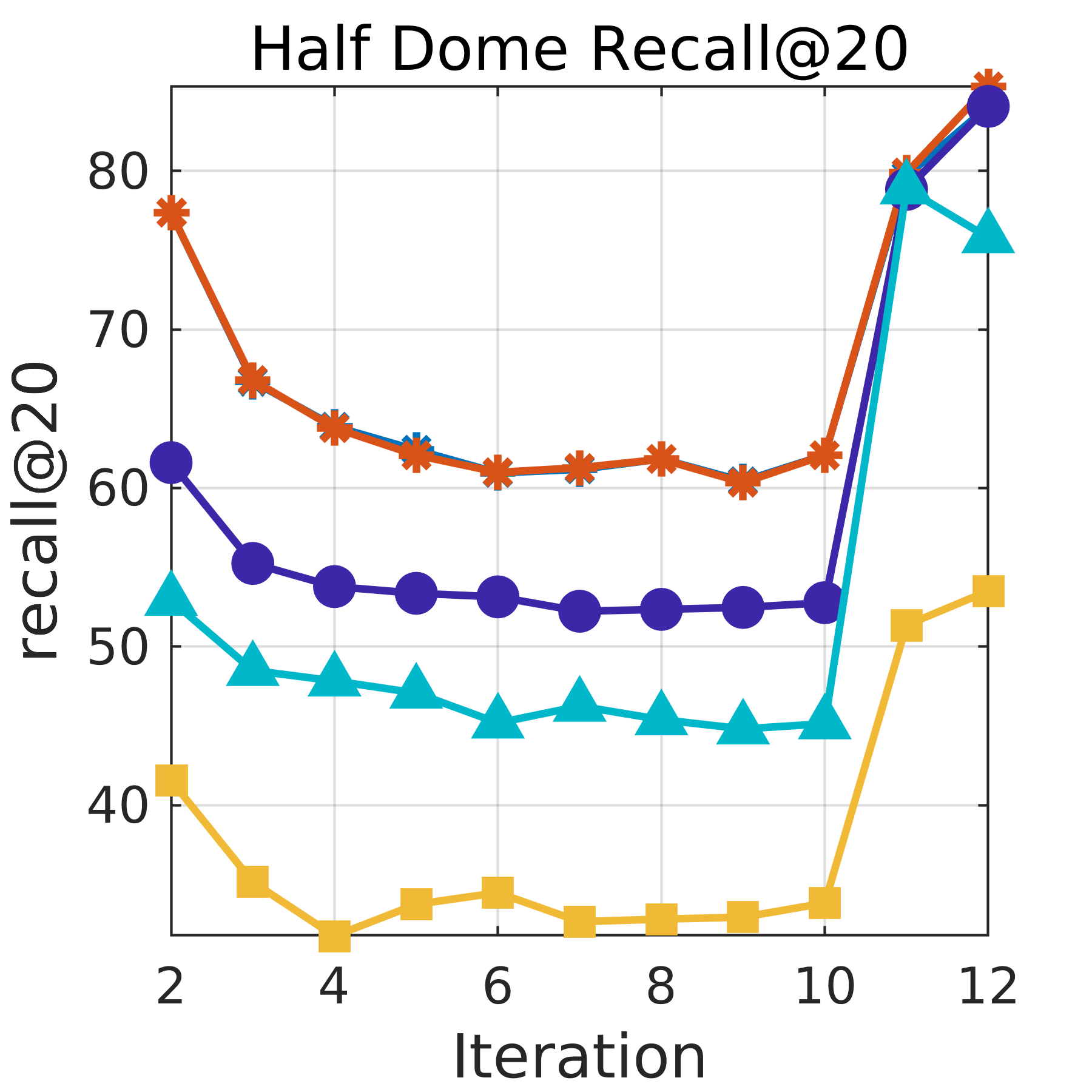}
	\end{subfigure}
	\begin{subfigure}[b]{0.23\textwidth}   
		\centering 
		\includegraphics[width=4.2cm,height=3.6cm]{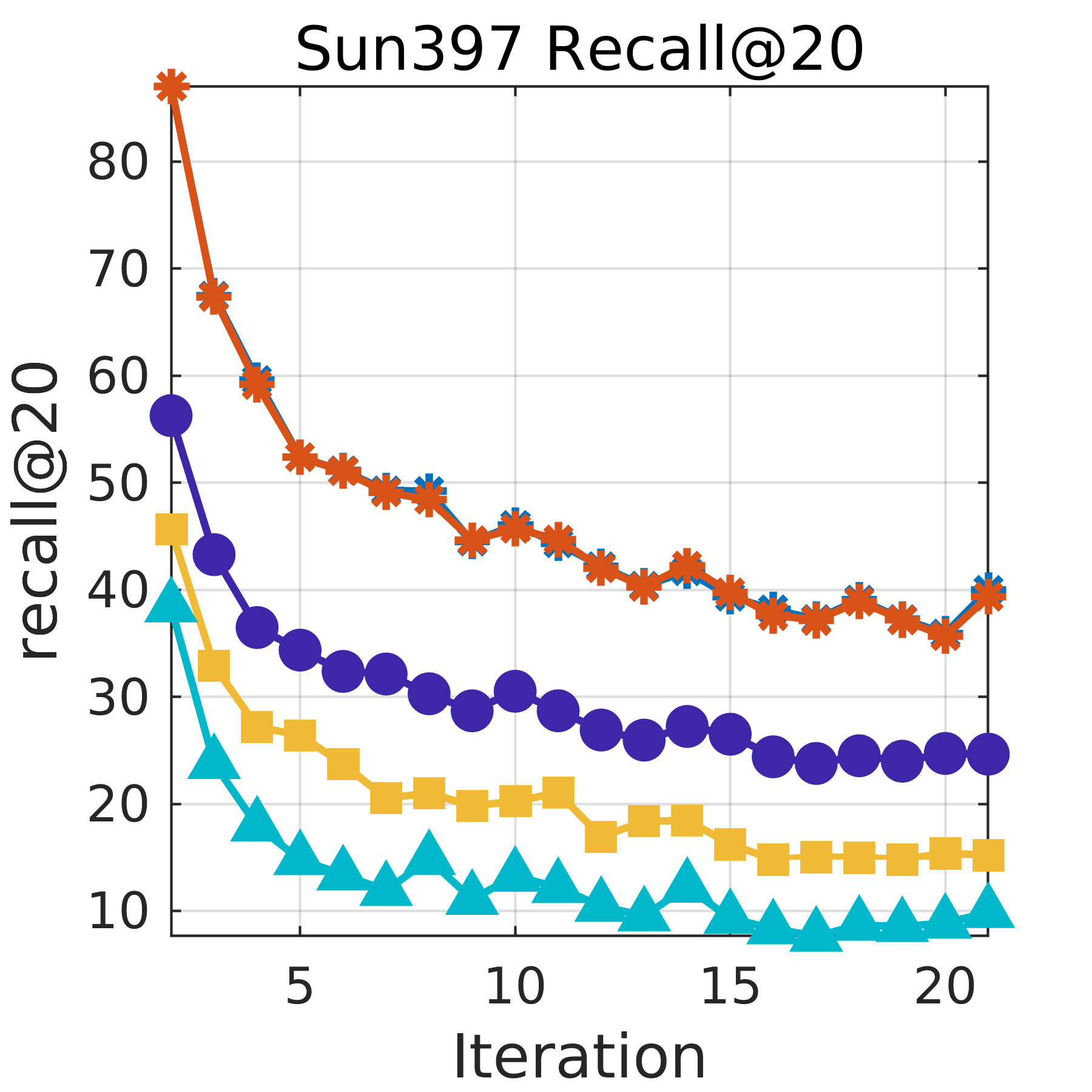}
	\end{subfigure}
	\begin{subfigure}[b]{0.23\textwidth}
		\centering
		\includegraphics[width=4.2cm,height=3.6cm]{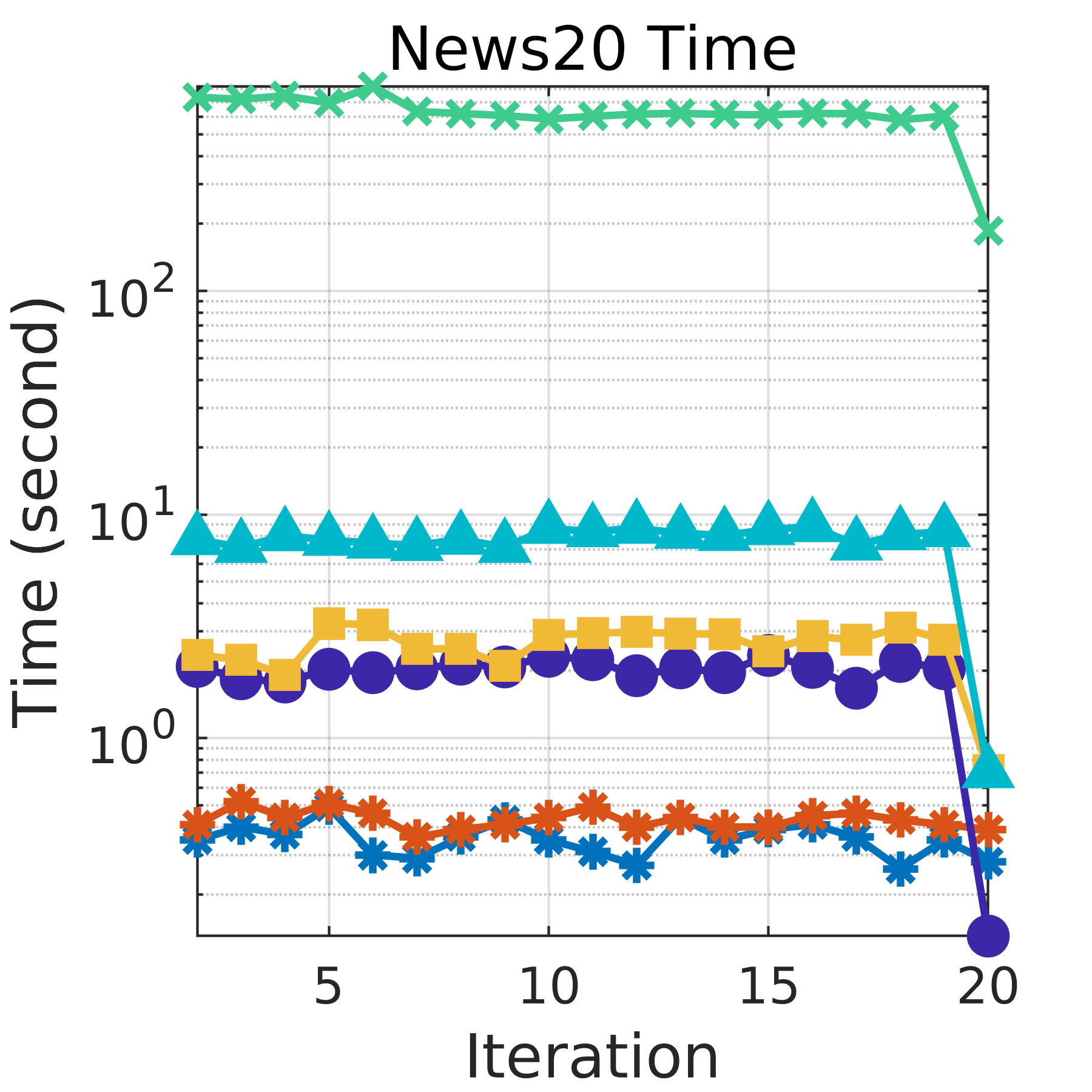}    
	\end{subfigure}
	\begin{subfigure}[b]{0.23\textwidth}  
		\centering 
		\includegraphics[width=4.2cm,height=3.6cm]{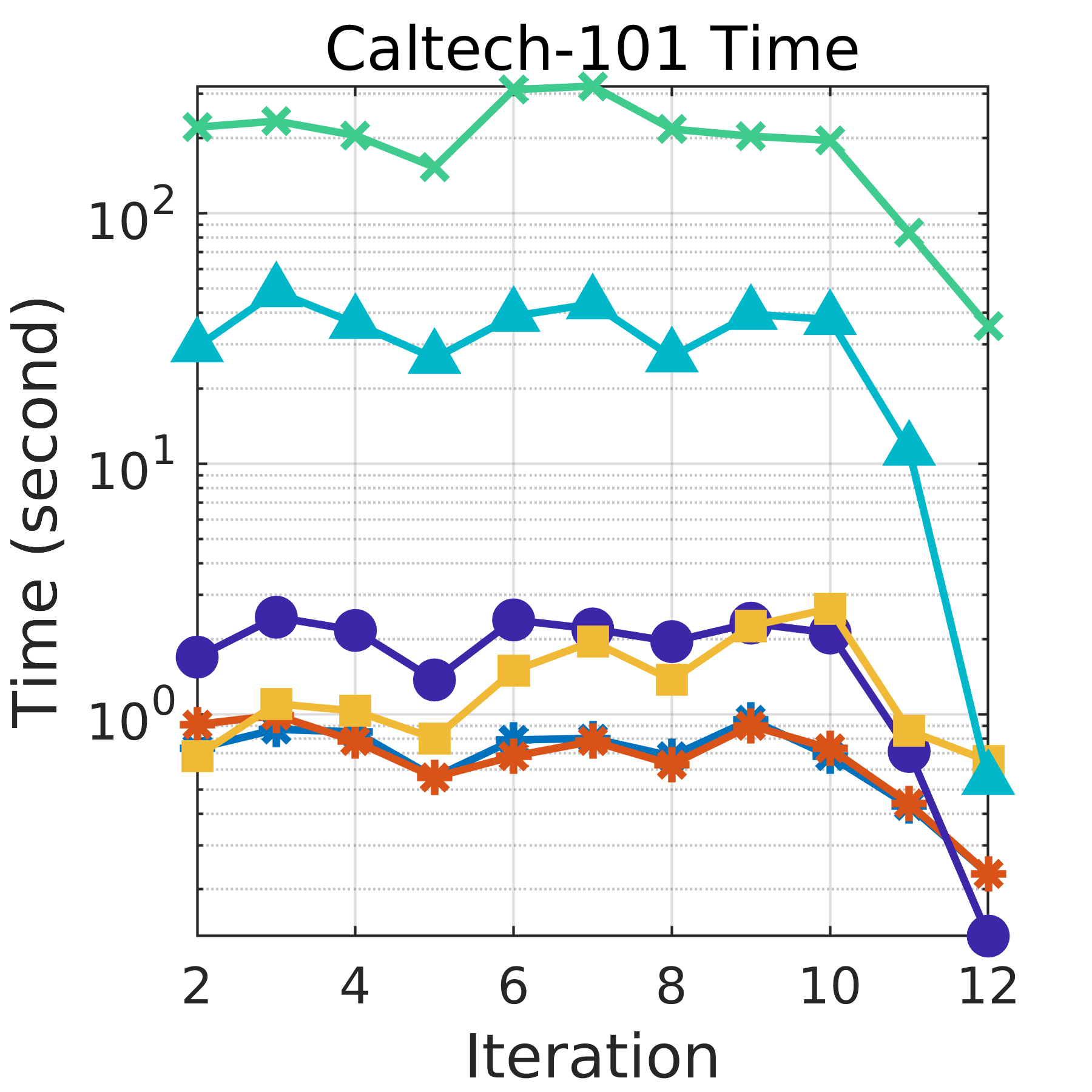}
	\end{subfigure}
	\begin{subfigure}[b]{0.23\textwidth}   
		\centering 
		\includegraphics[width=4.2cm,height=3.6cm]{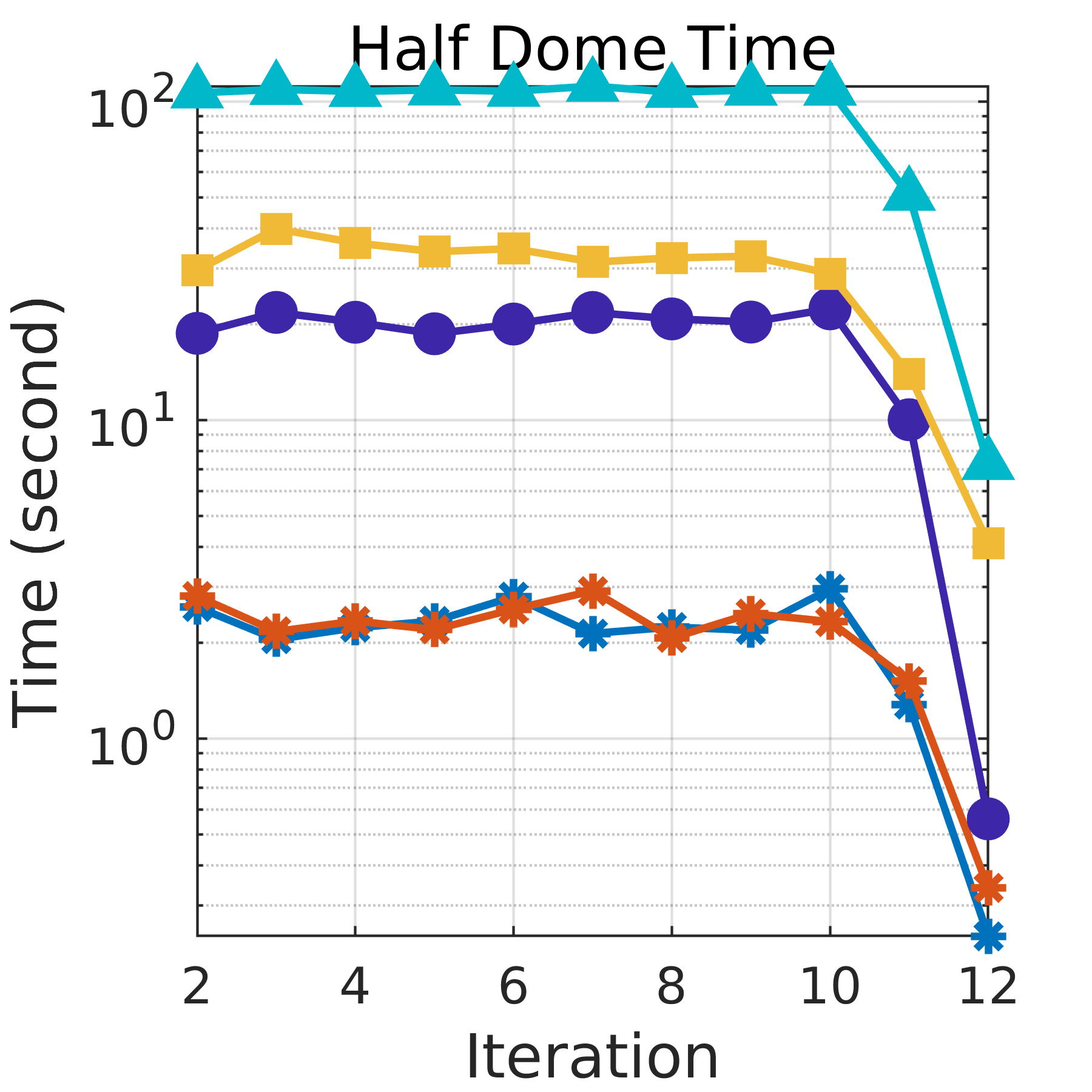}
	\end{subfigure}
	\begin{subfigure}[b]{0.23\textwidth}   
		\centering 
		\includegraphics[width=4.2cm,height=3.6cm]{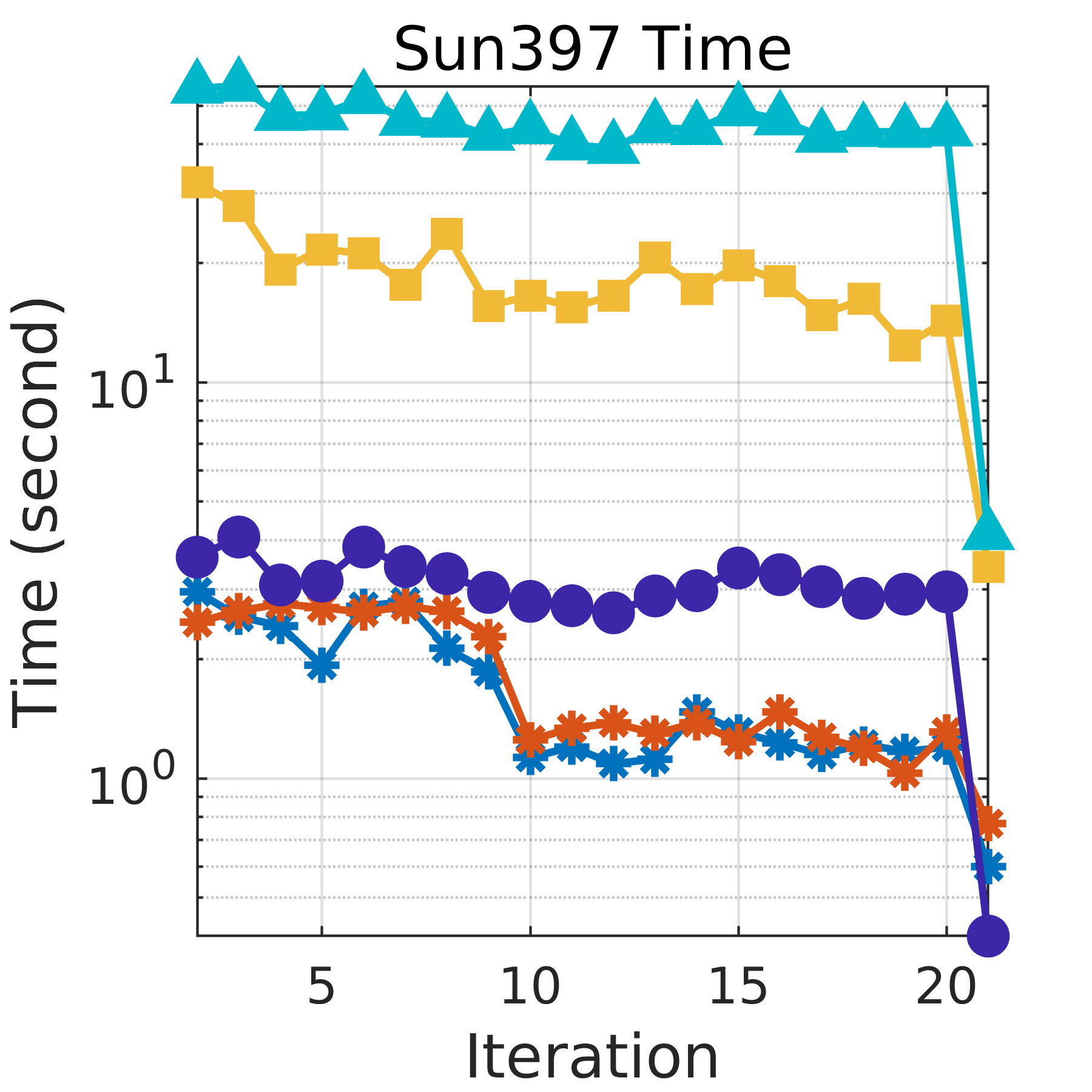}
	\end{subfigure}
	\caption[ Results for news and image retrieval in a dynamic database comparison against online hashing methods. Recall@20 performance (1st row) and Update time cost (2nd row). 1st column: News20. 2nd column: Caltech-101. 3rd column: Sun397. 4th column: Half dome. Time cost is in log scale.]
	{\small Results for news and image retrieval in a dynamic database comparison against online hashing methods. Recall@20 performance (1st row) and Update time cost (2nd row). 1st column: News20. 2nd column: Caltech-101. 3rd column: Half dome. 4th column: Sun397. Time cost is in log scale.} 
	\label{knn_img_online}
\end{figure*}

\begin{figure*}
	\centering
	\begin{subfigure}[b]{1\textwidth}
		\centering
		\includegraphics[width=0.6\textwidth]{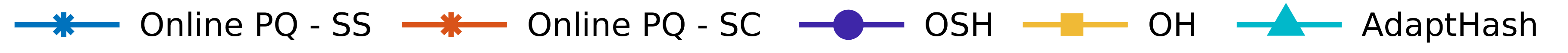}   
	\end{subfigure}
	\vskip\baselineskip

	\begin{subfigure}[b]{0.49\textwidth}
		\centering
		\includegraphics[width=9cm]{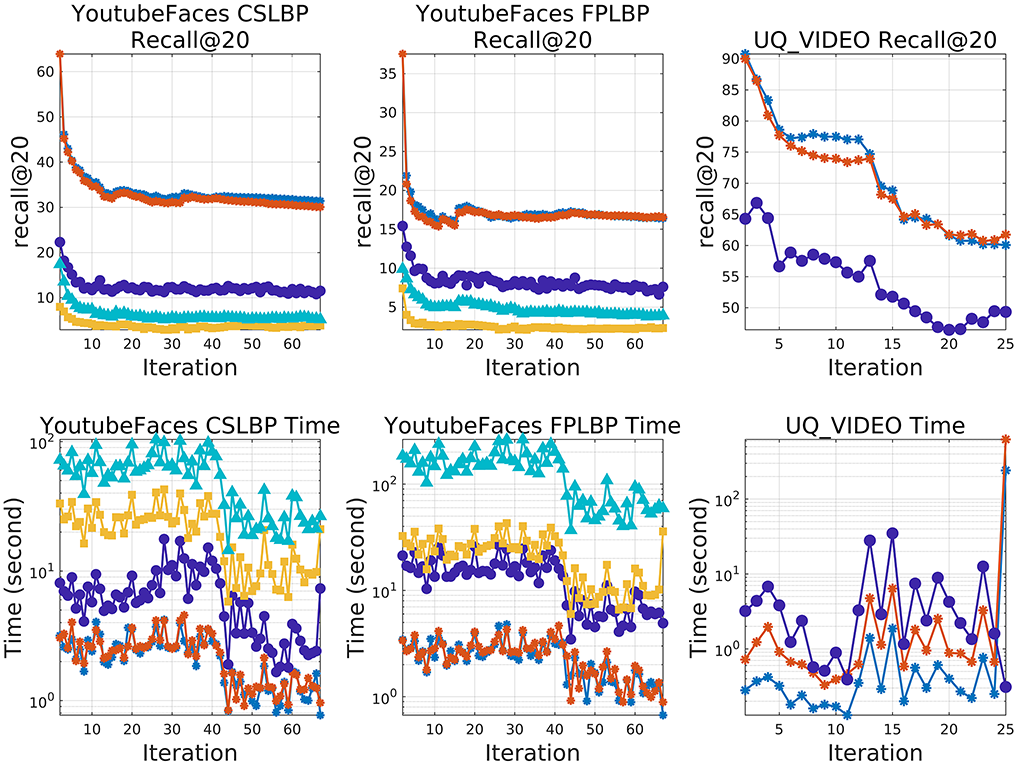}    
	\end{subfigure}
	\begin{subfigure}[b]{.48\textwidth}
		\includegraphics[width=8.5cm]{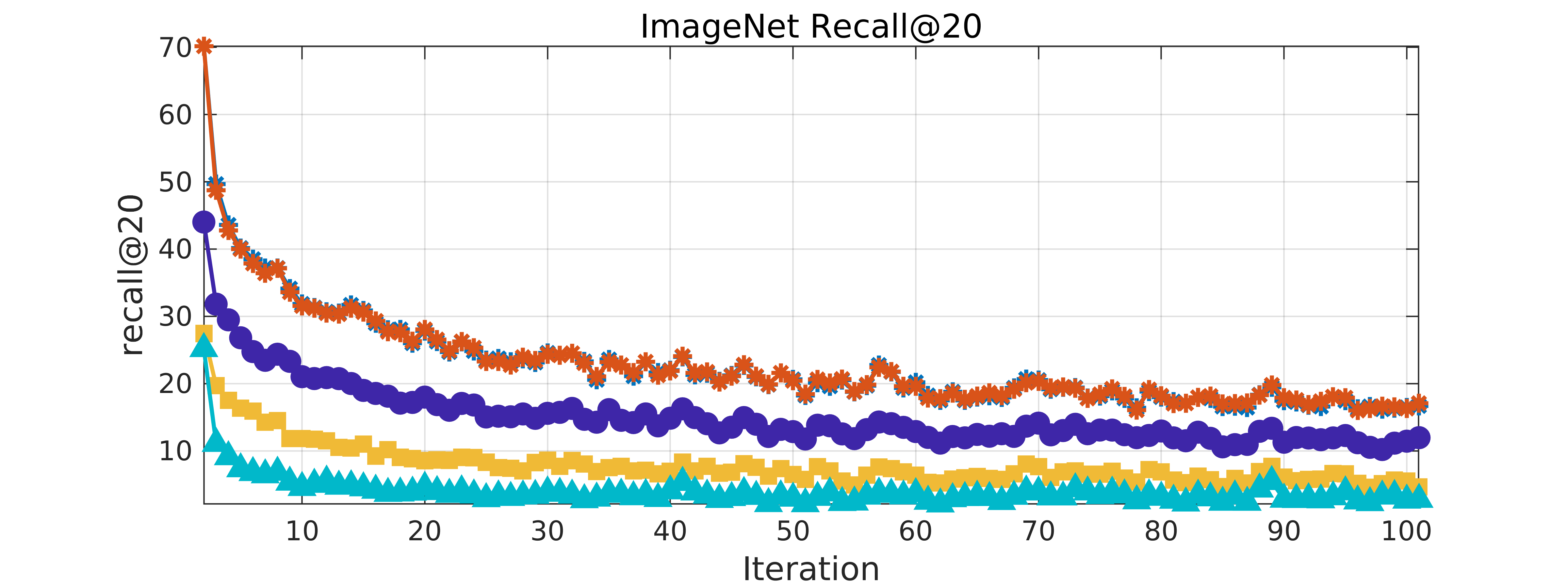}
		
		\vspace{2ex}
		
		\includegraphics[width=8.3cm]{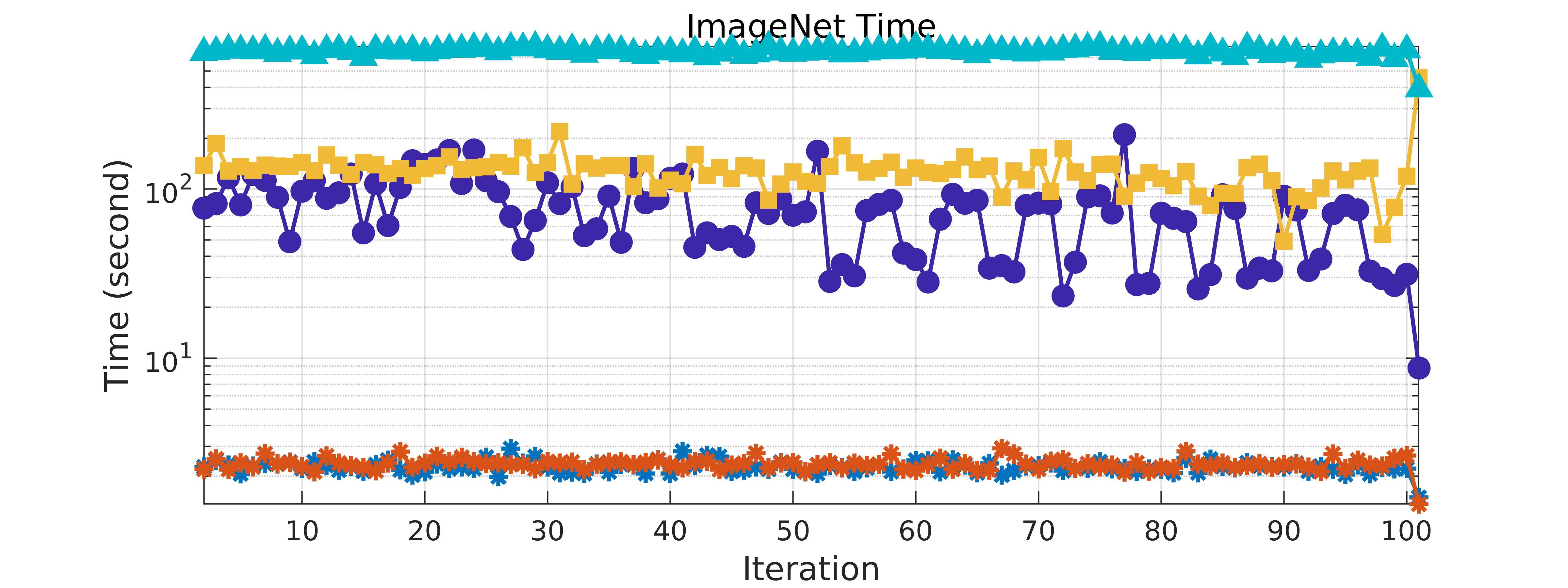}
	\end{subfigure}
	\caption[ ]
	{\small Results for YoutubeFaces on CSLBP and FPLBP features, UQ\_VIDEO and ImageNet in a dynamic database comparison against online hashing methods. Recall@20 (1st row) and Update time cost (2nd row). 1st column: YoutubeFaces CSLBP feature. 2nd column: YoutubeFaces FPLBP feature. 3rd column: UQ\_VIDEO. 4th column: ImageNet. Time cost is in log scale.}
	\label{knn_video_online}
\end{figure*}

\subsubsection{Online methods comparison}

Figure \ref{news20_diff_bits} demonstrate the performance of indexing update using different number of bits of the model on News20 dataset compared with four online hashing methods.
It clearly shows that our proposed models consistently outperforms other online hashing methods, with the lowest update time cost in different number of bits.
In particular, when the number of bits increases, the difference between online PQ and other online methods gets increasing.
SSBC achieves comparable search accuracy with online PQ in some of the iterations using 32 bits, but its update time cost is significantly higher than other methods.
The curve trends of all methods over different number of bits used are consistent.

As shown in Figure \ref{knn_img_online} and Figure \ref{knn_video_online}, it is evident that our proposed method with two different budget constraints can achieve superior performance in both efficiency and effectiveness compared to other online methods. 
Specifically, online PQ significantly outperforms the second best online method, OSH, in search accuracy for all datasets and is much faster in model update.
These two figures are quite revealing in several ways.
\begin{enumerate}[label={\alph*)},wide, labelwidth=!, labelindent=0pt]
\item Interestingly, there is a sharp recall drop at iteration 13 in the News20 dataset. Since the distribution between our dynamic query sets changes frequently, this drastic change is caused by the significant query distribution change. All of the methods can not respond quickly to this effect, which reflects the importance and the necessity of developing an online approach that accommodates streaming data with changing distributions. 
\item Similarly for the Half dome dataset, the sudden recall improvement for all methods at iteration 11 implies the similar data distribution between the query set and all the existing stored data sets at iteration 11.
\item The update cost for each method is the total update time for the mini-batch of new data at each iteration and it is relatively stable through most of the iterations because they have similar sizes of the mini-batch data to update. The last iteration only updates around half of the mini-batch size to the previous iterations, so its update cost drops with respect to the number of the data instances to be updated.
\item As OH and AdaptHash are supervised online hashing methods, and OSH performs the best over all baseline methods, we compare our model with two different budget constraints with OSH for UQ\_VIDEO dataset.
It is obvious that our method achieves better search accuracy with lower update time cost.
Moreover, although the performance difference between the two budget constraints of our model is minimum, updating sub-codewords in half of the subspaces performs slightly better than updating half of the sub-codewords of all in both search accuracy and update time.
\end{enumerate}

\begin{figure*}
	\centering
	\begin{subfigure}[b]{1\textwidth}
		\centering
		\includegraphics[width=1\textwidth]{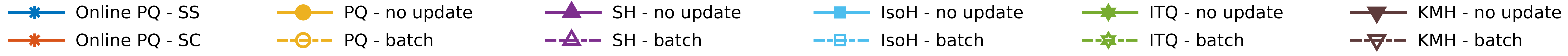}    
	\end{subfigure}
	\begin{subfigure}[b]{0.23\textwidth}
		\centering
		\includegraphics[width=4.2cm,height=3.6cm]{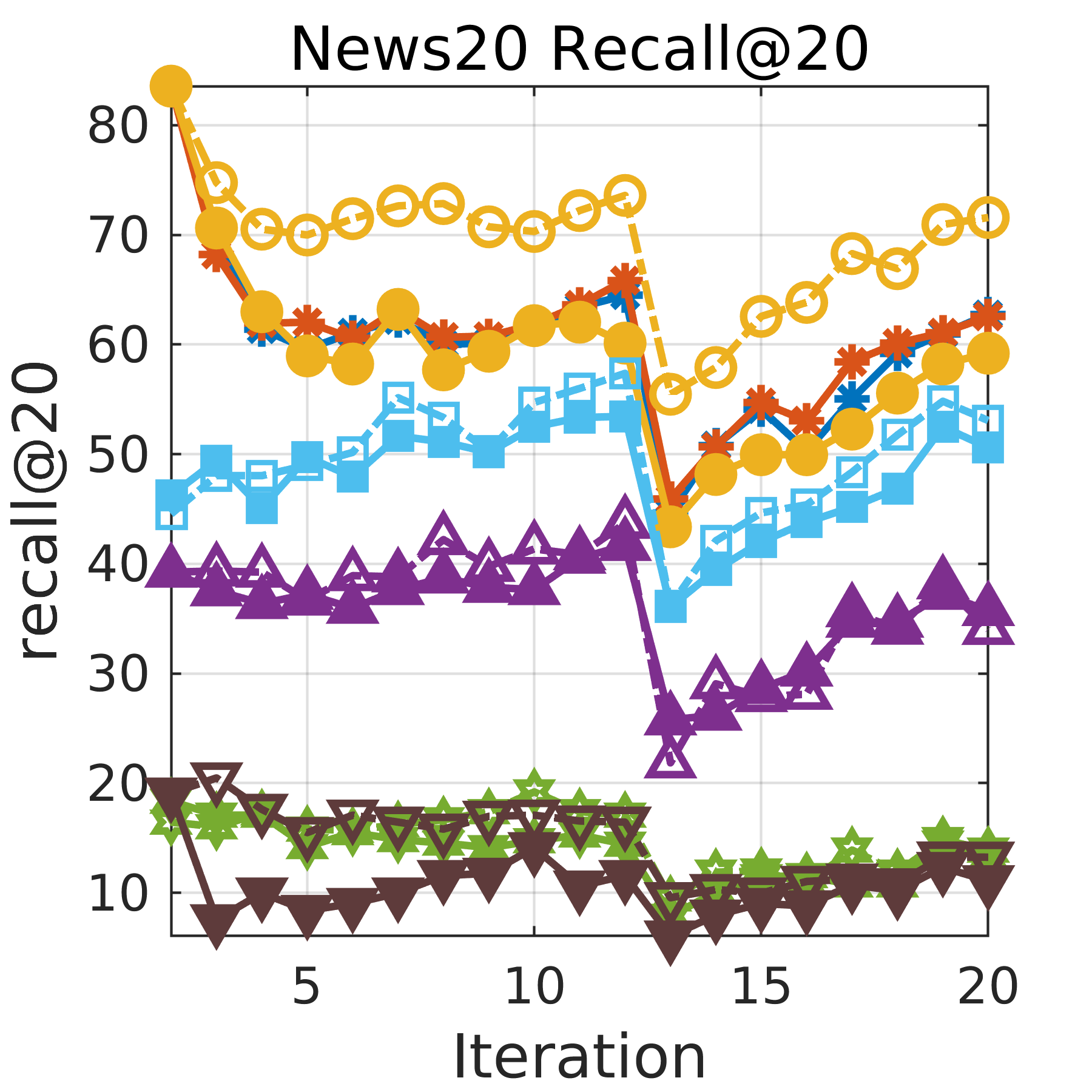}    
	\end{subfigure}
	\begin{subfigure}[b]{0.23\textwidth}  
		\centering 
		\includegraphics[width=4.2cm,height=3.6cm]{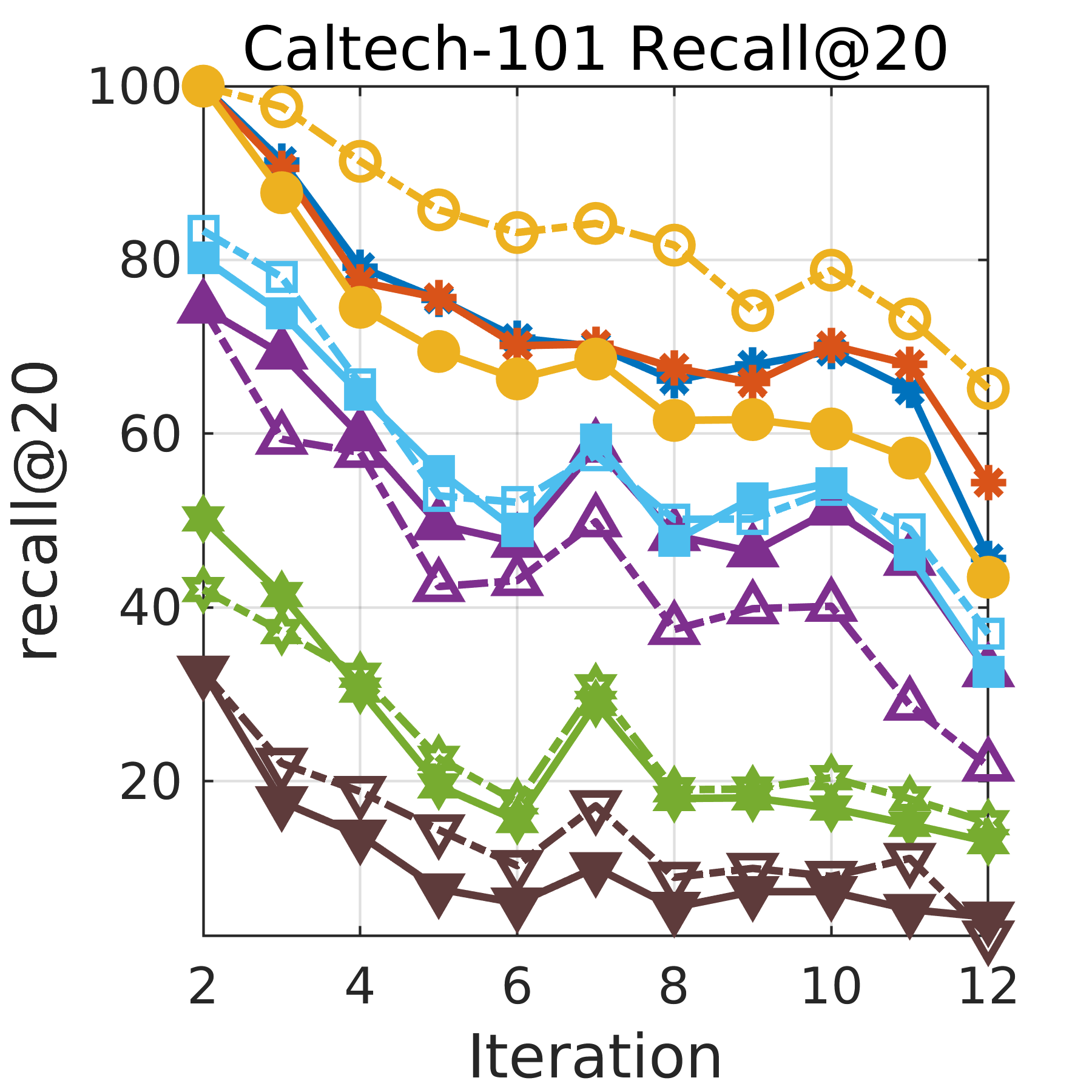}
	\end{subfigure}
	\begin{subfigure}[b]{0.23\textwidth}   
		\centering 
		\includegraphics[width=4.2cm,height=3.6cm]{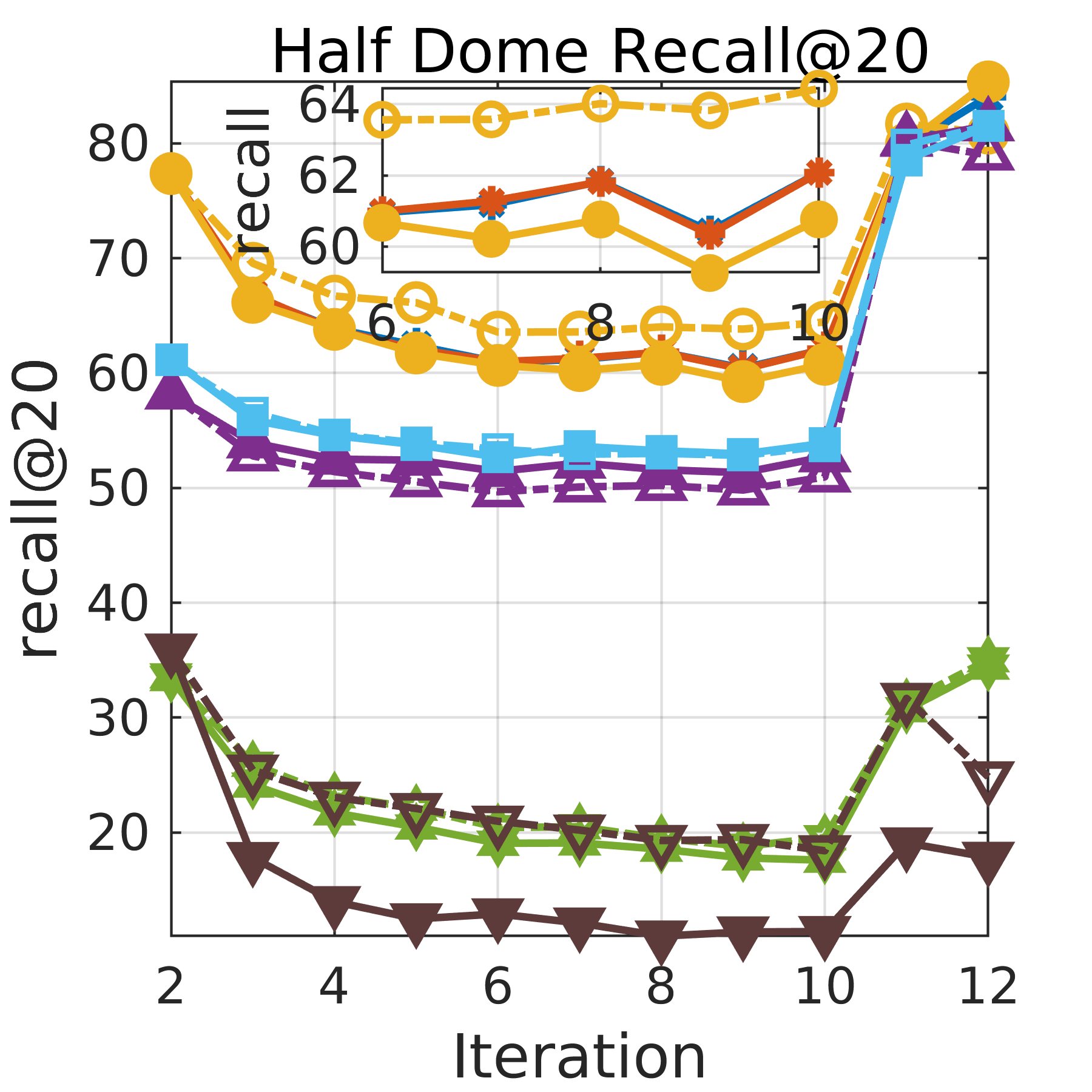}
	\end{subfigure}
	\begin{subfigure}[b]{0.23\textwidth}   
		\centering 
		\includegraphics[width=4.2cm,height=3.6cm]{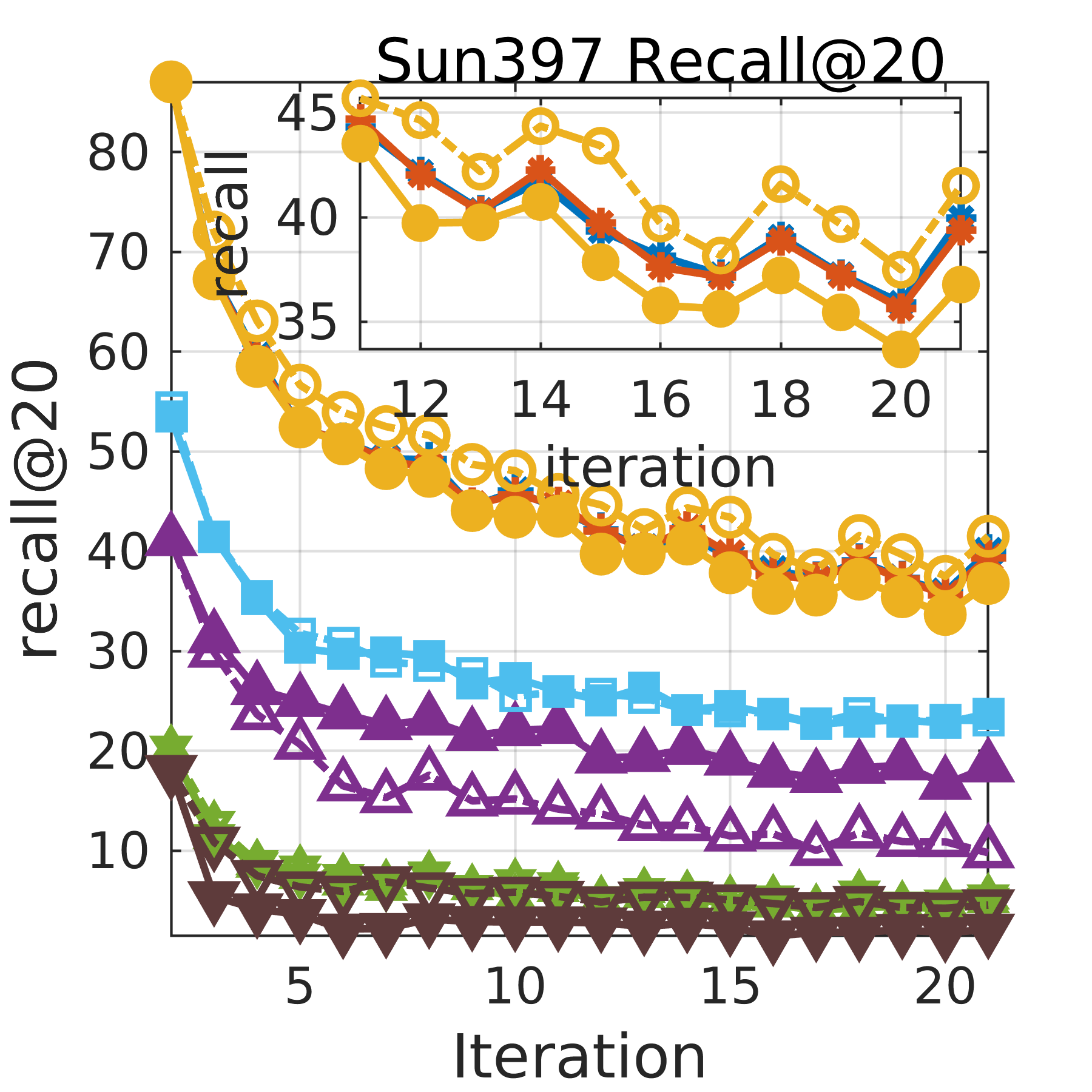}
	\end{subfigure}
	\begin{subfigure}[b]{0.23\textwidth}
		\centering
		\includegraphics[width=4.2cm,height=3.6cm]{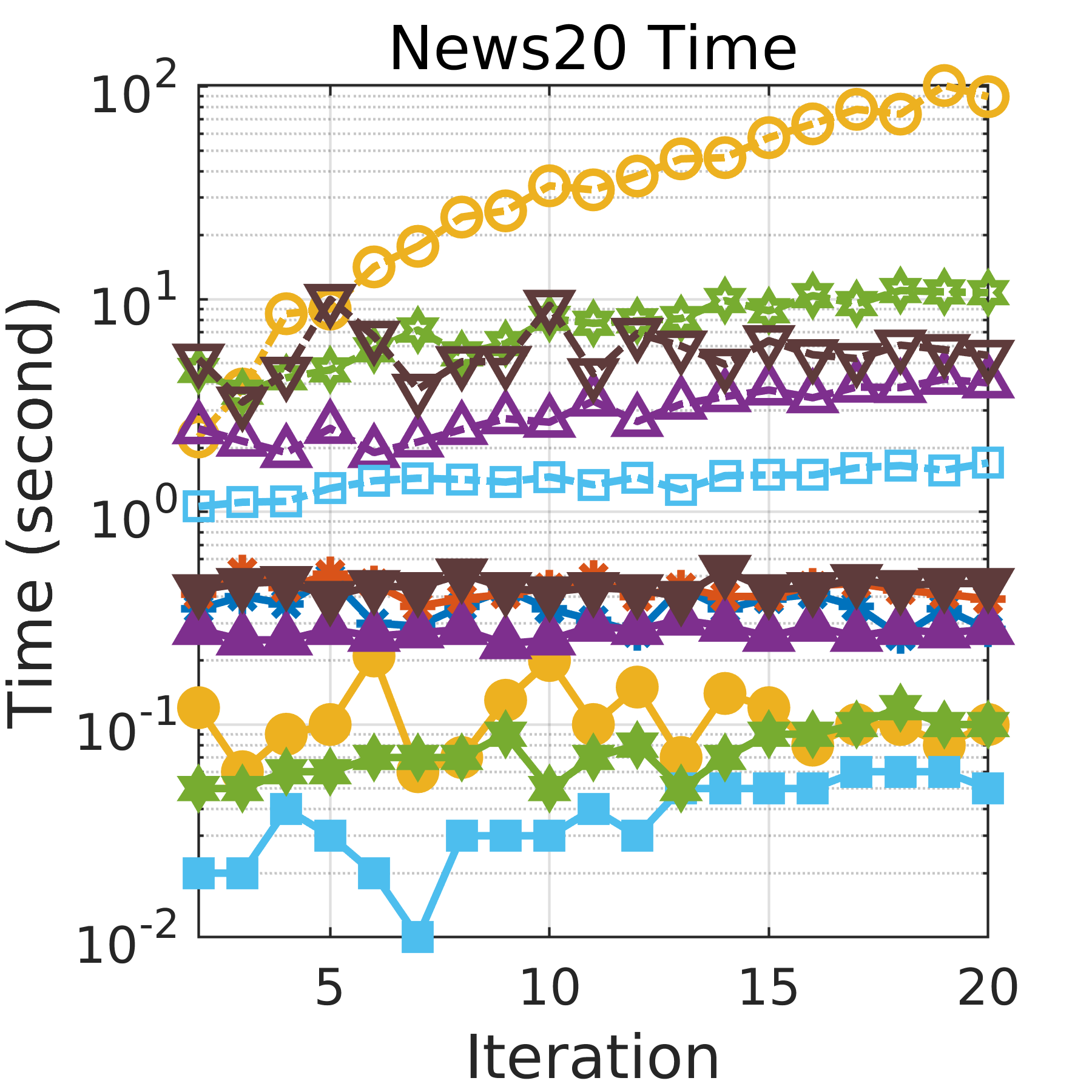}    
	\end{subfigure}
	\begin{subfigure}[b]{0.23\textwidth}  
		\centering 
		\includegraphics[width=4.2cm,height=3.6cm]{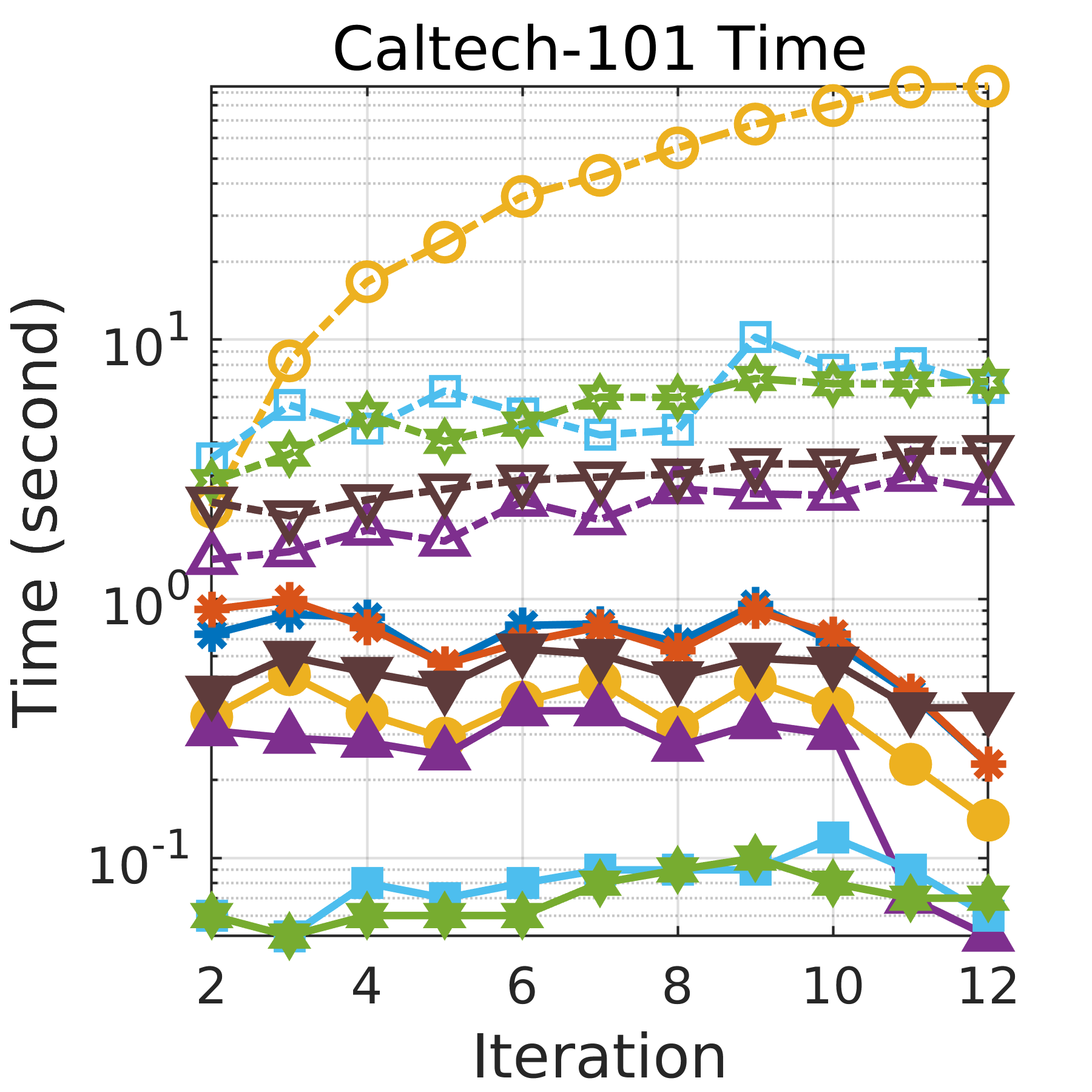}
	\end{subfigure}
	\begin{subfigure}[b]{0.23\textwidth}   
		\centering 
		\includegraphics[width=4.2cm,height=3.6cm]{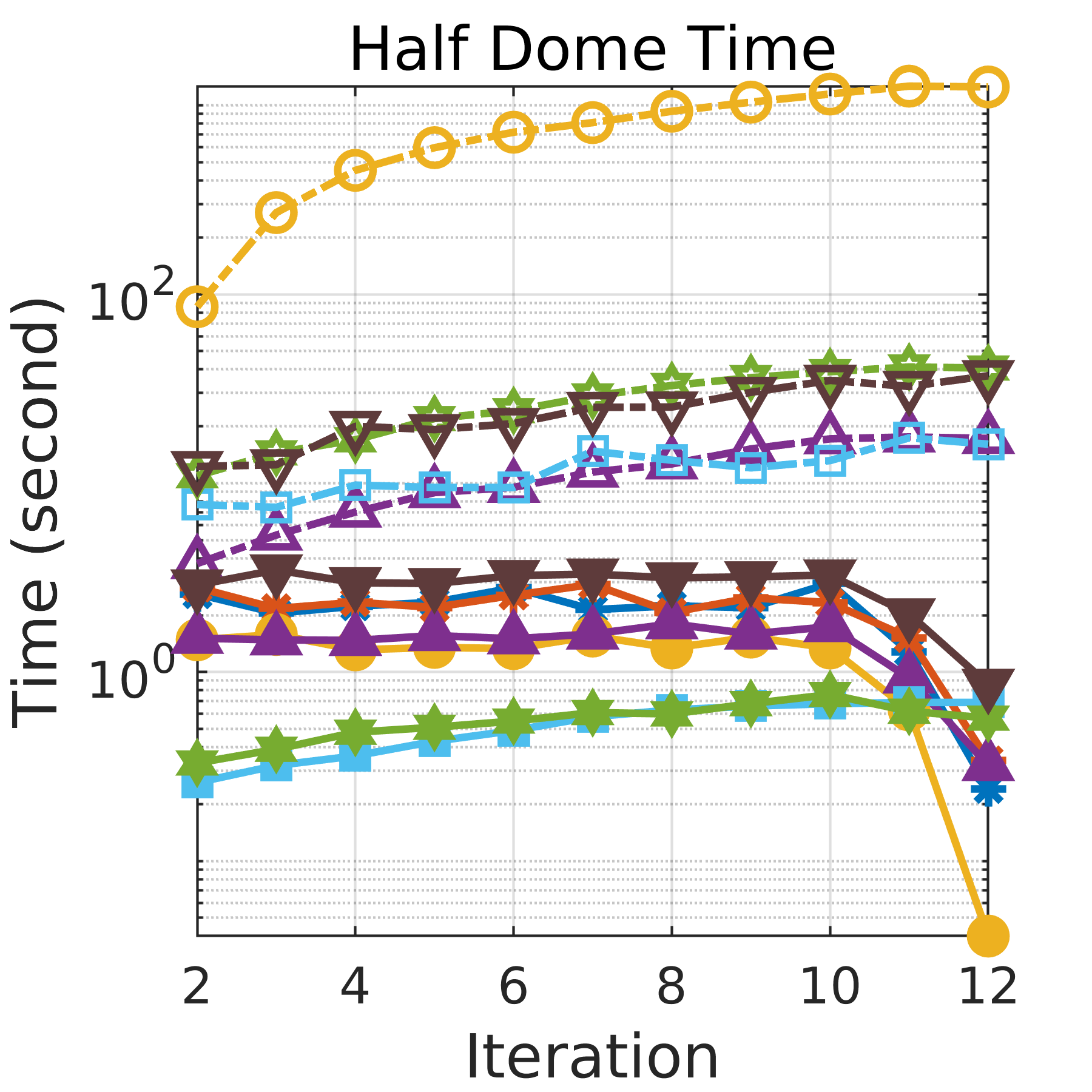}
	\end{subfigure}
	\begin{subfigure}[b]{0.23\textwidth}   
		\centering 
		\includegraphics[width=4.2cm,height=3.6cm]{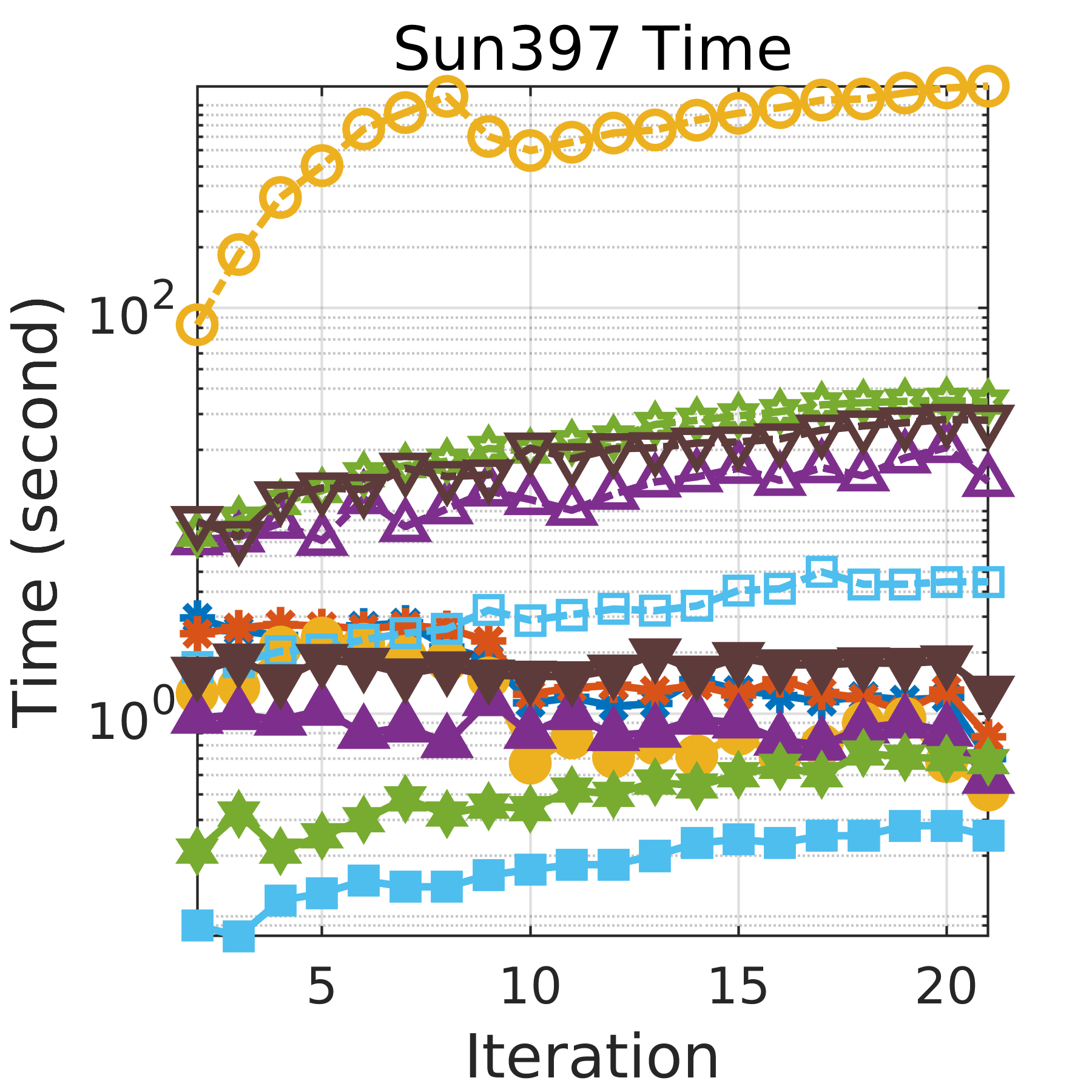}
	\end{subfigure}
	\caption[ ]
	{\small Results for news and image retrieval in a dynamic database comparison against batch methods. Recall@20 performance (1st row) and Update time cost (2nd row). 1st column: News20. 2nd column: Caltech-101. 3rd column: Sun397. 4th column: Half dome. Part of the recall plots of online and baseline PQ methods for Sun397 and Half dome are enlarged. Time cost is in log scale.} 
	\label{knn_img_batch}
\end{figure*}

\subsubsection{Batch methods comparison}
To further evaluate the performance of nearest neighbor search of our online model on how well it approaches to the search accuracy of batch mode methods and to the model update time of ``no update'' methods, we compare our model with each of the batch mode methods in two ways: retrain the model at each iteration (batch) and using the model trained on the initial iteration once for all (no update).
The comparison results displayed in Figure \ref{knn_img_batch} implies several interesting observations.
First, as the update time cost graphs are plot in log scale, the update time of online PQ is only slightly more than the one of the ``no update'' methods, but significantly lower than the one of the ``batch'' methods.
Second, online PQ and PQ methods significantly outperform other batch hashing methods, and online PQ performs slightly worse than ``batch'' PQ and better than ``no update'' PQ.
Therefore, we can conclude that our online model achieves good tradeoff between accuracy and update efficiency.
Though KMH performs slightly better than ITQ in its original paper, it performs worse in our experiment setting. This is because that we are in the online setting where the data distribution of the query set may be a lot different from the one of the existing database.
More results on our online model compared with batch methods for video datasets are in the Supplementary Material.

\subsection{Continuous querying on dynamic real-time data}
In many emerging application environments, they commonly require to emphasize the most recent data and ignore the expired data in retrieval search.
Examples of such applications include network monitoring and online portfolio selection.
Furthermore, applications such as hot topic news article retrieval system, object tracking given a recent period of time from social network albums or live videos require the real-time behaviour of the data.
Therefore, to reflect this requirement in an online indexing system, we employ sliding window technique.
In this experiment, we investigate the comparison between with and without employing the sliding window technique, and presents the comparison results on different hash methods.

\begin{figure}[ht]
	\centering
	\includegraphics[scale=0.45]{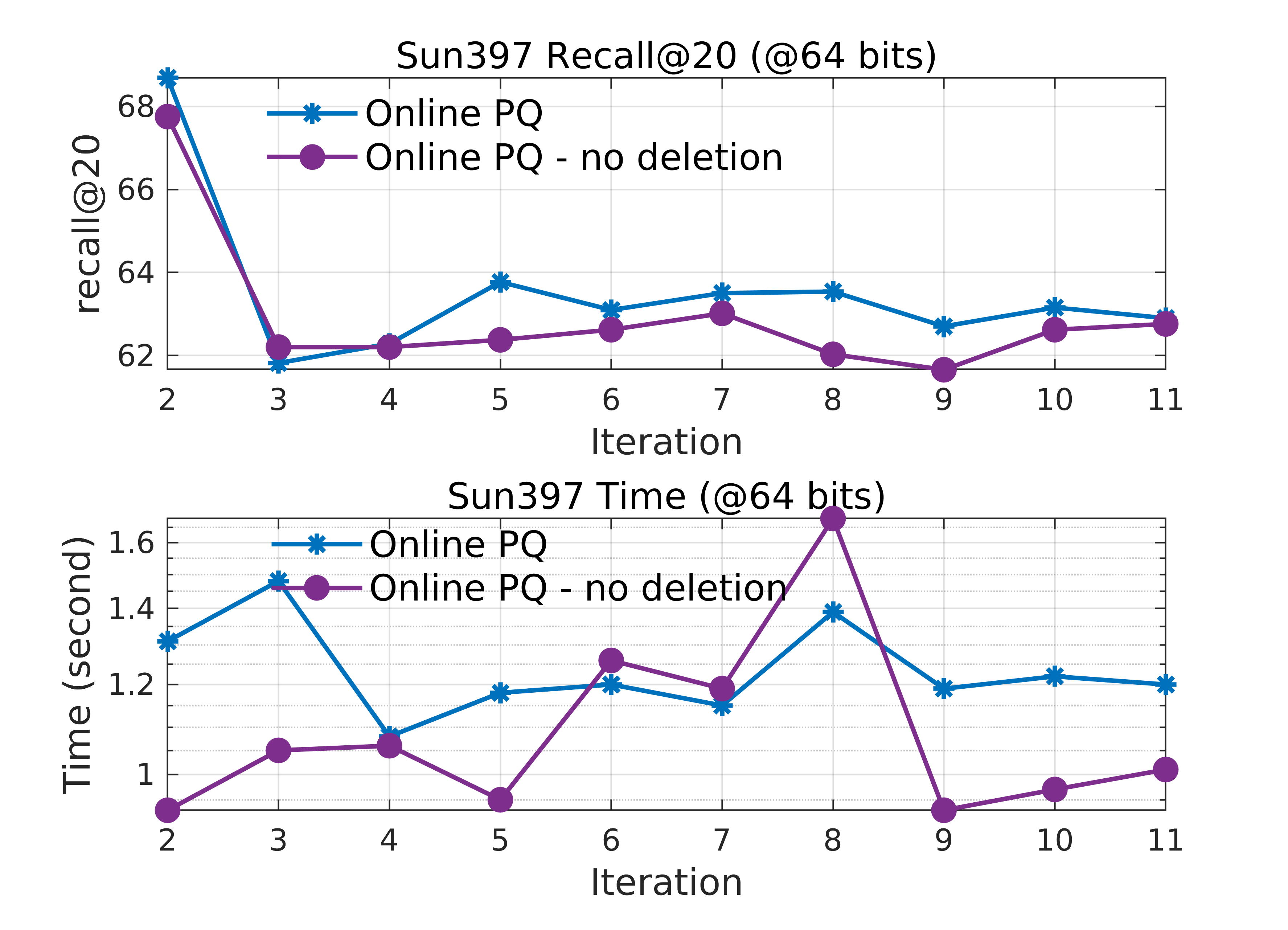}
	\caption{\label{slidingwindow_sun397_nodeletion} 
		\small Online PQ over a sliding window approach between deletion and no deletion of the expired data to the model for Sun397. Recall@20 (1st row) and Update time cost (2nd row).}
\end{figure}

\begin{figure*}
	\centering
	\begin{subfigure}[b]{1\textwidth}
		\centering
		\includegraphics[width=0.6\textwidth]{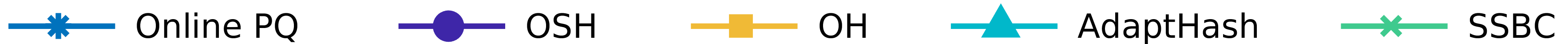}    
	\end{subfigure}
	\begin{subfigure}[b]{0.23\textwidth}
		\centering
		\includegraphics[width=4.2cm,height=3.6cm]{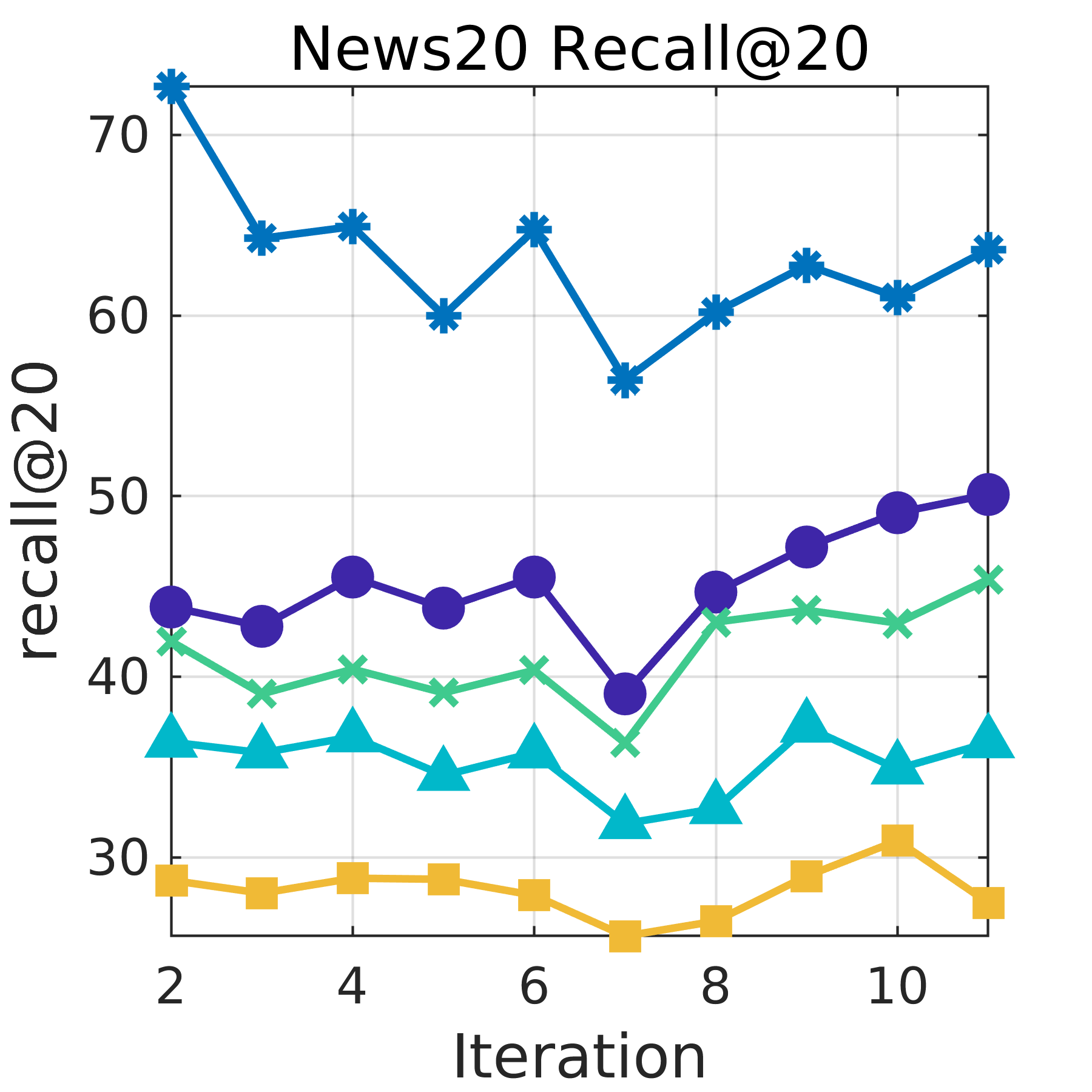}    
	\end{subfigure}
	\begin{subfigure}[b]{0.23\textwidth}  
		\centering 
		\includegraphics[width=4.2cm,height=3.6cm]{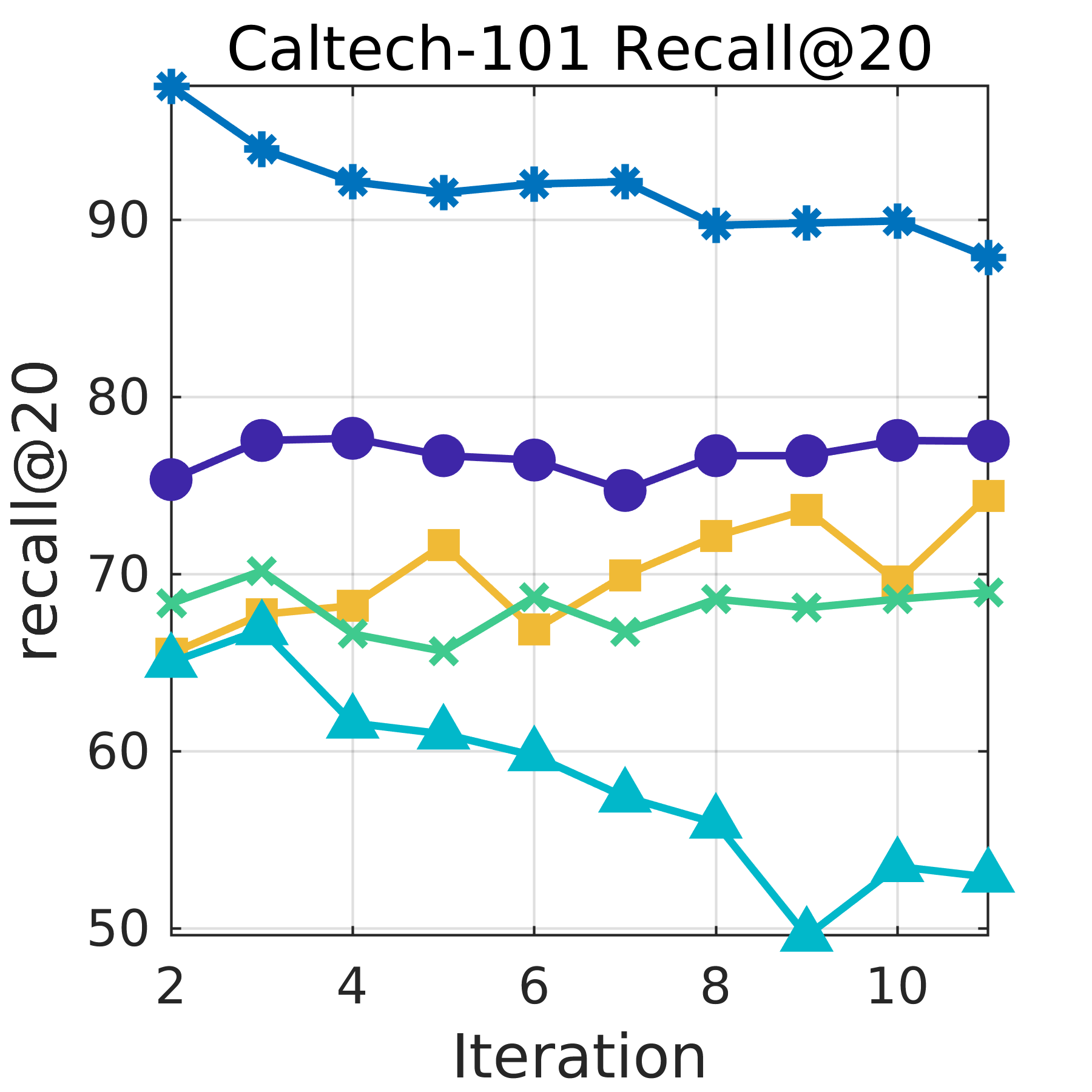}
	\end{subfigure}
	\begin{subfigure}[b]{0.23\textwidth}   
		\centering 
		\includegraphics[width=4.2cm,height=3.6cm]{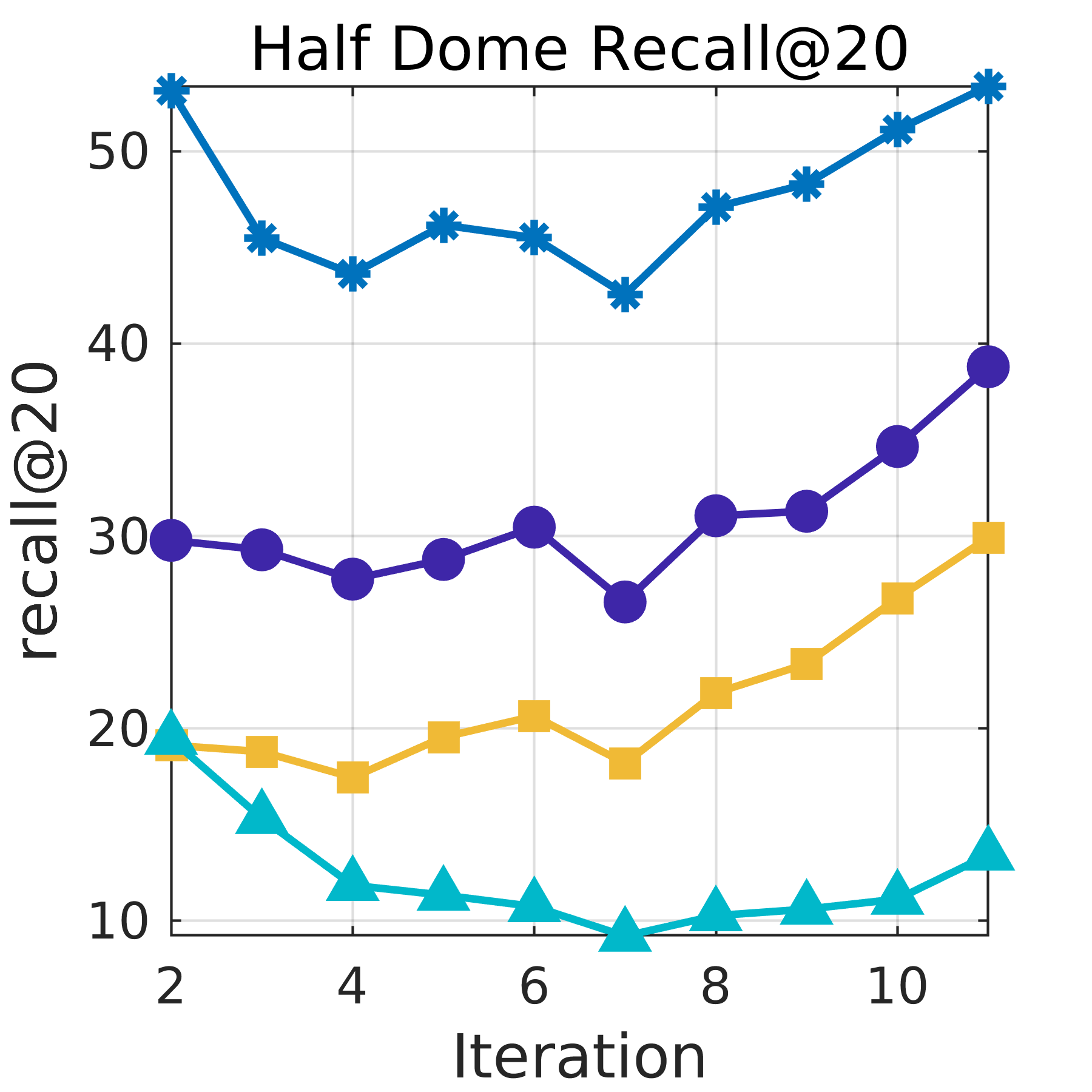}
	\end{subfigure}
	\begin{subfigure}[b]{0.23\textwidth}   
		\centering 
		\includegraphics[width=4.2cm,height=3.6cm]{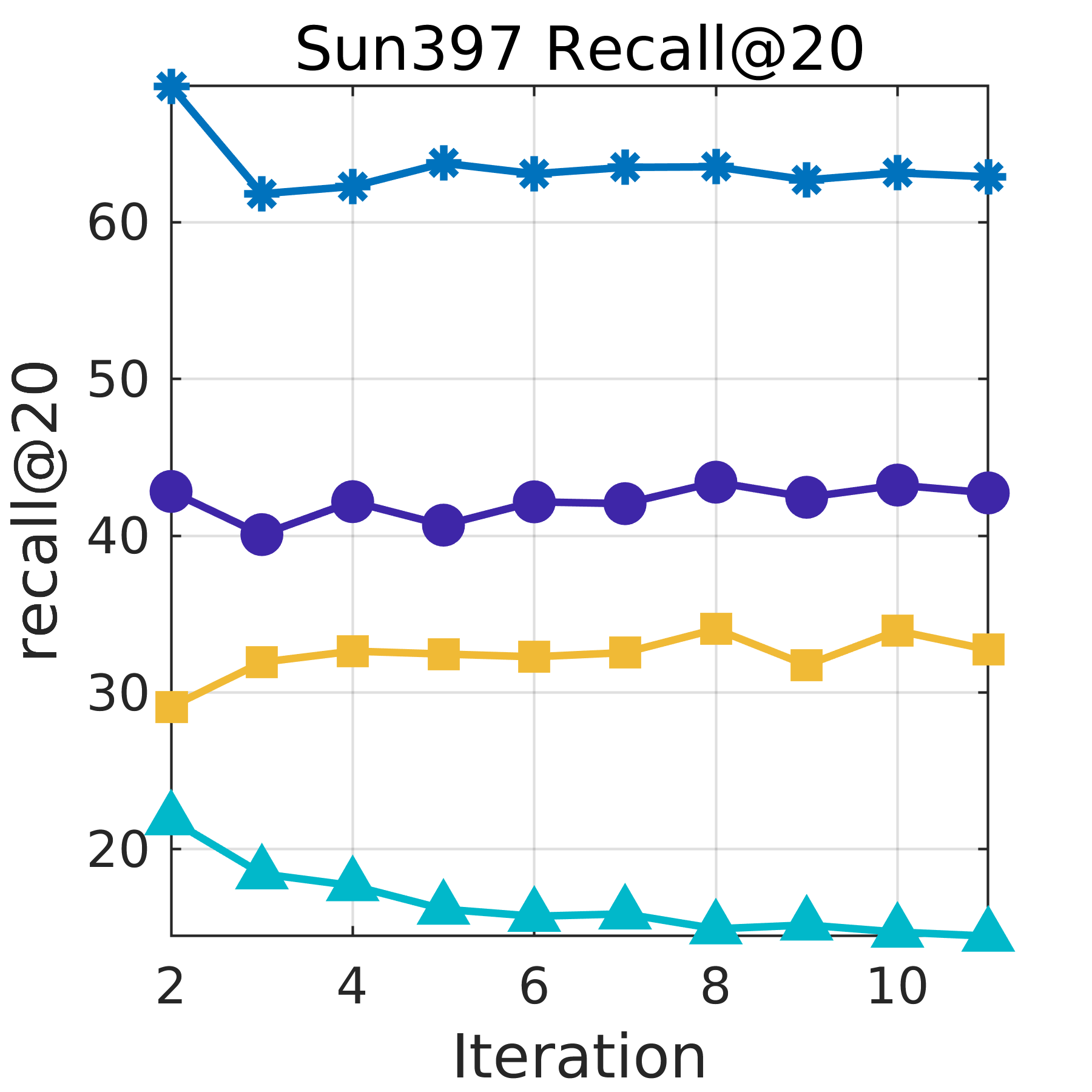}
	\end{subfigure}
	\begin{subfigure}[b]{0.23\textwidth}
		\centering
		\includegraphics[width=4.2cm,height=3.6cm]{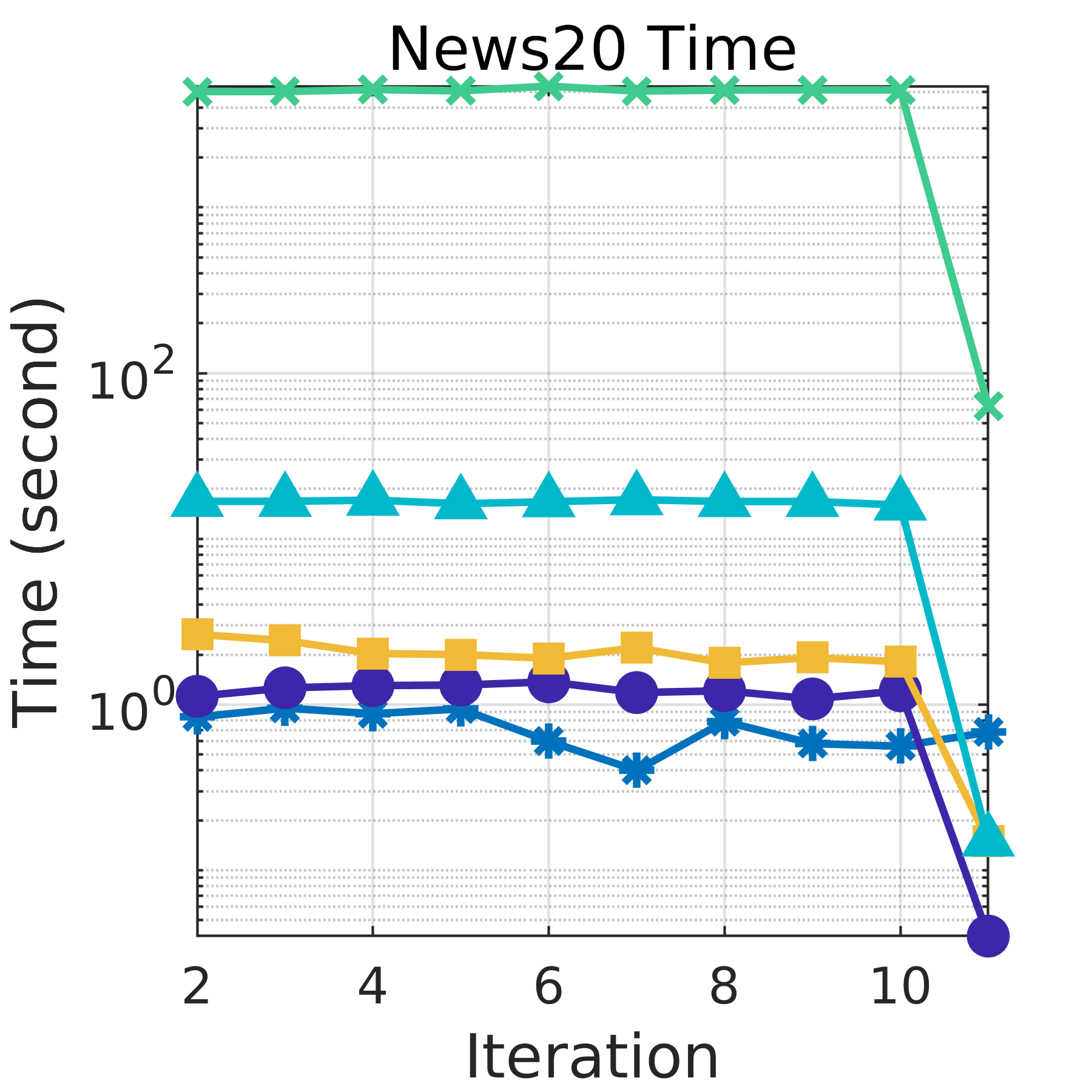}    
	\end{subfigure}
	\begin{subfigure}[b]{0.23\textwidth}  
		\centering 
		\includegraphics[width=4.2cm,height=3.6cm]{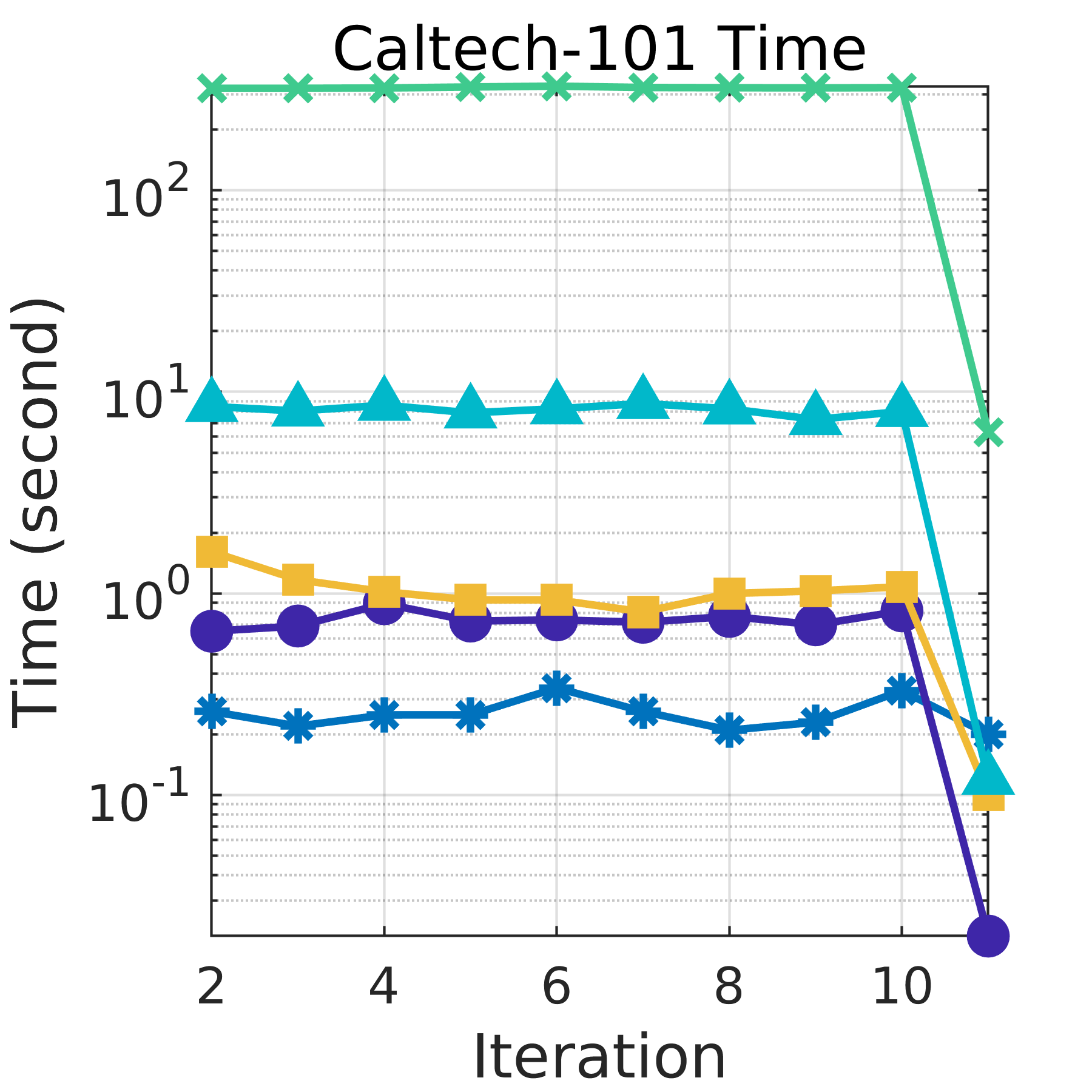}
	\end{subfigure}
	\begin{subfigure}[b]{0.23\textwidth}   
		\centering 
		\includegraphics[width=4.2cm,height=3.6cm]{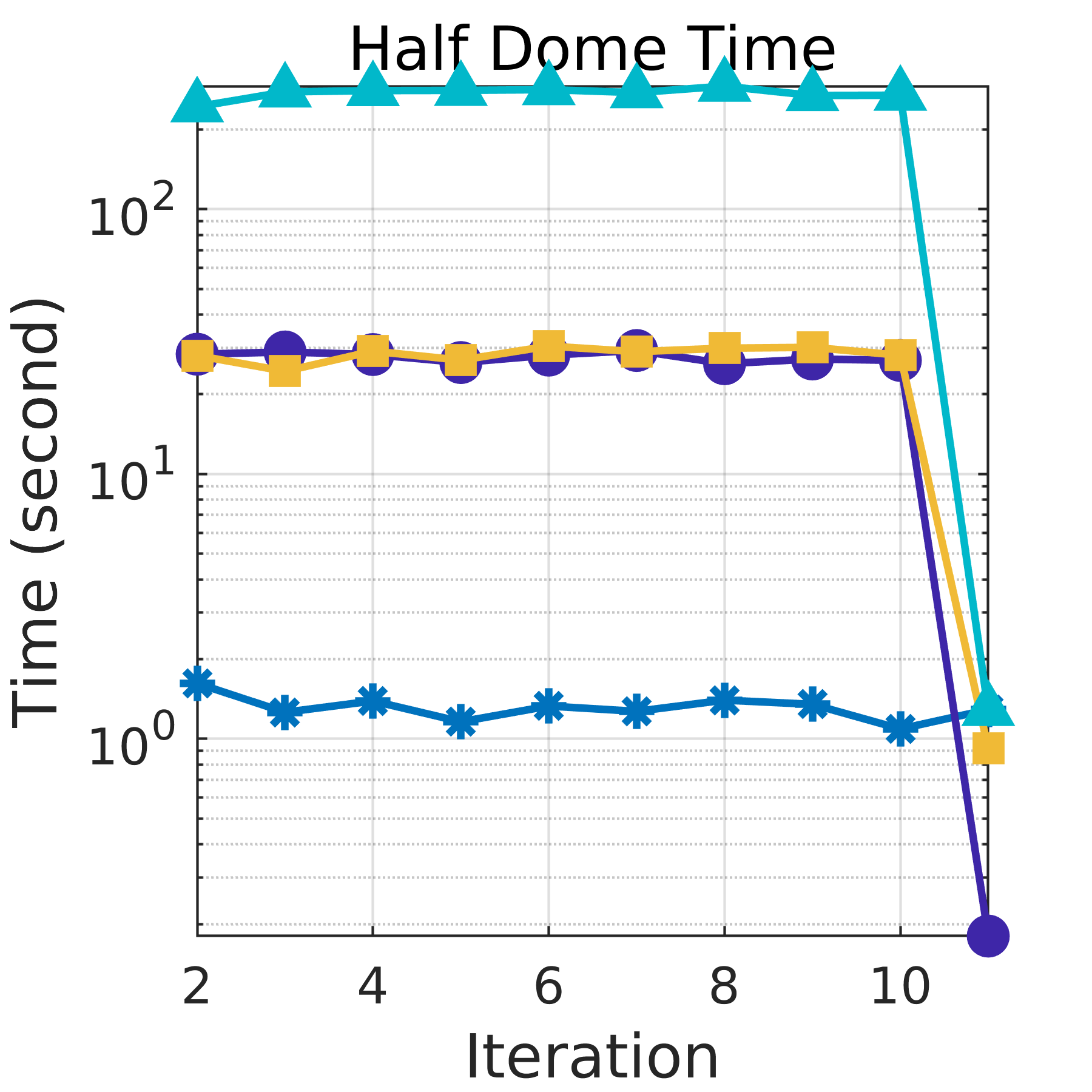}
	\end{subfigure}
	\begin{subfigure}[b]{0.23\textwidth}   
		\centering 
		\includegraphics[width=4.2cm,height=3.6cm]{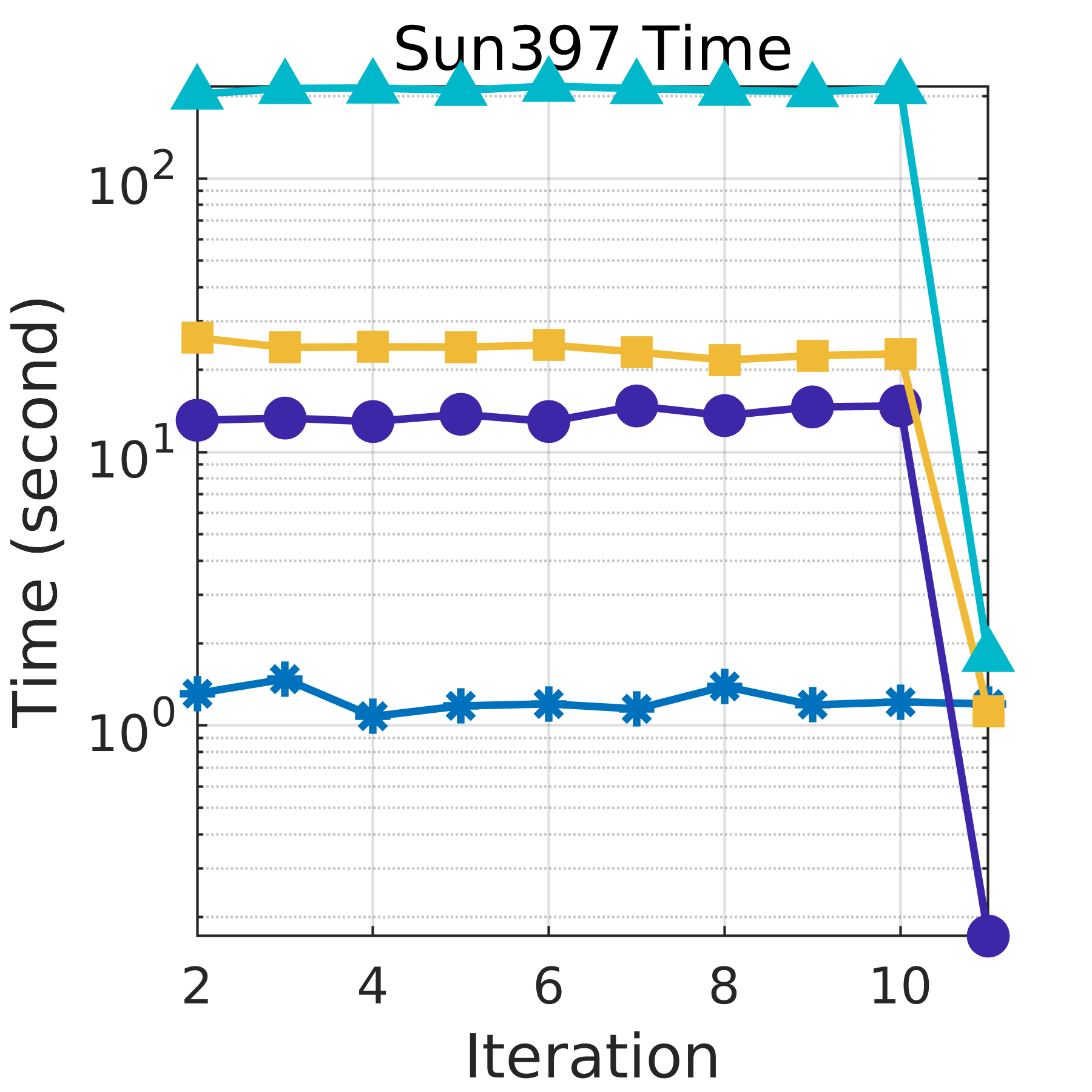}
	\end{subfigure}
	\caption[]
	{\small Results for news and image retrieval over a sliding window comparison against online hashing methods. Recall@20 performance (1st row) and Update time cost (2nd row). 1st column: News20. 2nd column: Caltech-101. 3rd column: Sun397. 4th column: Half dome. Time cost is in log scale.} 
	\label{slidingwindow_img_online}
\end{figure*}

\begin{figure}
	\centering
	\begin{subfigure}[b]{0.5\textwidth}
		\centering
		\includegraphics[width=0.8\textwidth]{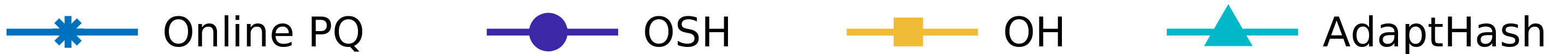}    
	\end{subfigure}
	\vskip\baselineskip
	\begin{subfigure}[b]{0.48\textwidth}
		\centering
		\centerline{\includegraphics[width=8.9cm]{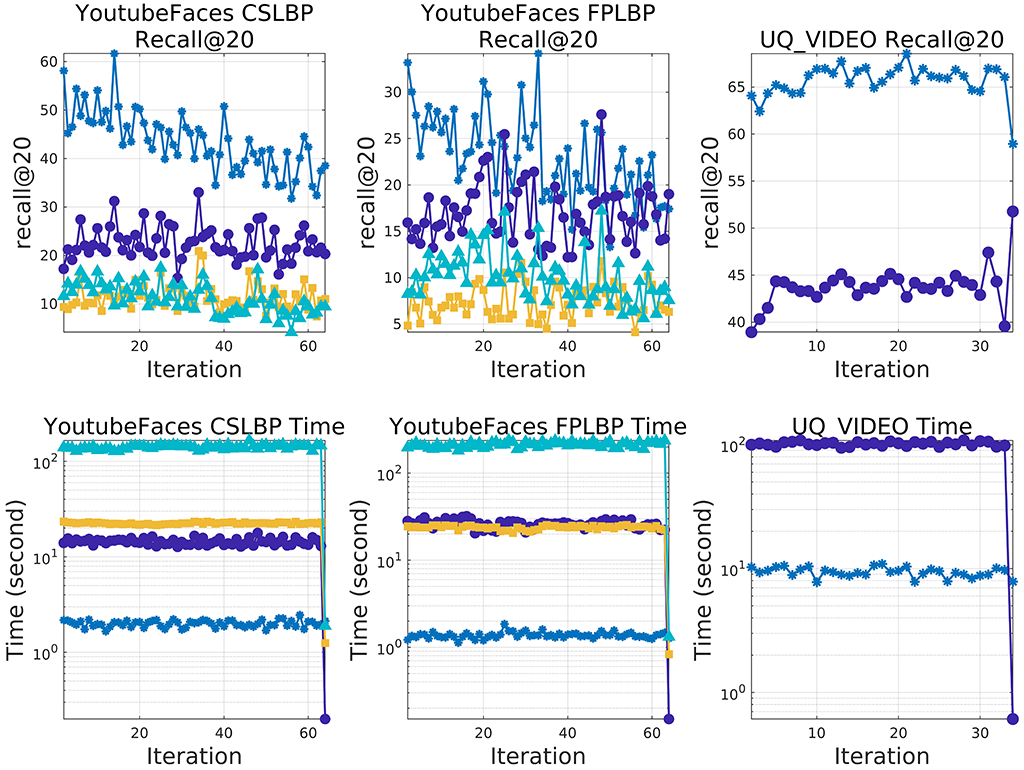}}    
	\end{subfigure}
	
	\caption[]
	{\small Results for YoutubeFaces dataset on CSLBP and FPLBP features and UQ\_VIDEO over a sliding window comparison against online hashing methods. Recall@20 (1st row) and Update time cost (2nd row). 1st column: YoutubeFaces CSLBP feature. 2nd column: YoutubeFaces FPLBP feature. 3rd column: UQ\_VIDEO. Time cost is in log scale.} 
	\label{slidingwindow_video_online}
\end{figure}

\subsubsection{Setting}
We follow the same way as in the setting in Section \ref{setting} to order the data in text and video datasets and the image data is ordered by classes without the overlapping of classes in each pair of two consecutive mini-batches this time.
We set the number of iterations to be 11 for text and image datasets, and 65 and 34 for YoutubeFaces and UQ\_VIDEO respectively.
The sliding window size is set to be 2000 for News20, 1000 for Caltech-101, 10000 for Sun397, Half dome and YoutubeFaces, and 100000 for UQ\_VIDEO.
Except for News20 dataset which contains news articles in chronological order, for the rest of the datasets, the sliding window contains images/videos belonging to a certain amount of classes/people at each iteration, so that the contribution of the classes/people for the expired data are removed from our proposed model.
We use the dynamic query set in this setting for all the datasets, so we use each new coming mini-batch of data as the query set first to retrieve similar data from the sliding window of the previous iteration and then use this mini-batch of data to update the model.
In our model, we remove the contribution of the data once it is removed from the sliding window.
We set $M$ = 8 and $K$ = 256, and update all the codewords over all subspaces.
Batch mode baseline methods retrain the model using the data in the sliding window at each iteration.

\subsubsection{Online methods comparison}

Our proposed method over a sliding window approach adds the contribution of the new incoming data to the index model and removes the contribution of the expired data to the index model at each iteration.
To evaluate our approach on data deletion using the sliding window technique, we highlight the difference between the models with and without expired data deletion in nearest neighbor search task in Figure \ref{slidingwindow_sun397_nodeletion}.
The update time of online PQ with expired data deletion is reasonably slightly higher than that of online PQ without expired data deletion and the search accuracy of online PQ with expired data deletion is slightly better as it emphasizes on the ``real-time'' data.

From Figure \ref{slidingwindow_img_online} and Figure \ref{slidingwindow_video_online}, we can see that online PQ over a sliding window approach achieves the best performance over other online methods in terms of update time efficiency and search effectiveness.
Specifically, the search accuracies for YoutubeFaces dataset of CSLBP feature and UQ\_VIDEO are significantly higher than OSH, with low update time cost.

\subsubsection{Batch methods comparison}

To investigate how well our proposed method over a sliding window technique with expired data deletion approaches the performance of the batch mode methods where the batch models will be retrained on the data from the sliding window at each iteration, we compare our model with batch mode methods.
In addition, we compare with ``no update'' model to show the update time complexity of our method.
Since the results are similar to the batch methods comparison for online PQ model in Section 6.6.3, we put them in the Supplementary Material.

\section{Conclusion and future work}
In this paper, we have presented our online PQ method to accommodate streaming data. In addition, we employ two budget constraints to facilitate partial codebook update to further alleviate the update time cost.
A relative loss bound has been derived to guarantee the performance of our model.
In addition, we propose an online PQ over sliding window approach, to emphasize on the real-time data.
Experimental results show that our method is significantly faster in accommodating the streaming data, outperforms the competing online and batch hashing methods in terms of search accuracy and update time cost, and attains comparable search quality with batch mode PQ.

In our future work, we will extend the online update for other MCQ methods, leveraging the advantage of them in a dynamic database environment to enhance the search performance.
Each of them has challenges to be effectively extended to handle streaming data.
For example, CQ \cite{DBLP:conf/icml/ZhangDW14} and SQ \cite{DBLP:conf/cvpr/ZhangQTW15} require the old data for the codewords update at each iteration due to the constant inter-dictionary-element-product in the model constraint.
AQ \cite{DBLP:conf/cvpr/BabenkoL14} requires a high computational encoding procedure, which will dominate the update process in an online fashion.
TQ \cite{DBLP:conf/cvpr/BabenkoL15} needs to consider the tree graph update together with the codebook and the indices of the stored data. 
Extensions to these methods can be developed to address the challenges for online update.
In addition, online PQ model can be extended to handle other learning problems such as multi-output learning \cite{liu2018metric,shen2017multilabel}.
Moreover, the theoretical bound for the online model will be further investigated.


%



\ifCLASSOPTIONcompsoc
  \section*{Acknowledgments}
\else
  \section*{Acknowledgment}
\fi

Ivor W. Tsang is supported by the ARC Future Fellowship FT130100746, ARC LP150100671 and DP180100106. Ying Zhang is supported by ARC FT170100128 and DP180103096.

\ifCLASSOPTIONcaptionsoff
  \newpage
\fi



\bibliographystyle{IEEEtran}

\begin{thebibliography}{10}
	\providecommand{\url}[1]{#1}
	\csname url@samestyle\endcsname
	\providecommand{\newblock}{\relax}
	\providecommand{\bibinfo}[2]{#2}
	\providecommand{\BIBentrySTDinterwordspacing}{\spaceskip=0pt\relax}
	\providecommand{\BIBentryALTinterwordstretchfactor}{4}
	\providecommand{\BIBentryALTinterwordspacing}{\spaceskip=\fontdimen2\font plus
		\BIBentryALTinterwordstretchfactor\fontdimen3\font minus
		\fontdimen4\font\relax}
	\providecommand{\BIBforeignlanguage}[2]{{%
			\expandafter\ifx\csname l@#1\endcsname\relax
			\typeout{** WARNING: IEEEtran.bst: No hyphenation pattern has been}%
			\typeout{** loaded for the language `#1'. Using the pattern for}%
			\typeout{** the default language instead.}%
			\else
			\language=\csname l@#1\endcsname
			\fi
			#2}}
	\providecommand{\BIBdecl}{\relax}
	\BIBdecl
	
	\bibitem{DBLP:journals/tkde/MoffatZS97}
	A.~Moffat, J.~Zobel, and N.~Sharman, ``Text compression for dynamic document
	databases,'' \emph{TKDE}, vol.~9, no.~2, pp. 302--313, 1997.
	
	\bibitem{DBLP:conf/www/PopoviciWG14}
	R.~Popovici, A.~Weiler, and M.~Grossniklaus, ``On-line clustering for real-time
	topic detection in social media streaming data,'' in \emph{SNOW 2014 Data
		Challenge}, 2014, pp. 57--63.
	
	\bibitem{DBLP:conf/icmcs/DongB03}
	A.~Dong and B.~Bhanu, ``Concept learning and transplantation for dynamic image
	databases,'' in \emph{ICME}, 2003, pp. 765--768.
	
	\bibitem{DBLP:journals/jmlr/CrammerDKSS06}
	K.~Crammer, O.~Dekel, J.~Keshet, S.~Shalev{-}Shwartz, and Y.~Singer, ``Online
	passive-aggressive algorithms,'' \emph{JMLR}, vol.~7, pp. 551--585, 2006.
	
	\bibitem{DBLP:conf/icml/ZhangYJXZ16}
	L.~Zhang, T.~Yang, R.~Jin, Y.~Xiao, and Z.~Zhou, ``Online stochastic linear
	optimization under one-bit feedback,'' in \emph{ICML}, 2016, pp. 392--401.
	
	\bibitem{DBLP:conf/ijcai/HuangYZ13}
	L.~Huang, Q.~Yang, and W.~Zheng, ``Online hashing,'' in \emph{IJCAI}, 2013, pp.
	1422--1428.
	
	\bibitem{huang2017online}
	------, ``Online hashing,'' \emph{TNNLS}, 2017.
	
	\bibitem{DBLP:journals/corr/GhashamiA15}
	M.~Ghashami and A.~Abdullah, ``Binary coding in stream,'' \emph{CoRR}, vol.
	abs/1503.06271, 2015.
	
	\bibitem{DBLP:conf/cvpr/LengWC0L15}
	C.~Leng, J.~Wu, J.~Cheng, X.~Bai, and H.~Lu, ``Online sketching hashing,'' in
	\emph{CVPR}, 2015, pp. 2503--2511.
	
	\bibitem{DBLP:conf/iccv/CakirS15}
	F.~Cakir and S.~Sclaroff, ``Adaptive hashing for fast similarity search,'' in
	\emph{ICCV}, 2015, pp. 1044--1052.
	
	\bibitem{DBLP:conf/ijcai/YangHZL13}
	Q.~Yang, L.~Huang, W.~Zheng, and Y.~Ling, ``Smart hashing update for fast
	response,'' in \emph{IJCAI}, 2013, pp. 1855--1861.
	
	\bibitem{cakir2016online}
	F.~Cakir, S.~A. Bargal, and S.~Sclaroff, ``Online supervised hashing,''
	\emph{CVIU}, 2016.
	
	\bibitem{DBLP:journals/pami/JegouDS11}
	H.~J{\'{e}}gou, M.~Douze, and C.~Schmid, ``Product quantization for nearest
	neighbor search,'' \emph{TPAMI}, vol.~33, no.~1, pp. 117--128, 2011.
	
	\bibitem{DBLP:journals/tip/MaTPL17}
	C.~Ma, I.~W. Tsang, F.~Peng, and C.~Liu, ``Partial hash update via hamming
	subspace learning,'' \emph{IEEE Transactions on Image Processing}, vol.~26,
	no.~4, pp. 1939--1951, 2017.
	
	\bibitem{DBLP:conf/icml/NorouziF11}
	M.~Norouzi and D.~J. Fleet, ``Minimal loss hashing for compact binary codes,''
	in \emph{ICML}, 2011, pp. 353--360.
	
	\bibitem{DBLP:conf/cvpr/LiuWJJC12}
	W.~Liu, J.~Wang, R.~Ji, Y.~Jiang, and S.~Chang, ``Supervised hashing with
	kernels,'' in \emph{CVPR}, 2012, pp. 2074--2081.
	
	\bibitem{DBLP:conf/cvpr/GongL11}
	Y.~Gong and S.~Lazebnik, ``Iterative quantization: {A} procrustean approach to
	learning binary codes,'' in \emph{CVPR}, 2011, pp. 817--824.
	
	\bibitem{DBLP:conf/nips/KongL12}
	W.~Kong and W.~Li, ``Isotropic hashing,'' in \emph{NIPS}, 2012, pp. 1655--1663.
	
	\bibitem{DBLP:conf/nips/WeissTF08}
	Y.~Weiss, A.~Torralba, and R.~Fergus, ``Spectral hashing,'' in \emph{NIPS},
	2008, pp. 1753--1760.
	
	\bibitem{DBLP:conf/vldb/GionisIM99}
	A.~Gionis, P.~Indyk, and R.~Motwani, ``Similarity search in high dimensions via
	hashing,'' in \emph{VLDB}, 1999, pp. 518--529.
	
	\bibitem{DBLP:conf/cvpr/BabenkoL14}
	A.~Babenko and V.~S. Lempitsky, ``Additive quantization for extreme vector
	compression,'' in \emph{CVPR}, 2014, pp. 931--938.
	
	\bibitem{DBLP:conf/icml/ZhangDW14}
	T.~Zhang, C.~Du, and J.~Wang, ``Composite quantization for approximate nearest
	neighbor search,'' in \emph{ICML}, 2014, pp. 838--846.
	
	\bibitem{DBLP:conf/cvpr/ZhangQTW15}
	T.~Zhang, G.~Qi, J.~Tang, and J.~Wang, ``Sparse composite quantization,'' in
	\emph{CVPR}, 2015, pp. 4548--4556.
	
	\bibitem{DBLP:conf/cvpr/BabenkoL15}
	A.~Babenko and V.~S. Lempitsky, ``Tree quantization for large-scale similarity
	search and classification,'' in \emph{CVPR}, 2015, pp. 4240--4248.
	
	\bibitem{DBLP:conf/cvpr/HeWS13}
	K.~He, F.~Wen, and J.~Sun, ``K-means hashing: An affinity-preserving
	quantization method for learning binary compact codes,'' in \emph{CVPR},
	2013, pp. 2938--2945.
	
	\bibitem{DBLP:conf/ecai/LiLWS16}
	Z.~Li, X.~Liu, J.~Wu, and H.~Su, ``Adaptive binary quantization for fast
	nearest neighbor search,'' in \emph{ECAI}, 2016, pp. 64--72.
	
	\bibitem{DBLP:journals/tip/LiuLDT17}
	X.~Liu, Z.~Li, C.~Deng, and D.~Tao, ``Distributed adaptive binary quantization
	for fast nearest neighbor search,'' \emph{IEEE Transactions on Image
		Processing}, vol.~26, no.~11, pp. 5324--5336, 2017.
	
	\bibitem{DBLP:journals/tcyb/LiuDDLL16}
	X.~Liu, B.~Du, C.~Deng, M.~Liu, and B.~Lang, ``Structure sensitive hashing with
	adaptive product quantization,'' \emph{IEEE Transactions on Cybernetics},
	vol.~46, no.~10, pp. 2252--2264, 2016.
	
	\bibitem{DBLP:journals/pami/GeHK014}
	T.~Ge, K.~He, Q.~Ke, and J.~Sun, ``Optimized product quantization,''
	\emph{TPAMI}, vol.~36, no.~4, pp. 744--755, 2014.
	
	\bibitem{DBLP:journals/tcyb/ChenXTL14}
	L.~Chen, D.~Xu, I.~W. Tsang, and X.~Li, ``Spectral embedded hashing for
	scalable image retrieval,'' \emph{IEEE Transactions on Cybernetics}, vol.~44,
	no.~7, pp. 1180--1190, 2014.
	
	\bibitem{matsui2018survey}
	Y.~Matsui, Y.~Uchida, H.~J{\'e}gou, and S.~Satoh, ``A survey of product
	quantization,'' \emph{ITE Transactions on Media Technology and Applications},
	vol.~6, no.~1, pp. 2--10, 2018.
	
	\bibitem{gray1998quantization}
	R.~M. Gray and D.~L. Neuhoff, ``Quantization,'' \emph{IEEE Transactions on
		Information Theory}, vol.~44, no.~6, pp. 2325--2383, 1998.
	
	\bibitem{Alpaydin:2010:IML:1734076}
	E.~Alpaydin, \emph{Introduction to Machine Learning}.\hskip 1em plus 0.5em
	minus 0.4em\relax The MIT Press, 2010.
	
	\bibitem{DBLP:conf/nips/DekelS05}
	O.~Dekel and Y.~Singer, ``Data-driven online to batch conversions,'' in
	\emph{NIPS}, 2005, pp. 267--274.
	
	\bibitem{DBLP:conf/icml/Lang95}
	K.~Lang, ``Newsweeder: Learning to filter {NET}news,'' in \emph{ICML}, 1995.
	
	\bibitem{DBLP:conf/cvpr/LiFP04}
	L.~Fei{-}Fei, R.~Fergus, and P.~Perona, ``Learning generative visual models
	from few training examples: An incremental bayesian approach tested on 101
	object categories,'' in \emph{CVPR Workshops}, 2004, p. 178.
	
	\bibitem{DBLP:conf/cvpr/WinderB07}
	S.~A.~J. Winder and M.~A. Brown, ``Learning local image descriptors,'' in
	\emph{CVPR}, 2007.
	
	\bibitem{DBLP:conf/cvpr/XiaoHEOT10}
	J.~Xiao, J.~Hays, K.~A. Ehinger, A.~Oliva, and A.~Torralba, ``{SUN} database:
	Large-scale scene recognition from abbey to zoo,'' in \emph{CVPR}, 2010, pp.
	3485--3492.
	
	\bibitem{DBLP:conf/cvpr/DengDSLL009}
	J.~Deng, W.~Dong, R.~Socher, L.~Li, K.~Li, and L.~Fei{-}Fei, ``Imagenet: {A}
	large-scale hierarchical image database,'' in \emph{CVPR}, 2009, pp.
	248--255.
	
	\bibitem{liu2018metric}
	W.~Liu, D.~Xu, I.~Tsang, and W.~Zhang, ``Metric learning for multi-output
	tasks,'' \emph{TPAMI}, 2018.
	
	\bibitem{shen2017multilabel}
	X.~Shen, W.~Liu, I.~Tsang, Q.-S. Sun, and Y.-S. Ong, ``Multilabel prediction
	via cross-view search,'' vol.~PP, pp. 1--15, 11 2017.
	
\end{thebibliography}

%



%





\begin{IEEEbiography}[{\includegraphics[width=1in,height=1.25in,clip,keepaspectratio]{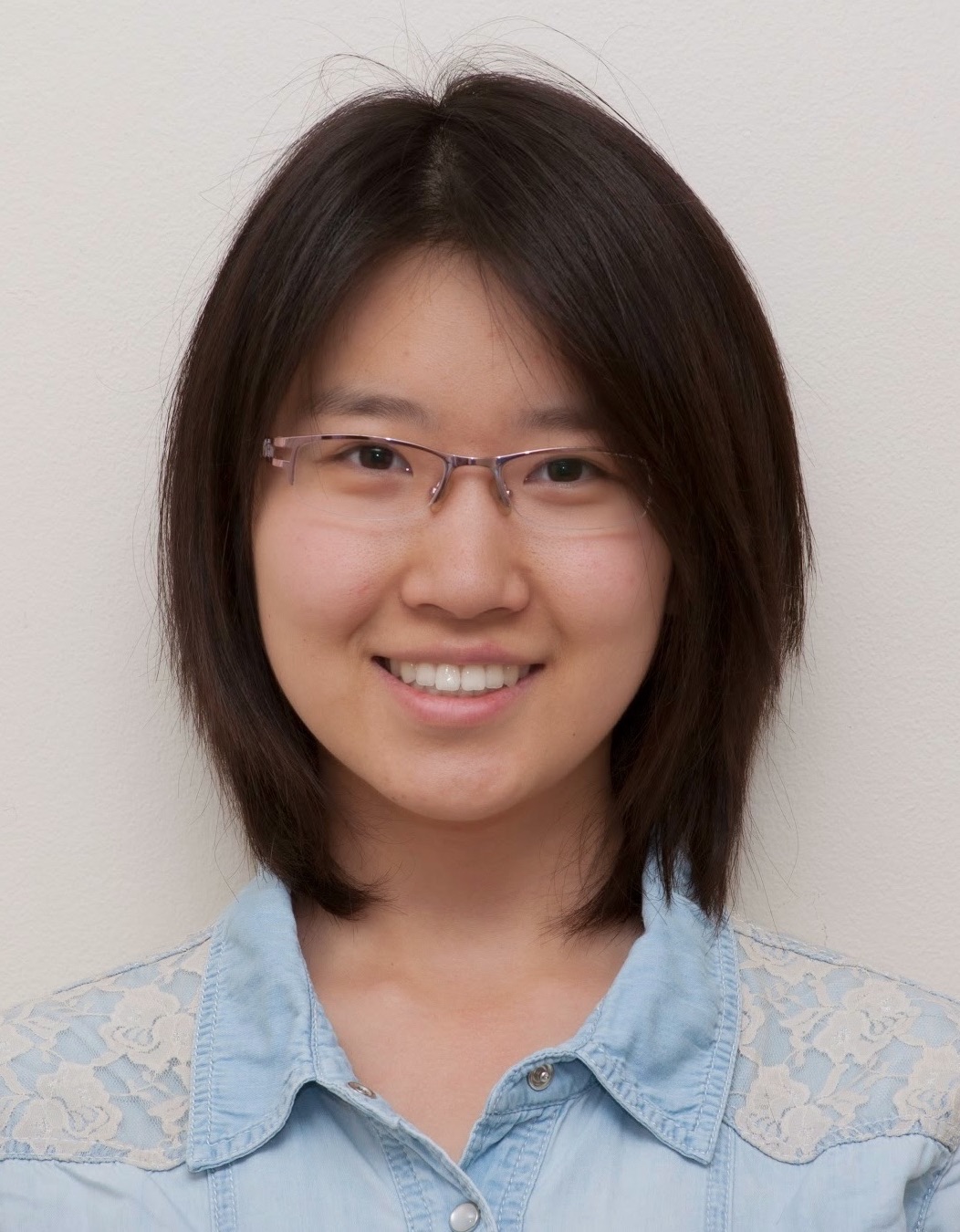}}]{Donna Xu}
	received a BCST (Honours) in computer science from the University of Sydney in 2014. She is currently pursuing a PhD degree under the supervision of Prof. Ivor W. Tsang at the Centre for Artificial Intelligence, FEIT, University of Technology Sydney, NSW, Australia. Her current research interests include multiclass classification, online hashing and information retrieval.
\end{IEEEbiography}
\begin{IEEEbiography}[{\includegraphics[width=1in,height=1.25in,clip,keepaspectratio]{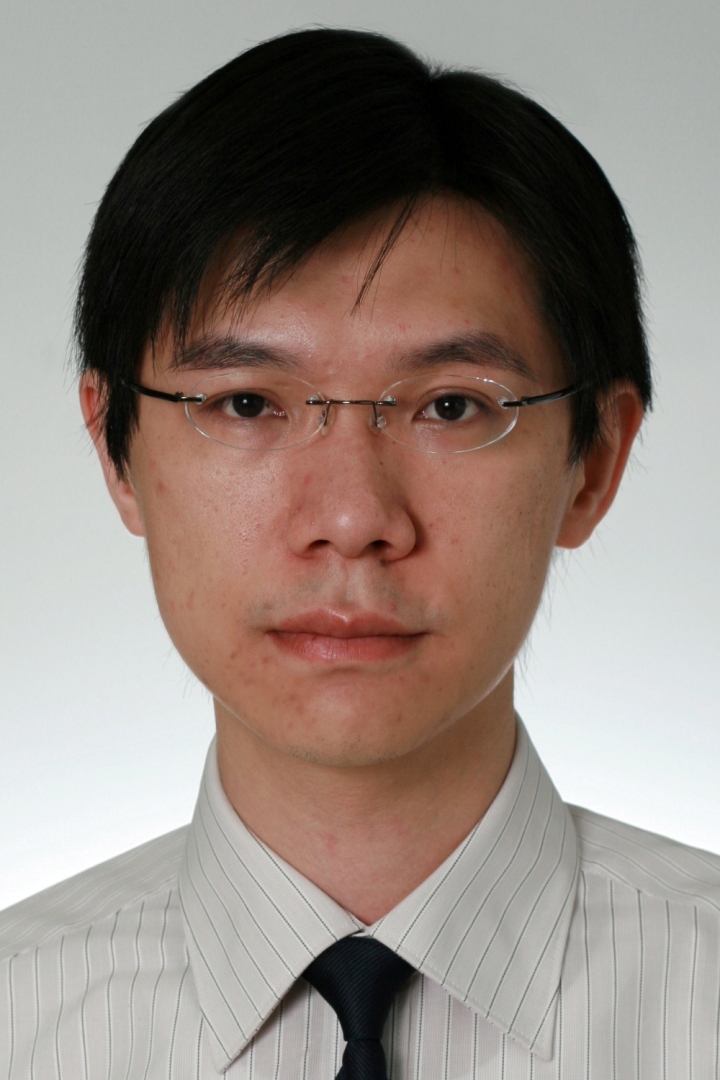}}]{Ivor W. Tsang}
	is an ARC Future Fellow and
	Professor at University of Technology Sydney (UTS). He 
	is also the Research Director of the UTS Priority Research
	Centre for Artificial Intelligence (CAI). He received his PhD 
	degree in computer science from the Hong Kong University
	of Science and Technology in 2007. In 2009, Dr Tsang was 
	conferred the 2008 Natural Science Award (Class II) by Ministry 
	of Education, China.  In addition, he had received the prestigious 
	IEEE Transactions on Neural Networks Outstanding 2004 Paper 
	Award in 2007, the 2014 IEEE Transactions on Multimedia Prize 
	Paper Award.
\end{IEEEbiography}
\begin{IEEEbiography}[{\includegraphics[width=1in,height=1.25in,clip,keepaspectratio]{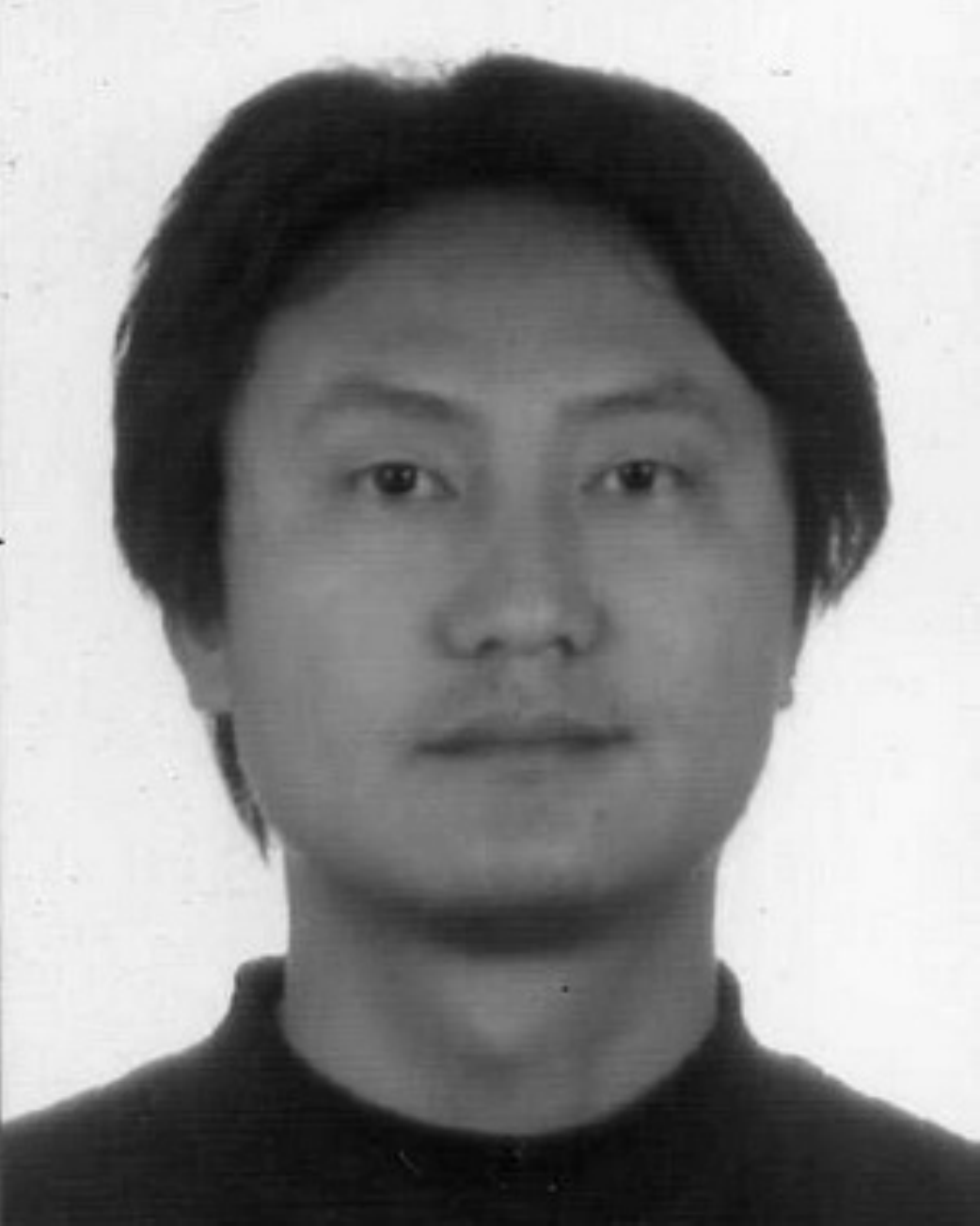}}]
	{Ying Zhang} is a senior lectuer and ARC DECRA research fellow
	(2014-2016) at QCIS, the University of Technology, Sydney (UTS).  He
	received his BSc and MSc degrees in Computer Science from Peking
	University, and PhD in Computer Science from the University of New
	South Wales. His research interests include query processing on data
	stream, uncertain data and graphs.  He was an Australian Research
	Council Australian Postdoctoral Fellowship (ARC APD) holder
	(2010-2013).
\end{IEEEbiography}




\end{document}